\documentclass[manuscript,screen]{acmart}

\AtBeginDocument{%
  \providecommand\BibTeX{{%
    \normalfont B\kern-0.5em{\scshape i\kern-0.25em b}\kern-0.8em\TeX}}}

\usepackage{mathtools}
\usepackage[nameinlink,capitalize]{cleveref}
\hypersetup{colorlinks=true,linkcolor=blue,citecolor=blue,urlcolor=blue,pdfborder={0 0 0}}
\usepackage{mathtools}

\usepackage[utf8]{inputenc} %
\usepackage[T1]{fontenc}    %
\usepackage{hyperref}       %
\usepackage{url}            %
\usepackage{booktabs}       %
\usepackage{amsfonts}       %
\usepackage{nicefrac}       %
\usepackage{microtype}      %
\usepackage{graphicx,wrapfig}
\usepackage{algorithm}
\usepackage[noend]{algpseudocode}
\usepackage{xcolor}
\usepackage{colortbl}
\usepackage{amsfonts}
\usepackage{url}
\usepackage{enumitem}
\usepackage[caption=false]{subfig}
\usepackage{amsthm}
\newtheorem{definition}{Definition}[section]
\newtheorem{theorem}{Theorem}[section]

\newtheorem{lemma}[theorem]{Lemma}
\newcommand{\rev}[1]{#1}  %

\newcommand{\cD}{{\mathcal{D}}}

\newcommand{\cM}{{\mathcal{M}}}
\newcommand{\cN}{{\mathcal{N}}}
\newcommand{\cO}{{\mathcal{O}}}

\newcommand{\cX}{{\mathcal{X}}}
\newcommand{\cY}{{\mathcal{Y}}}
\newcommand{\cZ}{{\mathcal{Z}}}

\newcommand{\RR}{\mathbb{R}}

\newcommand{\II}{\mathbb{I}}

\newcommand{\st}{{\text{s.t.}}} %
\DeclareMathOperator*{\argmin}{arg\,min}

\newcommand{\bc}{\begin{center}}
\newcommand{\ec}{\end{center}}

\newcommand{\bdm}{\begin{displaymath}}
\newcommand{\edm}{\end{displaymath}}

\newcommand{\beq}{\begin{equation}}
\newcommand{\eeq}{\end{equation}}

\newcommand{\bfl}{\begin{flushleft}}
\newcommand{\efl}{\end{flushleft}}

\newcommand{\bt}{\begin{tabbing}}
\newcommand{\et}{\end{tabbing}}

\newcommand{\beqn}{\begin{align}}
\newcommand{\eeqn}{\end{align}}

\newcommand{\beqs}{\begin{align*}} %
\newcommand{\eeqs}{\end{align*}}  %

\newcommand{\norm}[1]{\left\|#1\right\|}

\newcommand{\Ebb}{\mathbb{E}}
\newcounter{counterAllProofEnd}
\stepcounter{counterAllProofEnd}
\newenvironment{theoremEnd}[1]{
  \def\theoremEnd@name{#1}
  \begin{\theoremEnd@name} %
}{
  \end{\theoremEnd@name}
}
\excludecomment{proofEnd}
\newcommand{\stepThmRelabel}{
  \stepcounter{counterAllProofEnd}%
  \label{thm:prAtEnd\roman{counterAllProofEnd}}
}
 
\copyrightyear{2022}
\acmYear{2022}
\setcopyright{acmcopyright}\acmConference[FAccT '22]{2022 ACM Conference on Fairness, Accountability, and Transparency}{June 21--24, 2022}{Seoul, Republic of Korea}
\acmBooktitle{2022 ACM Conference on Fairness, Accountability, and Transparency (FAccT '22), June 21--24, 2022, Seoul, Republic of Korea}
\acmPrice{15.00}
\acmDOI{10.1145/3531146.3533070}
\acmISBN{978-1-4503-9352-2/22/06}

\begin{document}

\title{Dynamic Privacy Budget Allocation Improves Data Efficiency of Differentially Private Gradient Descent}

\author{Junyuan Hong}
\email{hongju12@msu.edu}
\affiliation{%
  \institution{Michigan State University}
  \city{East Lansing}
  \state{MI}
  \country{USA}
  \postcode{48823}
}
\author{Zhangyang Wang}
\email{atlaswang@utexas.edu}
\affiliation{%
  \institution{University of Texas at Austin}
  \city{Austin}
  \state{TX}
  \country{USA}
  \postcode{78712}
}
\author{Jiayu Zhou}
\email{jiayuz@msu.edu}
\affiliation{%
  \institution{Michigan State University}
  \city{East Lansing}
  \state{MI}
  \country{USA}
  \postcode{48823}
}

\renewcommand{\shortauthors}{Hong et al.}

\begin{abstract}
Protecting privacy in learning while maintaining the model performance has become increasingly critical in many applications that involve sensitive data. Private Gradient Descent (PGD) is a commonly used private learning framework, which noises gradients based on the Differential Privacy protocol. Recent studies show that \emph{dynamic privacy schedules} of decreasing noise magnitudes can improve loss at the final iteration, and yet theoretical understandings of the effectiveness of such schedules and their connections to optimization algorithms remain limited. In this paper, we provide comprehensive analysis of noise influence in dynamic privacy schedules to answer these critical questions. We first present a dynamic noise schedule minimizing the utility upper bound of PGD, and show how the noise influence from each optimization step collectively impacts utility of the final model. Our study also reveals how impacts from dynamic noise influence change when momentum is used. We empirically show the connection exists for general non-convex losses, and the influence is greatly impacted by the loss curvature.
\end{abstract}



\maketitle

\section{Introduction}

In the era of big data, privacy protection in machine learning systems is becoming a crucial topic as increasing personal data involved in training models~\citep{dwork2020privacy} and the presence of malicious attackers~\citep{shokri2017membership,fredrikson2015model}.
In response to the growing demand, differential-private (DP) machine learning~\citep{dwork2006calibrating} provides a computational framework for privacy protection and has been widely studied in various settings, including both convex and non-convex optimization~\citep{wang2017differentially, wang2019differentially,jain2019making}.

One widely used procedure for privacy-preserving learning is the (Differentially) Private Gradient Descent (PGD) \citep{bassily2014private,abadi2016deep}. A typical gradient descent procedure updates its model by gradients of the loss evaluated on a training dataset. When the data is sensitive, the gradients should be \emph{privatized} to prevent excess privacy leakage. The PGD privatizes a gradient by adding controlled noise. As such, the models from PGD is expected to have a lower utility as compared to those from unprotected algorithms. In the cases where strict privacy control is exercised, or equivalently, a tight \emph{privacy budget}, accumulating effects from highly-noised gradients may lead to unacceptable model performance. It is thus critical to design effective privatization procedures for PGD to maintain a great balance between utility and privacy.

\rev{Recent years witnessed a promising direction of privatization that \emph{dynamically allocate a privacy budget} for each iteration to boost utility, under the constraint of a specified total privacy budget.} One example is \citep{lee2018concentrated}, which reduces the budget-bonded noise magnitude when the loss does not decrease, due to the observation that gradients become very small when approaching convergence, and a static noise scale will overwhelm these gradients. Another example is \citep{yu2019differentially}, which periodically decreases the magnitude following a predefined strategy, e.g., exponential decaying or step decaying. Both approaches confirmed the empirically advantages of decreasing noise magnitudes. Intuitively, the dynamic mechanism may coordinate with certain properties of the learning task, e.g., training data and loss surface. 
\rev{Following the work, improved allocation policies are proposed, e.g., policies transferred from auxiliary datasets~\cite{hong2021learning}, policies with distributed budgets~\cite{cheng2020adaptive}, and a combination with adaptive learning rate \cite{xu2020adaptive}.}
Yet there is little theoretical analysis available and two important questions remain unanswered: 1) \emph{What is the form of utility-preferred budget (or noise equivalently) schedules?}
2) \emph{When and to what extent such an allocation policy improves utility?}

\rev{
Though there are theoretical studies of static-allocation policies, e.g., \cite{wang2017differentially}, the data efficiency is not the focus as discussions usually assume an unlimited amount of data is available.
However, we argue that the data efficiency with limited data size is critical in practice, especially when DP makes the learning more data-hungry~\cite{mcmahan2018learning}.
One example is federated learning \cite{mcmahan2017communicationefficient,mcmahan2018learning}, a distributed learning framework that aggregates many local models to form a stronger global model, where each model is privately trained on a local client, typically with very limited private data.
Another example is biomedical applications, where collecting samples involves expansive clinical trials or cohort studies, resulting the scarcity of training set. 
To study biomarkers of Alzheimer's, NIH has funded Alzheimer’s Disease Neuroimaging Initiative for \$40 million, which collected imaging and genetic biomarkers from only 800 patients after 5 years' extensive and collaborative efforts~\cite{weiner2013alzheimer}.
Therefore, we believe data efficiency needs to be taken into account in developing a private learning algorithm. 
}

To answer these questions, in this paper we develop a principled approach to construct dynamic schedules and quantify their utility bounds in different learning algorithms. 
Our contributions are summarized as follows.
1) For the class of loss functions satisfying the Polyak-Lojasiewicz condition~\citep{polyak1963gradient}, \rev{we show that dynamic schedules, that improve the utility upper bound with high data-efficiency, are shaped by the changing influence of per-iteration noise on the final loss.
As the influence is tightly connected to the loss curvature, the advantage of using dynamic schedules therefore depends on the loss function.}
2) Beyond vanilla gradient descent, our results show the gradient methods with momentum implicitly introduce a dynamic schedule and result in an non-monotonous influence trend.
3) We also show that our results are generalizable to population bounds in high probability or based on uniform stability theorems.
Though our major focus is on the theoretic study, we empirically validate the results on a non-convex loss function formulated by a neural network.
The empirical results suggest that a preferred dynamic schedule admits the exponentially decaying form, and works better when learning with high-curvature loss functions.
Moreover, dynamic schedules give higher utility under stricter privacy conditions (e.g., smaller sample size and less privacy budget).

\section{Related Work}

\textbf{Differentially Private Learning}.
Differential privacy (DP) characterizes the chance that an algorithm output (e.g., a learned model) leaks private information of its training data when the output distribution is known.
Since outputs of many learning algorithms have undetermined distributions, the probabilistic risk is hard to measure.
A common approach to tackle this issue is to inject randomness with known distribution to \emph{privatize} the learning procedures.
Classical methods include output perturbation~\citep{chaudhuri2011differentially}, objective perturbation~\citep{chaudhuri2011differentially} and gradient perturbation~\citep{abadi2016deep,bassily2014private,wu2017bolton}.
Among these approaches, the Private Gradient Descent (PGD) has attracted extensive attention in recent years because it can be flexibly integrated with variants of gradient-based iteration methods, e.g., stochastic gradient descent, momentum methods~\citep{qian1999momentum}, and Adam~\citep{kingma2015adam}, for (non-)convex problems.

\textbf{Dynamic Policies for Privacy Protection}.
\citet{wang2017differentially} studied the empirical risk minimization using dynamic variation reduction of perturbed gradients.
They showed that the utility upper bound can be achieved by gradient methods under uniform noise parameters.
Instead of enhancing the gradients, \citet{yu2019differentially,lee2018concentrated} showed the benefits of using a dynamic schedule of privacy parameters or equivalently noise scales.
\rev{Following \cite{lee2018concentrated}, a series of work \cite{cheng2020adaptive,huang2019differentially,xie2021differential,zhang2021adaptive} adaptively allocate privacy budget towards better privacy-utility trade-off.}
Moreover, adaptive sensitivity control \citep{pichapati2019adaclip,thakkar2019differentially} and dynamic batch sizes \citep{feldman2020private} are also shown to improve convergence.

\begin{table*}[t]
  \caption{\small Comparison of utility upper bound using different privacy schedules. The algorithms are $T$-iteration $\frac{1}{2}R$-zCDP under the PL condition (unless marked with * for convexity). %
  The $O$ notation in this table drops other $\ln$ terms. 
  Unless otherwise specified, all algorithms (including non-private GD) terminate at step $T=\cO(\ln{\frac{N^2R}{D}})$.
  Assume loss functions are $1$-smooth and $1$-Lipschitz continuous, and
  all parameters satisfy their numeric assumptions.
  Key notations:  
  $\cO_p$ -- bound occurs in probability $p$; 
  $D$ -- feature dimension; 
  $N$ -- sample size;  
  $R$ -- privacy budget where $R_{\epsilon,\delta}$ is the equivalent budget accounted by $(\epsilon,\delta)$-DP; 
  $c_i$ -- constant;
  other notations can be found in \cref{sec:dynamic}. 
  We provide citation after algorithm names for concise reference here and defer detailed explanations of baseline algorithms to \cref{sec:comp_alg}.
  Extensions to generalization error bounds are given in \cref{tbl:compare_gen_bd_appd}.
  }
  \label{tbl:compare_bd}
  \small
  \centering
  \begin{tabular}{lcc}
    \toprule
    \multicolumn{1}{c}{\bf Algorithm} & \multicolumn{1}{c}{\bf Schedule ($\sigma_t^2$)} & \multicolumn{1}{c}{\bf Utility Upper Bound} \\
    \midrule
    *GD+Adv~\citep{bassily2014private} & $\cO\left({\ln(N/\delta) \over R_{\epsilon, \delta}}\right)$  & $\cO\left( D \ln^3 N\over N R_{\epsilon, \delta} \right)$ \\
    GD+MA~\citep{wang2017differentially} & $\cO({T \over R_{\epsilon, \delta}})$ & $\cO \left( {D \ln^2 N \over N^2 R_{\epsilon, \delta}} \right)$   \\
    GD+MA (adjusted utility)~\citep{yu2020gradientb} & $\cO({T \over R_{\epsilon, \delta}})$ & $\cO \left( \min {\sqrt{D} \over N R_{\epsilon, \delta}}, {D \ln N \over N^2 R^2_{\epsilon, \delta}} \right)$   \\
    *GD+Adv+BBImp~\citep{cummings2018differential} & $\cO\left({n^2\ln(n/\delta) \over R_{\epsilon, \delta}}\right)$ & $ \cO_p \left( D^2 \ln^2 (1/p) \over R_{\epsilon, \delta} N^{1-c} \right)$  \\
    Adam+MA~\citep{zhou2020private} & $\cO({T \over R_{\epsilon, \delta}})$ & $\cO_p\left( {\sqrt{D} \ln(ND\epsilon/(1-p) ) \over N R_{\epsilon, \delta}} \right)$ \\
    \midrule
    GD, Non-Private \vspace{+0.05in} & $0$ & $\cO \left( \frac{D }{ N^2 R} \right)$ \\
    GD+zCDP, Static Schedule \vspace{+0.05in} & $\frac{T}{R}$ & $\cO \left( \frac{D \ln N }{ N^2 R} \right)$   \\
    GD+zCDP, Dynamic Schedule \vspace{+0.05in} & $\cO\left( \frac{\gamma^{(t-T)/2}}{R} \right)$ & $ \cO \left( \frac{D }{ N^2 R} \right)$ \\
    Momentum+zCDP, Static Schedule \vspace{+0.05in} & $\frac{T}{R}$ &  $ \cO \left( \frac{D }{ N^2 R} (c + \ln N \II_{T>\hat T}) \right)$   \\
    Momentum+zCDP, Dynamic Schedule & $\cO\left( {c_1 \gamma^{T+t}  + c_2 \gamma^{(T-t)/2} \over R} \right)$ & $ \cO \left( \frac{D }{ N^2 R} (1 + \frac{c D}{ N^2 R}\II_{T>\hat T} ) \right)$ \\
    \bottomrule
  \end{tabular}
\end{table*}

\textbf{Utility Upper Bounds}.
Utility upper bounds are a critical metric for privacy schedules, which characterizes the maximum utility that a schedule can deliver in theory. 
\citet{wang2017differentially} is the first to prove the utility bound under the PL condition.
Recently, \citeauthor{zhou2020private} proved the utility bound by using the momentum of gradients \citep{polyak1964methods,kingma2015adam}.
In this paper, we improve the upper bound by a more accurate estimation of the dynamic influence of step noise. 
We show that introducing a dynamic schedule further boosts the sample-efficiency of the upper bound. 
\cref{tbl:compare_bd} summarizes the upper bounds of 
a selection of state-of-the-art algorithms based on private gradients (up block, see \cref{sec:comp_alg} for the full list), and methods studied in this paper (down block), showing the benefits of dynamic influence. 
Especially, a closely-related work by Feldman \emph{et al.}  achieved a convergence rate similar to ours in terms of generalization error bounds (c.f. SSGD in \cref{tbl:compare_gen_bd_appd}), by dynamically adjusting batch sizes~\cite{feldman2020private}. 
However, the approach requires controllable batch sizes, which may not be feasible in many applications.  
In federated learning, for example, where users update models locally and then pass the parameters to server for aggregation, the server has no control over batch sizes, and coordinating users to use varying batch sizes may not be realistic. 
On the other hand, our proposed method can still be applied for enhancing utility, as the server can dynamically allocate privacy budget for each round when the presence of a user in the global aggregation is privatized \cite{mcmahan2018learning}.
\section{Private Gradient Descent}

\textbf{Notations}. 
We consider a learning task by empirical risk minimization (ERM) 
$f(\theta) = {1\over N} \sum_{n=1}^N f(\theta; x_n)$
on a private dataset $\{x_n\}_{n=1}^N$ and $\theta \in \RR^D$.
The gradient methods are defined as 
$\theta_{t+1} = \theta_t - \eta_t \nabla_t$,
where $\nabla_t = \nabla f(\theta_t) = {1\over N} \sum_{n} \nabla f(\theta_t; x_n)$ denotes the non-private gradient at iteration $t$, $\eta_t$ is the step learning rate. %
$\nabla_{t}^{(n)} = \nabla f(\theta_t; x_n)$ denotes the gradient on a sample $x_n$.
$\II_{c}$ denotes the indicator function that returns $1$ if the condition $c$ holds, otherwise $0$.

\textbf{Assumptions}. (1) In this paper, we assume $f(\theta)$ is continuous and differentiable. Many commonly used loss functions satisfy this assumption, e.g., the logistic function. 
(2) For a learning task, only finite amount of privacy cost is allowed where the maximum cost is called \emph{privacy budget} and denoted as $R$.
(3) Generally, we assume that loss functions $f(\theta;x)$ (sample-wise loss) are $G$-Lipschitz continuous and $f(\theta)$ (the empirical loss) is $M$-smooth.

\begin{definition}[$G$-Lipschitz continuity]
A function $f(\cdot)$ is $G$-Lipschitz continuous if, for $G > 0$ and all $x,y$ in the domain of $f(\cdot)$, $f(\cdot)$ satisfies
$\norm{f(y) - f(x)} \le G \norm{y-x}$.
. %
\end{definition}

\begin{definition}[$m$-strongly convexity]
A function $f(\cdot)$ is $m$-strongly convex if 
$f(y) \geq f(x) + \nabla f(x)^T (y-x) + \frac{m}{2} \|y-x\|^2$,
for some $m > 0$ and all $x,y$ in the domain of $f(\cdot)$. %
\end{definition}

\begin{definition}[$M$-smoothness] \label{def:smooth}
A function is $M$-smooth w.r.t. $l_2$ norm if 
$f(y) \leq f(x) + \nabla f(x)^T (y-x) + \frac{M}{2} \|y-x\|^2$,
for some constant $M > 0$ and all $x, y$ in the domain of $f(\cdot)$.
\end{definition}

For a private algorithm $\cM(d)$ which maps a dataset $d$ to some output, the privacy cost is measured by the bound of the output difference on the adjacent datasets.
\emph{Adjacent datasets} are defined to be datasets that only differ in one sample.
In this paper, we use the zero-Concentrated Differential Privacy (zCDP, see \cref{def:zCDP}) as the privacy measurement, because it provides the simplicity and possibility of adaptively composing privacy costs at each iteration.
Various privacy metrics are discussed or reviewed in \citep{desfontaines2019sok}. %
A notable example is Moment Accoutant (MA) \citep{abadi2016deep}, which adopts similar principle for composing privacy costs while is less tight for a smaller privacy budget.
We note that alternative metrics can be adapted to our study without major impacts to the analysis.  
\begin{definition}[$\rho$-zCDP \citep{bun2016concentrated}]
\label{def:zCDP}
Let $\rho>0$.
A randomized algorithm $\cM: \cD^n \rightarrow \RR$ satisfies $\rho$-zCDP if, for all adjacent datasets $d, d' \in \cD^n$,
$D_{\alpha} (\cM(d) \| \cM(d')) \le \rho \alpha, ~ \forall \alpha \in (1, \infty)$
where $D_{\alpha} (\cdot \| \cdot)$ denotes the R{\'e}nyi divergence \citep{renyi1961measures} of order $\alpha$.
\end{definition}
The zCDP provides a linear composition of privacy costs of sub-route algorithms.
When the input vector is privatized by injecting Gaussian noise of $\cN(0, \sigma_t^2 I)$ for the $t$-th iteration, the composed privacy cost is proportional to $\sum_t \rho_t$ where the step cost is $\rho_t = {1\over  \sigma_t^{2}}$. 
For simplicity, we absorb the constant coefficient into the (residual) \emph{privacy budget} $R$. 
The formal theorems for the privacy cost computation of composition and Gaussian noising is included in \cref{thm:comp_postprocess_zCDP,thm:gauss_zCDP}.

\begin{algorithm}[ht]
\begin{flushleft}
\textbf{Input}: Raw gradients $[\nabla_t^{(1)}, \dots, \nabla_t^{(n)}]$ ($n=N$ by default), $v_t$, residual privacy budget $R_t$ assuming the full budget is $R$ and $R_1=R$. \\
\end{flushleft}
\begin{algorithmic}[1]
\State $\rho_t \leftarrow 1/\sigma^2_t$, $\nabla_t \leftarrow {1\over n} \sum_{i=1}^n \nabla_t^{(i)}$ (Budget request)
\If{$\rho_t < R_t$}
	\State $R_{t+1} \leftarrow R_t - \rho_t$
	\State $g_t \leftarrow \nabla_t + G \sigma_t \nu_t / N$, $\nu_t \sim \mathcal{N}(0,I)$ \label{alg:ln:private_noise_grad} (Privacy noise) %
	\State $m_{t+1} \leftarrow \phi(m_t, g_t) $ or $g_1$ if $t=1$
	\State \textbf{return} $\eta_t m_{t+1}$, $R_{t+1}$ (Utility projection) \label{alg:line}
\Else
	\State Terminate
\EndIf
\end{algorithmic}
\caption{Privatizing Gradients}
\label{alg:private_grad}
\end{algorithm}

Generally, we define the Private Gradient Descent (PGD) method as iterations for $t=1\dots T$:
\begin{align}
\theta_{t+1} = \theta_t - \eta_t \phi_t = \theta_t - \eta_t (\nabla_t + \sigma_t G \nu_t / N), \label{eq:private_iter}
\end{align}
where \rev{$\phi_t=g_t$ is the gradient privatized from $\nabla_t$} as shown in \cref{alg:private_grad},
$G/N$ is the bound of sensitivity of the gathered gradient excluding one sample gradient, and 
$\nu_t \sim \mathcal{N}(0,I)$ is a vector element-wisely subject to Gaussian distribution.
We use $\sigma_t$ to denote the noise scale at step $t$ and use $\sigma$ to collectively represents the schedule $(\sigma_1, \dots, \sigma_T)$ if not confusing.
When the Lipschitz constant is unknown, we can control the upper bound by scaling the gradient if it is over some constant. The scaling operation is often called \emph{clipping} in literatures since it clips the gradient norm at a threshold.
After the gradient is noised, we apply a modification, $\phi(\cdot)$, to enhance its utility.
In this paper, we consider two types of $\phi(\cdot)$:
\begin{align}
\phi(m_t, g_t)&=g_t \text{ (GD)}, \\ ~ \phi(m_t, g_t)&=[\beta (1-\beta^{t-1}) m_t + (1-\beta) g_t]/(1-\beta^t) \text{ (Momentum)}
\end{align}

We now show that the PGD using \cref{alg:private_grad} guarantees a privacy cost less than $R$:
\begin{theoremEnd}{theorem} \stepThmRelabel
\label{thm:privacy_guarantee}
Suppose $f(\theta;x)$ is $G$-Lipschitz continuous and the PGD algorithm with privatized gradients defined by \cref{alg:private_grad}, stops at step $T$.
The PGD algorithm outputs $\theta_T$ and satisfies $\rho$-zCDP where $\rho \le {1\over 2} R$.
\end{theoremEnd}
\begin{proofEnd}
	Because all sample gradient are $G$-Lipschitz continuous, the sensitivity of the averaged gradient is upper bounded by $G/N$.
	Based on \cref{thm:gauss_zCDP}, the privacy cost of $g_t$ is $1\over 2\sigma_t^2$\footnote{For brevity, when we say the privacy cost of some value, e.g., gradient, we actually refer to the cost of mechanism that output the value.}.
	
	Here, we make the output of each iteration a tuple of $(\theta_{t+1}, v_{t=1})$.
	For the $1$st iteration, because $\theta_1$ does not embrace private information by random initialization, the mapping,
	\begin{gather*}
	\left[\begin{array}{cc}
	v_{2} \\
	\theta_{2}
	\end{array} \right] = \left[\begin{array}{cc}
	g_1 \\
	\theta_1 - \eta_1 g_1
	\end{array} \right],
	\end{gather*}
	is $\hat \rho_1$-zCDP where $\hat \rho_1 = {1\over 2\sigma_t^2}$.
	
	Suppose the output of the $t$-th iteration, $(\theta_t, v_t)$, is $\hat \rho_t$-zCDP.
	At each iteration, we have the following mapping $(\theta_t, v_t) \rightarrow (\theta_{t+1}, v_{t+1})$ defined as
	\begin{gather*}
	\left[\begin{array}{cc}
	v_{t+1} \\
	\theta_{t+1}
	\end{array} \right] = \left[\begin{array}{cc}
	\phi(v_t, g_t) \\
	\theta_t - \eta_t \phi(v_t, g_t)
	\end{array} \right].
	\end{gather*}
	Thus, the output tuple $(\theta_{t+1}, v_{t+1})$ is $(\hat \rho_t + {1\over 2\sigma_t^2})$-zCDP by \cref{thm:comp_postprocess_zCDP}.
	
	Thus, the recursion implies that $(\theta_{T+1}, v_{T+1})$ has privacy cost as
	\begin{align*}
	\hat \rho_{T+1} = \hat \rho_{T} + {1\over 2\sigma_T^2} = \cdots = \sum_{t=1}^T {1\over 2\sigma_t^2} = {1\over 2} \sum_{t=1}^T \rho_t \le {1\over 2} (R-R_T) \le {1\over 2} R.
	\end{align*} 
	Let $\rho = \hat \rho_{T+1}$. Then we can get the conclusion.
\end{proofEnd}
Note that \cref{thm:privacy_guarantee} allows $\sigma_t$ to be different throughout iterations.
Next we present a principled approach for deriving dynamic schedules optimized for the final loss $f(\theta_T)$.

\section{Dynamic Policies by Minimizing Utility Upper Bounds}
\label{sec:dynamic}

To characterize the utility of the PGD, we adopt the Expected Excess Risk (EER), which notion is widely used for analyzing the convergence of random algorithms, e.g., \citep{bassily2014private,wang2017differentially}.
Due to the presence of the noise and the limitation of learning iterations, optimization using private gradients is expected to reach a point with a higher loss (i.e., excess risk) as compared to the optimal solution without private protection. 
Define $\theta^* = \argmin\nolimits_\theta f(\theta)$,
after \cref{alg:private_grad} is iterated for $T$ times in total, the EER gives the expected utility degradation: 
\begin{align*}
\operatorname{EER} = \Ebb_\nu[f(\theta_{T+1})] - f(\theta^*).
\end{align*}
Due to the variety of loss function and complexity of recursive iterations, an exact EER with noise is intractable for most functions.
Instead, we study the worst case scenario, i.e., the upper bound of the EER, and our goal is to minimize the upper bound.
For consistency, we call the upper bound of EER divided by the initial error as ERUB. %
Since the analytical form of $\operatorname{EER}$ is either intractable or complicated due to the recursive iterations of noise, studying the $\operatorname{ERUB}$ is a convenient and tractable alternative.
The upper bound often has convenient functional forms which are (1) sufficiently simple, such that we can directly minimize it, and (2) closely related to the landscape of the objective depending on both the training dataset and the loss function.
As a consequence, it is also used in previous PGD literature~\citep{pichapati2019adaclip,wang2017differentially} for choosing proper parameters.
Moreover, we let $\operatorname{ERUB}_{\min}$ be the achievable optimal upper bound by a specific choice of parameters, e.g., the $\sigma$ and $T$.

As the EER is iteratively determined by \cref{eq:private_iter}, we define the influence of the dynamics in noise magnitude $\sigma_t$ as the derivative: $q_t^* = \frac{\partial \operatorname{EER}}{\partial \sigma_t}$.
Accordingly, we can approximate the EER shift as $q_t^* \Delta \sigma_t$ when $\sigma_t$ increases by $\Delta \sigma_t$.
However, because the EER is strongly data-dependent, the derived $q_t^*$ on a given dataset may not generalize to another dataset.
Instead, we consider a more general term based on ERUB, i.e., $q_t = \frac{\partial \operatorname{ERUB}}{\partial \sigma_t}$.

In this paper, we consider the class of loss functions satisfying the Polyak-Lojasiewicz (PL) condition which bounds losses by corresponding gradient norms.
It is more general than the $m$-strongly convexity.
If $f$ is differentiable and $M$-smooth, then $m$-strongly convexity implies the PL condition.
\begin{definition}[Polyak-Lojasiewicz condition \citep{polyak1963gradient}]
For $f(\theta)$, there exists $\mu > 0$ and for every $\theta$, $\norm{\nabla f(\theta)}^2 \ge 2 \mu (f(\theta) - f(\theta^*))$.
\end{definition} 
The PL condition helps us to reveal how the influence of step noise propagates to the final excess error, i.e., EER.
Though the assumption was also used previously in \citet{wang2017differentially,zhou2020private}, neither did they discuss the propagated influence of noise.
In the following sections, we will show how the influence can tighten the upper bound in gradient descent and its momentum variant.

\subsection{Gradient Descent Methods and Noise Influences}
\label{sec:gd_method}

For the brevity of variables, we first define the following summarized constants:
\begin{align}
\text{non-private ERUB : } & {\alpha} \triangleq {D G^2 \over 2R M N^2 ( f(\theta_1) - f(\theta^*))} \le \cO\left( {D G^2 \over R M N^2 } \right),\\
\text{curvature : } & \kappa \triangleq {M\over \mu}, \\
\text{convergence rate : } & \gamma \triangleq 1 - {1 \over \kappa} \label{eq:def:alpha_kappa_gamma},
\end{align}
which satisfy $\kappa \ge 1$ and $\gamma \in [0, 1)$.
Here, $\alpha$ is upper bounded by non-private ERUB within $T=\left\lceil \cO(\ln{N^2R(M-\mu)\over DG^2}) \right\rceil$ iterations.
Therefore, $\alpha$ provide a simple reference of an ideal convergence bound, reaching which indicates a superior performance with privacy guarantee.
$\kappa$ characterizes the curvature of $f(\cdot)$ which is the condition number of $f(\cdot)$ if $f(\cdot)$ is strongly convex, and $\gamma$ is the convergence rate for non-private SGD (c.f. \cref{thm:excess_loss_motivation} with $\sigma_t=0$).
$\kappa$ tends to be large if the function is sensitive to small differences in inputs, and $1/\alpha$ tends to be large if more samples are provided and with a less strict privacy budget.
The convergence of PGD under the PL condition has been studied for private \citep{wang2017differentially} and non-private \citep{karimi2016linear,nesterov2006cubic,reddi2016stochastic} ERM.
Below we extend the bound in \citep{wang2017differentially} by considering dynamic influence of noise and relax $\sigma_t$ to be dynamic:
\begin{theoremEnd}{theorem}\stepThmRelabel
\label{thm:excess_loss_motivation}
Let $\alpha$, $\kappa$ and $\gamma$ be defined in \cref{eq:def:alpha_kappa_gamma}, and $\eta_t = {1\over M}$.
Suppose $f(\theta; x_i)$ is $G$-Lipschitz and $f(\theta)$ is $M$-smooth satisfying the Polyak-Lojasiewicz condition.
For PGD, the following holds:  %
\begin{gather}
\operatorname{ERUB} = \gamma^T + R \sum\nolimits_{t=1}^T q_t \sigma_t^2, \text{ where } q_t \triangleq \gamma^{T-t} \alpha. \label{eq:excess_loss_motivation}
\end{gather}
\end{theoremEnd}
\begin{proofEnd}
	With the definition of smoothness in \cref{def:smooth} and \cref{eq:private_iter}, we have
	\begin{align*}
	f(\theta_{t+1}) - f(\theta_t) &\le - \eta_t \nabla_t^\top (\nabla_t + G \sigma_t \nu_t / N) + {1\over 2} M \eta_t^2 \norm{\nabla_t + G \sigma_t \nu_t / N}^2 \\
	&= - \eta_t (1 - {1\over 2} M \eta_t) \norm{\nabla_t}^2 - (1 - M\eta_t) \eta_t \nabla_t^\top G \sigma_t \nu_t / N + {1\over 2} M \eta_t^2 \norm{ G \sigma_t \nu_t / N}^2 \\
	&\le - 2 \mu \eta_t (1 - {1\over 2} M \eta_t) (f(\theta_t) - f(\theta^*) ) - (1 - M\eta_t) \eta_t \nabla_t^\top G \sigma_t \nu_t / N \\
	&\quad + {1\over 2} M \eta_t^2 \norm{ G \sigma_t \nu_t / N}^2.
	\end{align*}
	where the last inequality is due to the Polyak-Lojasiewicz condition.
	Taking expectation on both sides, we can obtain
	\begin{align*}
	\Ebb[f(\theta_{t+1})] - \Ebb[f(\theta_t)] &\le - 2 \mu \eta_t (1 - {M\over 2} \eta_t) (\Ebb[f(\theta_t)] - f(\theta^*) ) + {M\over 2} (\eta_t G \sigma_t / N)^2 \Ebb \norm{ \nu_t }^2
	\end{align*}
	which can be reformulated by substacting $f(\theta^*)$ on both sides and re-arranged as
	\begin{align*}
	\Ebb[f(\theta_{t+1})] - f(\theta^*) &\le \left( 1 - 2 \mu \eta_t (1 - {M\over 2} \eta_t) \right) (\Ebb[f(\theta_t)] - f(\theta^*) ) + {M\over 2} (\eta_t G \sigma_t / N)^2 D
	\end{align*}
	Recursively using the inequality, we can get
	\begin{align*}
	\Ebb[f(\theta_{T+1})] - f(\theta^*) &\le \prod_{t=1}^T \left(1 - 2 \mu \eta_t (1 - {M\over 2} \eta_t) \right) (\Ebb[f(\theta_1)] - f(\theta^*) ) \\
	&\quad + {M D\over 2} \sum_{t=1}^T \prod_{i=t+1}^T \left(1 - 2 \mu \eta_i (1 - {M\over 2} \eta_i) \right) (\eta_t G \sigma_t / N)^2.
	\end{align*}
	Let $\eta_t \equiv 1/M$. Then the above inequality can be simplified as
	\begin{align*}
	\Ebb[f(\theta_{T+1})] - f(\theta^*) &\le \gamma^T (\Ebb[f(\theta_1)] - f(\theta^*) ) + R  \sum_{t=1}^T \gamma^{T-t} {MD \over 2R} \left({\eta_t G \over N}\right)^2 \sigma_t^2 \\
	&= \gamma^T (\Ebb[f(\theta_1)] - f(\theta^*) ) + R  \sum_{t=1}^T \gamma^{T-t} \alpha \sigma_t^2 (\Ebb[f(\theta_1)] - f(\theta^*) ) \\
	&= \left( \gamma^T + R \sum\nolimits_{t=1}^T q_t \sigma_t^2 \right) (f(\theta_1) - f(\theta^*))
	\end{align*}
\end{proofEnd}
\cref{thm:excess_loss_motivation} degenerates to a non-private variant as no noise is applied, i.e., $\sigma_t=0$ for all $t$.
In \cref{eq:excess_loss_motivation}, the step noise magnitude $\sigma_t^2$ has an exponential \emph{influence}, $q_t$, on the EER.
Note we ignore the constant factor $R$ in the influence.
The \cref{eq:excess_loss_motivation} implies that the influence of noise at step $t$ increase quickly by an exponential rate.
Importantly, the increasing rate is the same as the convergence rate, i.e., the first term in \cref{eq:excess_loss_motivation}.
The dynamic characteristic of the influence is the key to prove a tighter bound.
Plus, on the presence of the dynamic influence, it is natural to choose a dynamic $\sigma_t^2$.
When relaxing $q_t$ to a static $1$, a static $\sigma_t^2$ was studied by \citeauthor{wang2017differentially} 
They proved a bound which is nearly optimal except a $\ln^2 N$ factor.
To get the optimal bound, in the following sections, we look for the $\sigma$ and $T$ that minimize the upper bound.

\subsubsection{Uniform Schedule}
The uniform setting of $\sigma_t$ has been previously studied in \citet{wang2017differentially}.
Here, we show \rev{that the bound can be further tightened by considering the dynamic influence of iterations and a proper $T$}.
\begin{theoremEnd}{theorem}\stepThmRelabel
\label{thm:optimal_T_eer}
Suppose conditions in \cref{thm:excess_loss_motivation} are satisfied.
When $\sigma_t^2 = T/R$, let $\alpha$, $\gamma$ and $\kappa$ be defined in \cref{eq:def:alpha_kappa_gamma} and let $T$ be:
$T = \left\lceil \cO \left( \kappa \ln \left(1 + {1 \over \kappa \alpha} \right) \right) \right\rceil$. %
Meanwhile, if $\kappa \ge {1\over 1 - c} > 1$, $1/\alpha > 1/ \alpha_0$ for some constant $c \in (0,  1)$ and $\alpha_0>0$, the corresponding bound is:
\begin{align}
\operatorname{ERUB}^{\text{uniform}}_{\min} = \Theta\left({\kappa^2 \over \kappa + 1/ \alpha} \ln \left( 1 + {1 \over \kappa \alpha} \right)\right). \label{eq:optimal_T_eer:eer}
\end{align}
\end{theoremEnd}
\begin{proofEnd}
	The minimizer of the upper bound of \cref{eq:excess_loss_motivation} can be written as
	\begin{align}
	T^* &= \argmin_T \gamma^T + \alpha \kappa (1 - \gamma^T) T  \label{eq:optimal_T_eer:erub}
	\end{align}
	where we substitute $\sigma^2 = T/R$ in the second line.
	To find the convex minimization problem, we need to vanishing its gradient which involves an equation like $T\gamma^T = c$ for some real constant $c$.
	However, the solution is $W_k(c)$ for some integer $k$ where $W$ is Lambert W function which does not have a simple analytical form.
	Instead, because $\gamma^T > 0$, we can minimize a surrogate upper bound as following
	\begin{align}
	T^* &= \argmin_T \gamma^T + \alpha \kappa T = 
		{1\over \ln(1/\gamma)} \ln \left( {\ln (1/\gamma) \over \kappa \alpha} \right), \text{ if } \kappa \alpha  + \ln \gamma < 0  
	\end{align}
	where we use the surrogate upper bound in the second line and utilize $\gamma = 1 - {1 \over \kappa}$.
	However, the minimizer of the surrogate objective is not optimal for the original objective.
	When $\kappa$ is large, the term, $- \alpha \kappa \gamma^T T$, cannot be neglected as we expect.
	On the other hands, $T$ suffers from explosion if $\kappa \rightarrow \infty$ and meanwhile $1/\gamma \rightarrow_+ 1$.
	The tendency is counterintuitive since a small $T$ should be taken for sharp losses.
	To fix the issue, we change the form of $T^*$ as 
	\begin{align}
	T^* = {1\over \ln (1/\gamma)} \ln \left( 1+ { \ln (1/\gamma) \over \alpha} \right), \label{eq:optimal_T_eer:optimal_T:proof}
	\end{align}
	which gradually converges to $0$ as $\kappa \rightarrow \infty$.
	
	Now we substitute \cref{eq:optimal_T_eer:optimal_T:proof} into the original objective function, \cref{eq:optimal_T_eer:erub}, to get
	\begin{align}
	\operatorname{ERUB}^{\text{uniform}} &= {1 \over 1 + \ln (1/\gamma)/ \alpha} \left[ 1 + \kappa \ln \left( 1 + { \ln (1/\gamma) \over \alpha} \right) \right] . \label{eq:optimal_T_eer:proof:erub_opt}
	\end{align}
	Notice that
	\begin{align*}
	\ln(1/\gamma) &= \ln(\kappa/(\kappa-1)) = \ln(1 + 1/(\kappa-1)) \le {1\over \kappa - 1} \le {1\over c\kappa}
	\end{align*}
	because $\kappa \ge {1\over 1 - c} > 1$ for some constant $c \in (0,  1)$.
	In addition,
	\begin{align*}
	\ln(1/\gamma) &= - \ln(1-1/\kappa) \ge 1/\kappa.
	\end{align*}
	Now, we can get the upper bound of \cref{eq:optimal_T_eer:proof:erub_opt} as
	\begin{align*}
	\operatorname{ERUB}^{\text{uniform}} &\le {\kappa \over \kappa + 1/ \alpha} \left[ 1 + \kappa \ln \left( 1 + { 1 \over c\kappa \alpha} \right) \right] \\
	&\le c_1 {\kappa \over \kappa + 1/ \alpha} \kappa \left[ \ln \left( 1 + { 1 \over \kappa \alpha} \right) + \ln ({1\over c}) )\right] \\
	&\le c_1 c_2 {\kappa^2 \over \kappa + 1/ \alpha} \ln \left( 1 + { 1 \over \kappa \alpha} \right)
	\end{align*}
	for some constants $c_1, c_2$ and large enough $1\over \alpha$.
	Also, we can get the lower bound
	\begin{align*}
	\operatorname{ERUB}^{\text{uniform}} &\ge {c\kappa \over c\kappa + 1/ \alpha} \left[ 1 + \kappa \ln \left( 1 + {1 \over \kappa \alpha} \right) \right] \ge c {\kappa^2 \over \kappa + 1/ \alpha} \ln \left( 1+ {1 \over \kappa \alpha} \right) .
	\end{align*}
	where we use the condition $c\in (0,1)$.
	Thus, $\operatorname{ERUB}^{\text{uniform}} = \Theta\left({\kappa^2 \over \kappa + 1/ \alpha} \ln \left( 1 + {1 \over \kappa \alpha} \right)\right)$.
\end{proofEnd}
\begin{proof}[Sketch of proof]
	The key of proof is to find a proper $T$ to minimize
	\begin{align}
	\operatorname{ERUB} &= E = \gamma^T + \sum\nolimits_{t=1}^T \gamma^{T-t} \alpha R \sigma^2 \notag
	\\ &= \gamma^T + \alpha T {1 - \gamma^T \over 1 - \gamma} = \gamma^T + \alpha \kappa (1 - \gamma^T) T  
	\end{align}
	where we use $\sigma_t = \sqrt{T/R}$.
	Vanishing its gradient is to solve $\gamma^T \ln \gamma + \alpha \kappa (1 - \gamma^T) - \alpha \kappa T \gamma^T \ln \gamma = 0$, which however is intractable.
	In \citep{wang2017differentially}, $T$ is chosen to be $\cO(\ln(1/\alpha))$ and ERUB is relaxed as $\gamma^T + \alpha \kappa T^2$.
	The approximation results in a less tight bound as $\cO(\alpha (1 + \kappa\ln^2(1/\alpha)))$ which explodes as $\kappa \rightarrow \infty$.
	
	We observe that for a super sharp loss function, i.e., a large $\kappa$, any minor perturbation may result in tremendously fluctuating loss values.
	In this case, not-stepping-forward will be a good choice.
	Thus, we choose $T={1\over \ln (1/\gamma)} \ln \left(1 + {\ln (1/\gamma) \over \alpha} \right) \le \cO \left( \kappa \ln \left(1 + {1 \over \kappa \alpha} \right) \right)$ which converges to $0$ as $\kappa\rightarrow + \infty $.
	The full proof is deferred to the appendix.
\end{proof}

\subsubsection{Dynamic Schedule}
A dynamic schedule can improve the upper bound delivered by the uniform schedule. 
First, we observe that the excess risk in \cref{eq:excess_loss_motivation} is upper bounded by two terms:
the first term characterizes the error due to the finite iterations of gradient descents;
the second term, a weighted sum, comes from error propagated from noise at each iteration.
Now we show for any $\{q_t | q_t > 0, t=1,\dots, T\}$ (not limited to the $q_t$ defined in \cref{eq:excess_loss_motivation}), there is a unique $\sigma_t$ minimizing the weighted sum:
\begin{theoremEnd}{lemma}[Dynamic schedule]\stepThmRelabel
	\label{thm:cachy-scharz_weighted_sum}
	Suppose $\sigma_t$ satisfy $\sum_{t=1}^T \sigma_t^{-2}=R$.
	Given a positive sequence $\{q_t\}$, the following equation holds:
	\begin{align}
	\min_{\sigma} R \sum\nolimits_{t=1}^T q_t \sigma_t^2 &= \left(\sum\nolimits_{t=1}^T \sqrt{q_t} \right)^2,\text{ when } \sigma_t^2 = {1\over R} \sum\nolimits_{i=1}^T \sqrt{q_i\over q_t}. \label{eq:pl_dyn_sigmat} %
	\end{align}
	Remarkably, the difference between the minimum and $T\sum_{t=1}^T q_t$ (uniform $\sigma_t$) monotonically increases by the variance of $\sqrt{q_t}$ w.r.t. $t$.
\end{theoremEnd}
\begin{proofEnd}
	By $\sum_{t=1}^T \sigma^{-2}=R$ and Cauchy-Schwarz inequality, we can derive the achievable lower bound as
	\begin{align*}
	R \sum_t q_t \sigma_t^2 = \sum_t {1\over \sigma_t^2} \sum_t q_t \sigma_t^2 \ge \left(\sum_{t=1}^T \sqrt{q_t} \right)^2
	\end{align*}
	where the inequality becomes equality if and only if ${s/ \sigma_t^2} = q_t \sigma_t^2$, i.e., $\sigma_t = (s/q_t)^{1/4}$, for some positive constant $s$.
	The equality $\sum_{t=1}^T \sigma_t^{-2}=R$ immediately suggests $\sqrt{s} = {1\over R} \sum_{t=1}^T \sqrt{q_t}$.
	Thus, we get the $\sigma_t$.

	\rev{
	Notice
	\begin{align}
		T\sum_{t=1}^T q_t - \left(\sum\nolimits_{t=1}^T \sqrt{q_t} \right)^2 = T^2 {1\over T} \sum_{t=1}^T \left(\sqrt{q_t} - {1\over T} \sum_{i=1}^T \sqrt{q_i} \right)^2 = T^2 \operatorname{Var}[q_t]
	\end{align}
	where the variance is w.r.t. $t$.
	}
\end{proofEnd}
We see that the dynamics in $\sigma_t$ come from the non-uniform nature of the weight $q_t$.
Since $q_t$ presents the impact of the $\sigma_t$ on the final error, we denote it as \emph{influence}.
Given the dynamic schedule in \cref{eq:pl_dyn_sigmat}, it is of our interest to which extent the ERUB can be improved.
First, we present \cref{thm:optimal_T_eer_dyn} to show the optimal $T$ and ERUB.
\begin{theoremEnd}{theorem} \stepThmRelabel
\label{thm:optimal_T_eer_dyn}
Suppose conditions in \cref{thm:excess_loss_motivation} are satisfied.
Let $\alpha$, $\kappa$ and $\gamma$ be defined in \cref{eq:def:alpha_kappa_gamma}.
When $\eta_t={1\over M}$, $\sigma_t$ (based on \cref{eq:excess_loss_motivation,eq:pl_dyn_sigmat}) and the $T$ minimizing ERUB are, i.e.,
$\sigma_t^2 = {1\over R} {\sqrt{(1/\gamma)^T} - 1 \over 1 - \sqrt{\gamma} } \sqrt{\gamma^{t}}$, $T = \left\lceil \left( {2\kappa} \ln \left( 1 + {1 \over \kappa \alpha } \right) \right) \right\rceil$. %
Meanwhile, when $\kappa \ge 1$ and $1/\alpha \ge 1/ \alpha_0$ for some positive constant $\alpha_0$, the minimal bound is:
\begin{align}
\operatorname{ERUB}^{\text{dynamic}}_{\min} &= \Theta \left( {\kappa^2 \over \kappa^2 + 1/\alpha} \right).  \label{eq:thm:optimal_T_eer_dyn:eer}
\end{align}
\end{theoremEnd}
\begin{proofEnd}
	The upper bound of \cref{eq:excess_loss_motivation} can be written as
	\begin{align*}
	\text{ERUB}^{\text{dyn}} &= \gamma^T + \sum\nolimits_{t=1}^T \gamma^{T-t} \alpha R \sigma^2_t \\
	&= \gamma^T + \alpha \left( \sum\nolimits_{t=1}^T \sqrt{\gamma^{T-t}} \right)^2 \\
	&= \gamma^T + \alpha \left({1 - \gamma^{T/2} \over 1 - \sqrt{\gamma}} \right)^2
	\end{align*}
	where we make use of \cref{thm:cachy-scharz_weighted_sum}.
	Then, the minimizer of the ERUB is
	\begin{align}
	T^* &= \argmin_T \gamma^T + \alpha \left({1 - \gamma^{T/2} \over 1 - \sqrt{\gamma}} \right)^2 \notag \\
	&= 2 \log_{\gamma} \left( {\alpha \over \alpha + (1 - \sqrt{\gamma})^2} \right).  \label{eq:thm:optimal_T_eer_dyn:proof:T}
	\end{align}
	We can substitute \cref{eq:thm:optimal_T_eer_dyn:proof:T} into $\text{ERUB}^{\text{dyn}}$ to get
	\begin{align*}
	\operatorname{ERUB}^{\text{dyn}}_{\min} &= \left( {\alpha \over \alpha + (1 - \sqrt{\gamma})^2} \right)^2 + \alpha \left({1 \over 1 - \sqrt{\gamma}} \right)^2 \left( 1 - {\alpha \over \alpha + (1 - \sqrt{\gamma})^2} \right)^2 \\
	&= \left( {\alpha (1 - \sqrt{\gamma})^{-2} \over \alpha (1 - \sqrt{\gamma})^{-2} + 1} \right)^2 + {\alpha (1 - \sqrt{\gamma})^{-2} \over \left( \alpha (1 - \sqrt{\gamma})^{-2} + 1 \right)^2} \\
	&= {\alpha (1 - \sqrt{\gamma})^{-2} \over \alpha (1 - \sqrt{\gamma})^{-2} + 1}
	\end{align*}
	Notice that $\left( 1 - \sqrt{\gamma} \right)^{-2} = \kappa^2 + \kappa^2 - \kappa + 2 \kappa \sqrt{\kappa (\kappa - 1)} = \kappa (2\kappa -1 + 2 \sqrt{\kappa (\kappa - 1)} )$ and it is bounded by
	\begin{gather*}
	\kappa (2\kappa -1 + 2 \sqrt{\kappa (\kappa - 1)} ) \le 4 \kappa^2, \\
	\kappa (2\kappa -1 + 2 \sqrt{\kappa (\kappa - 1)} ) \ge \kappa (2\kappa - (3\kappa-2) + 2 \sqrt{(\kappa - 1) (\kappa - 1)} ) = \kappa (-\kappa +2 + 2\kappa - 2 ) = \kappa^2.
	\end{gather*}
	Therefore, $\kappa \le \left( 1 - \sqrt{\gamma} \right)^{-1} \le 2\kappa$, with which we can derive
	\begin{align*}
	\operatorname{ERUB}^{\text{dyn}}_{\min} &\le 4 {\kappa^2 \alpha \over \kappa^2 \alpha + 1}, \\
	\operatorname{ERUB}^{\text{dyn}}_{\min} &\ge {\kappa^2 \alpha \over 4 \kappa^2 \alpha + 1} \ge {1\over 4} {\kappa^2 \alpha \over \kappa^2 \alpha + 1}.
	\end{align*}
	Thus, $\operatorname{ERUB}^{\text{dyn}}_{\min} = \Theta \left({\kappa^2 \alpha \over \kappa^2 \alpha + 1}\right)$.
\end{proofEnd}

\subsubsection{Discussion}
\label{sec:gd_discussion}

In \cref{thm:optimal_T_eer,thm:optimal_T_eer_dyn}, we present the tightest bounds for functions satisfying the PL condition, to our best knowledge.
We further analyze the advantages of our bounds from two aspects: sample efficiency and robustness to sharp losses.

\textbf{Sample efficiency}.
Since dataset cannot be infinitely large, it is critical to know how accurate the model can be trained privately with a limited number of samples.
Formally, it is of interest to study when $\kappa$ is fixed and $N$ is large enough such that $\alpha \gg 1$.
Then we have the upper bound in \cref{eq:optimal_T_eer:eer} as
\begin{align}
\operatorname{ERUB}^{\text{uniform}}_{\min} \le \cO\left({\kappa^2 \alpha} \ln \left( {1 \over \kappa \alpha} \right)\right) \le \tilde \cO \left( {D G^2 \ln(N) \over M N^2 R} \right), \label{eq:erub_uni_min_tilde_cO}
\end{align}
where we ignore $\kappa$ and other logarithmic constants with $\tilde \cO$ as done in \citet{wang2017differentially}.
As a result, we get a bound very similar to \citep{wang2017differentially}, except that $R$ is replaced by $R_{MA}=\epsilon^2 / \ln(1/\delta)$ using Moment Accountant.
In comparison, based on \cref{lm:zCDP2DP}, $R=2\rho=2\epsilon + 4\ln(1/\delta) + 4 \sqrt{\ln(1/\delta)(\epsilon + \ln(1/\delta)}$ if $\theta_T$ satisfies $\rho$-zCDP.
Because $\ln(1/\delta) > 1$, it is easy to see $R=R_{zCDP} > R_{MA}$ when $\epsilon \le 2\ln(1/\delta)$.
As compared to the one reported in \citep{wang2017differentially}, our bound saved a factor of $\ln N$ and thus require less sample to achieve the same accuracy.
Remarkably, the saving is due to the maintaining of the influence terms as shown in the proof of \cref{thm:optimal_T_eer}.

Using the dynamic schedule, we have
$\operatorname{ERUB}^{\text{dynamic}}_{\min} \le \cO(\alpha) = \cO \left( {DG^2\over MN^2R}  \right)$,
which saved another $\ln N$ factor in comparison to the one using the uniform schedule \cref{eq:erub_uni_min_tilde_cO}.
As shown in \cref{tbl:compare_bd}, such advantage maintains when comparing with other baselines and reaches the ideal non-private case
, recalling the meaning of $\alpha$.

\textbf{Stability on ill-conditioned loss}. 
Besides sample efficiency, we are also interested in robustness of the convergence under the presence of privacy noise.
Because of the privacy noise, the private gradient descent will be unable to converge to where a non-private algorithm can reach.
Specifically, when the samples are noisy or have noisy labels, the loss curvature may be sharp.
The sharpness also implies lower smoothness, i.e., a small $M$ or has a very small PL parameter.
Thus, gradients may change tremendously at some steps especially in the presence of privacy noise.
Such changes have more critical impact when only a less number of iterations can be executed due to the privacy constraint.
Assume $\alpha$ is some constant while $\kappa \gg 1/\alpha$, we immediately get:
\begin{align*}
\operatorname{ERUB}^{\text{uniform}}_{\min} &= \Theta \left(\kappa \ln \left(1 + {1 \over \kappa \alpha} \right)\right) = \Theta\left({1\over \alpha}\right) \le \cO\left( {MN^2R\over DG^2}\right), \\ ~ \operatorname{ERUB}^{\text{dynamic}}_{\min} &= \Theta(1).
\end{align*}
Both are robust, but the dynamic schedule has a smaller factor since $1/\alpha$ could be a large number.
In addition, the factor implies that when more samples are used, the dynamic schedule is more robust.

\subsection{Gradient Descent Methods with Momentum}

\cref{sec:gd_method} shows that the step noise has an exponentially increasing influence on the final loss, and therefore a decreasing noise magnitude improves the utility upper bound by a $\ln N$ factor.
However, the proper schedule can be hard to find when the curvature information, e.g., $\kappa$, is absent.
A parameterized method that less depends on the curvature information is preferred.
On the other hand, long-term iterations will result in forgetting of the initial iterations, since accumulated noise overwhelmed the propagated information from the beginning.
This effect will reduce the efficiency of the recursive learning frameworks.

Alternative to GD, the momentum method can mitigate the two issues.
It was originally proposed to stabilize the gradient estimation \citep{polyak1964methods}.
In this section, we show that momentum (agnostic about the curvature) can flatten the dynamic influence and improve the utility upper bound.
Previously, \citeauthor{pichapati2019adaclip} used the momentum as an estimation of gradient mean, without discussions of convergence improvements.
\citeauthor{zhou2020private} gave a bound for the Adam with DP.
However, the derivation is based on gradient norm, which results in a looser bound (see \cref{tbl:compare_bd}).
The momentum method stabilizes gradients by moving average history coordinate values and thus greatly reduces the variance.
The $\phi(m_t, g_t)$ can be rewritten as:
\begin{align}
m_{t+1} &= \phi(m_t, g_t) = {v_{t+1} \over 1-\beta^t}, \notag \\
v_{t+1} &= {\beta v_{t} + (1- \beta) g_t} = (1- \beta) \sum\nolimits_{i=1}^t \beta^{t-i} g_t, ~ v_1 = 0, \label{eq:simp_mom_v_2} %
\end{align}
where $\beta \in [0,1]$.
Note $v_{t+1}$ is a biased estimation of the gradient expectation while $m_{t+1}$ is unbiased.

\begin{theoremEnd}{theorem}[Convergence under PL condition]\stepThmRelabel
\label{thm:excess_loss_momentum}
	Suppose $f(\theta; x_i)$ is $G$-Lipschitz, and $f(\theta)$ is $M$-smooth and satisfies the Polyak-Lojasiewicz condition.
	Assume $\beta \neq \gamma$ and $\beta \in (0, 1)$.
	Let $\eta_t = {\eta_0\over 2M}$ and $\eta_0 \le 8 \left( \sqrt{1 + {64 \beta\gamma (\gamma - \beta)^{-2} (1 - \beta)^{-3}}} + 1 \right)^{-1}$.
	Then the following holds:
	\begin{gather}
	\operatorname{EER} \le \big( \gamma^T + 2R\eta_0\alpha \underbrace{ U_3(\sigma, T) }_{\text{noise varinace}} \big) (f(\theta_1) - f(\theta^*)) \notag \\
	- \zeta {\eta_0 \over 2 M} \underbrace{ \sum \nolimits_{t=1}^T \gamma^{T-t} \Ebb \norm{v_{t+1}}^2 }_{\text{momentum effect}} \label{eq:excess_loss_momentum_1} %
	\end{gather}
	where $\gamma=1 - {\eta_0 \over \kappa}$, $\zeta = 1 - {1 \over \beta (1 - \beta)^3} \eta_0^2 - {1\over 4} \eta_0 \ge 0$, and $U_3 = \sum\nolimits_{t=1}^T \gamma^{T-t} {(1 - \beta)^2 \over (1-\beta^t)^2} \sum \nolimits_{i=1}^t \beta^{2(t-i)} \sigma_i^2$.
\end{theoremEnd}
\begin{proofEnd}
	Without loss of generality, we absorb the $C\sigma_t/N$ into the variance of $\nu_t$ such that $\nu_t \sim \mathcal{N}(0, {C \sigma_t^2\over N} I)$ and $g_t \leftarrow \nabla_t + \nu_t$.
	Define $b_t = 1 - \beta^t$.

	By smoothness and \cref{eq:private_iter}, we have
	\begin{align}
	f(\theta_{t+1}) - f(\theta_t) &\le \nabla_t^\top (\theta_{t+1} - \theta_t) + {1\over 2} M \norm{\theta_{t+1} - \theta_t}^2 \notag \\
	&= - {\eta_t\over b_t^2} b_t \nabla_t^\top v_{t+1} + {1\over 2} M {\eta_t^2 \over b_t^2} \norm{v_{t+1}}^2 \notag \\
	&= {\eta_t \over b_t^2} \left( \norm{b_t \nabla_t - v_{t+1} }^2 - \norm{b_t \nabla_t}^2 - \norm{v_{t+1}}^2 \right) + {1\over 2} M {\eta_t^2 \over b_t^2} \norm{v_{t+1}}^2 \notag \\
	&= {\eta_t \over b_t^2} \underbrace{ \norm{b_t \nabla_t - v_{t+1} }^2 }_{U_1(t)} - \eta_t \norm{ \nabla_t}^2  - {\eta_t\over b_t^2}  (1 - {1\over 2} M \eta_t )  \norm{v_{t+1}}^2,  \label{eq:mom:fthetat+1-ftheta_t_annoted}
	\end{align}
	where only the $U_1(t)$ is non-negative.
	Specifically, $U_1(t)$ describes the difference between current gradient and the average.
	We can expand $v_{t+1}$ to get an upper bound:
	\begin{align*}
	U_1(t) &= \norm{b_t \nabla_t - v_{t+1} }^2  \\
	&= \norm{(1 - \beta) \sum \nolimits_{i=1}^t \beta^{t-i} \nabla_t - (1 - \beta) \sum \nolimits_{i=1}^t \beta^{t-i} g_i }^2 \\
	&= (1 - \beta)^2 \norm{\sum \nolimits_{i=1}^t \beta^{t-i} (\nabla_t - g_i) }^2  \\
	&= (1 - \beta)^2 \norm{\sum \nolimits_{i=1}^t \beta^{t-i} (\nabla_t - \nabla_i) + \sum \nolimits_{i=1}^t \beta^{t-i} (\nabla_i - g_i) }^2  \\
	&\le 2 (1 - \beta)^2 \left[ \norm{\sum \nolimits_{i=1}^t \beta^{t-i} (\nabla_t - \nabla_i) }^2 + \norm{\sum \nolimits_{i=1}^t \beta^{t-i} (\nabla_i - g_i) }^2 \right] \\
	&\le 2 (1 - \beta) \left[ b_t \underbrace{ \sum \nolimits_{i=1}^t \beta^{t-i} \norm{\nabla_t - \nabla_i}^2}_{U_2(t)\text{ (gradient variance)}} + (1 - \beta) \underbrace{ \norm{\sum \nolimits_{i=1}^t \beta^{t-i} \nu_i }^2 }_{\text{noise variance}} \right]
	\end{align*}
	where we use $\norm{x + y}^2 \le (\norm{x} + \norm{y})^2 \le 2(\norm{x}^2 + \norm{y}^2)$.
	The last inequality can be proved by Cauchy-Schwartz inequality for each coordinate.

	We plug the $U_1(t)$ into \cref{eq:mom:fthetat+1-ftheta_t_annoted} and use the PL condition to get
	\begin{align*}
	f(\theta_{t+1}) - f(\theta_t) &\le {\eta_t \over b_t^2} U_1(t) - \eta_t \norm{ \nabla_t}^2  - {\eta_t\over b_t^2}  (1 - {1\over 2} M \eta_t )  \norm{v_{t+1}}^2 \\
	&\le  - \eta_t \norm{ \nabla_t}^2 + {\eta_t \over b_t^2} 2 (1 - \beta) \left[ b_t U_2(t) + (1 - \beta) \norm{\sum \nolimits_{i=1}^t \beta^{t-i} \nu_i }^2 \right] \\
	&\quad  - {\eta_t\over b_t^2}  (1 - {1\over 2} M \eta_t )  \norm{v_{t+1}}^2 \\
	&\le  - 2 \mu \eta_t (f(\theta_{t}) - f(\theta^*))  + {2 (1 - \beta) \eta_t \over b_t} U_2(t) + {2 (1 - \beta)^2\eta_t \over b_t^2} \norm{\sum \nolimits_{i=1}^t \beta^{t-i} \nu_i }^2 \\
	&\quad  - {\eta_t\over b_t^2}  (1 - {1\over 2} M \eta_t )  \norm{v_{t+1}}^2.
	\end{align*}
	Rearranging terms and taking expectation to show
	\begin{align*}
	\Ebb[f(\theta_{t+1})] - f(\theta^*) &\le \gamma (\Ebb[f(\theta_{t})] - f(\theta^*)) + {2 (1 - \beta)^2\eta_t \over b_t^2} \sum \nolimits_{i=1}^t \beta^{t-i} \Ebb \norm{\nu_i }^2 \\
	&\quad + {2 (1 - \beta) \eta_t \over b_t} \Ebb [ U_2(t) ] - {\eta_t\over b_t^2}  (1 - {1\over 2} M \eta_t ) \Ebb \norm{v_{t+1}}^2 \\
	&= \gamma (\Ebb[f(\theta_{t})] - f(\theta^*)) + {2 (1 - \beta)^2\eta_t \over b_t^2} \sum \nolimits_{i=1}^t \beta^{2(t-i)} {C^2 D \sigma_t^2 \over N^2} \\
	&\quad + {2 (1 - \beta) \eta_t \over b_t} \Ebb [ U_2(t) ] - {\eta_t\over b_t^2}  (1 - {1\over 2} M \eta_t ) \Ebb \norm{v_{t+1}}^2
	\end{align*}
	where $\gamma = 1 - \eta_0 / \kappa = 1 - 2 \mu \eta_t$.
	The recursive inequality implies
	\begin{align*}
	\Ebb [f(\theta_{T+1})] - f(\theta^*) &\le \gamma^T (f(\theta_{1}) - f(\theta^*)) + \sum_{t=1}^T \gamma^{T-t} {2 (1 - \beta)^2\eta_t \over b_t^2} \sum \nolimits_{i=1}^t \beta^{2(t-i)} {C^2 D \sigma_t^2 \over N^2} \\
	&\quad + \sum_{t=1}^T \gamma^{T-t} {2 (1 - \beta) \eta_t \over b_t} \Ebb [ U_2(t) ] - \sum_{t=1}^T \gamma^{T-t} {\eta_t\over b_t^2}  (1 - {1\over 2} M \eta_t ) \Ebb \norm{v_{t+1}}^2 \\
	&= \bigg( \gamma^T + 2 \eta_0 \alpha R \underbrace{ \sum \nolimits_{t=1}^T \gamma^{T-t} {(1 - \beta)^2 \over b_t^2} \sum \nolimits_{i=1}^t \beta^{2(t-i)} \sigma_t^2 }_{U_3} \bigg) (f(\theta_{1}) - f(\theta^*)) \\
	&\quad + \underbrace{ \sum_{t=1}^T \gamma^{T-t} {2 (1 - \beta) \eta_t \over b_t} \Ebb [ U_2(t) ] - \sum_{t=1}^T \gamma^{T-t} {\eta_t\over b_t^2}  (1 - {1\over 2} M \eta_t ) \Ebb \norm{v_{t+1}}^2 }_{U_4(t)}.
	\end{align*}
	where we utilize $\alpha = {D C^2 \over 2 M N^2 R} {1\over f(\theta_1) - f(\theta^*)}$ and $\eta_t = {\eta_0\over 2M}$.
	
	By \cref{lm:ineq:sum_of_prop_mom_x1}, we have
	\begin{align*}
	\sum_{t=1}^T \gamma^{T-t} {2 (1 - \beta) \eta_t \over b_t} U_2(t) &\le {\eta^3_0 \beta \gamma \over 2 M (1 - \beta)^3(\gamma - \beta)^2} \sum_{i=1}^{T-1} \gamma^{T-i} \norm{v_{i+1}}^2.
	\end{align*}
	Thus, by ${1\over b_t} \ge 1$,
	\begin{align*}
	U_4(t)  &\le {\eta^3_0 \beta \gamma \over 2 M (1 - \beta)^3(\gamma - \beta)^2} \sum_{i=1}^{T-1} \gamma^{T-i} \Ebb \norm{v_{i+1}}^2 - {\eta_0\over 2M} (1 - {\eta_0\over 4} ) \sum_{t=1}^T \gamma^{T-t} \Ebb \norm{v_{t+1}}^2  \\
	&= - {\eta_0\over 2M} \zeta \sum_{t=1}^T \gamma^{T-t} \Ebb \norm{v_{t+1}}^2
	\end{align*}
	where
	\begin{align*}
	\zeta = 1 - {1 \over 4} \eta_0 - {\beta \gamma \over (\gamma - \beta)^2 (1 - \beta)^3} \eta^2_0 = 1 - {1 \over 4} \eta_0 - {\beta/ \gamma \over (1 - \beta/\gamma)^2 (1 - \beta)^3} \eta^2_0 %
	\end{align*}
	When a small enough $\eta_0$, e.g.,
	Specifically,
	\begin{align*}
	\eta_0 &\le {(\gamma - \beta)^2 (1 - \beta)^3 \over 8 \beta \gamma} \left[ \sqrt{1 + {64 \beta\gamma \over (\gamma - \beta)^2 (1 - \beta)^3}} - 1 \right] \\
	&= {8 \over \sqrt{1 + {64 \beta\gamma (\gamma - \beta)^{-2} (1 - \beta)^{-3}}} + 1}
	\end{align*}
	We can have $\zeta \ge 0$.
	
	By the definition of $U_3(T, \sigma)$, we can get
	\begin{align*}
	\Ebb [f(\theta_{T+1})] - f(\theta^*) &\le \left( \gamma^T + 2 \eta_0 \alpha R U_3(T, \sigma) \right) (f(\theta_{1}) - f(\theta^*)) - {\eta_0\over 2M} \zeta \sum_{t=1}^T \gamma^{T-t} \Ebb \norm{v_{t+1}}^2.
	\end{align*}

\end{proofEnd}

The upper bound includes three parts that influence the bound differently:
\underline{\emph{(1) Convergence.}} The convergence term is mainly determined by $\eta_0$ and $\kappa$.
$\eta_0$ should be in $(0, \kappa)$ such that the upper bound can converge.
A large $\eta_0$ will be preferred to speed up convergence if it does not make the other two terms worse.
\underline{\emph{(2) Noise Variance.}}
The second term compressed in $U_3$ is the effect of the averaged noise, $\sum \nolimits_{i=1}^t \beta^{2(t-i)} \sigma_i^2$.
One difference introduced by the momentum is the factor $(1-\beta)/(1-\beta^t)$ which is less than $\gamma^t$ at the beginning and converges to a non-zero constant $1-\beta$.
Therefore, in $U_3$, $\gamma^{T-t}(1-\beta)/(1-\beta^t)$ will be constantly less than $\gamma^T$ meanwhile.
Furthermore, when $t>\hat T$, the moving average $\sum \nolimits_{i=1}^t \beta^{2(t-i)} \sigma_i^2$ smooths the influence of each $\sigma_t$. %
\underline{\emph{(3) Momentum Effect.}}
The momentum effect term can improve the upper bound when $\eta_0$ is small.
For example, when $\beta = 0.9$ and $\gamma=0.99$, then $\eta_0 \le 0.98/M$ which is a rational value.
Following the analysis, when $M$ is large which means the gradient norms will significantly fluctuate, the momentum term may take the lead. Adjusting the noise scale in this case may be less useful for improving utility.

To give an insight on the effect of dynamic schedule, we provide the following utility bounds.
\begin{theoremEnd}{theorem}[Uniform schedule]\stepThmRelabel
\label{thm:mom:uniform_sch_bd}
Suppose the assumptions in \cref{thm:excess_loss_momentum} are satisfied.
Let $\sigma_t^2 = T/R$, and let:
$\hat T= \max t~ \st ~\gamma^{t-1} \ge {1-\beta \over 1- \beta^t}, ~ T = \left\lceil \cO \left( {\kappa\over \eta_0} \ln \left( 1 + {\eta_0 \over \kappa \alpha } \right)  \right) \right\rceil$.
Given some positive constant $c$ and $\alpha_0>0$ with $1/\alpha > 1/ \alpha_0$, the following inequality holds:
\begin{align*}
	\operatorname{ERUB}_{\min} \le \cO \left({\kappa^2 \over \kappa + \eta_0 / \alpha} \left[ \II_{T \le \hat T} + \gamma^{\hat T-1} \ln \left( 1 + {\eta_0 \over \kappa \alpha } \right) \II_{T>\hat T} \right] \right).
\end{align*}
\end{theoremEnd}
\begin{proofEnd}
	Since $\sigma_t$ is static, by definition of $U_3$ in \cref{thm:excess_loss_momentum},
	\begin{align*}
	U_3 &= \sum \nolimits_{t=1}^T \gamma^{T-t} {(1 - \beta)^2 \over (1 - \beta^t)^2} \sum \nolimits_{i=1}^t \beta^{2(t-i)} \sigma^2 \\
	&= \sigma^2 \sum \nolimits_{t=1}^T \gamma^{T-t} {(1 - \beta)^2 \over (1 - \beta^t)^2} \sum \nolimits_{i=1}^t \beta^{2(t-i)} \\
	&= \sigma^2 \sum \nolimits_{t=1}^T \gamma^{T-t} {(1 - \beta)^2 \over (1 - \beta^t)^2} {1 - \beta^{2t} \over 1 - \beta^2} \\
	&= \sigma^2 \sum \nolimits_{t=1}^T \gamma^{T-t} {1 - \beta \over 1 - \beta^t} {1 + \beta^{t} \over 1 + \beta}. %
	\end{align*}
	Because ${1 - \beta \over 1 - \beta^t} {1 + \beta^{t} \over 1 + \beta} \le 1$, the $U_3$ will be smaller than the corresponding summation in GD with uniform schedule.
	
	By \cref{lm:beta_gamma_cond}, when $T>\hat T$, we can rewrite $U_3$ as
	\begin{align*}
		U_3 &\le \sigma^2 \sum \nolimits_{t=1}^T \gamma^{T-t} {1 - \beta \over 1 - \beta^t} \\
		&= \sigma^2 \sum \nolimits_{t=1}^{\hat T} \gamma^{T-t} {1 - \beta \over 1 - \beta^t} + \sigma^2 \sum \nolimits_{t=\hat T+1}^T \gamma^{T-t} {1 - \beta \over 1 - \beta^t} \\
		&\le \sigma^2 \sum \nolimits_{t=1}^{\hat T} \gamma^{T-t} \gamma^{t-1} + \sigma^2 \sum \nolimits_{t=\hat T+1}^T \gamma^{T-t} \gamma^{\hat T-1} \\
		&= \sigma^2 \gamma^{T-1} \hat T + \sigma^2 \gamma^{\hat T-1} \sum \nolimits_{t=1}^{T-\hat T} \gamma^{T- \hat T -t} \\
		&= \sigma^2 \gamma^{T-1} \hat T + \sigma^2 {\gamma^{\hat T-1} - \gamma^{T-1} \over 1 - \gamma} \\
		&= {T\over \gamma R} \gamma^{T} \left( \hat T + {\gamma^{\hat T- T} - 1 \over 1 - \gamma} \right)
	\end{align*}
	where we use $\sigma^2 = T/R$ in the last line.
	Without assuming $T> \hat T$, we can generally write the upper bound as
	\begin{align*}
		U_3 &\le {T\over \gamma R} \gamma^{T} \left( \min\{\hat T, T\} + \max\{ {\gamma^{\hat T- T} - 1 \over 1 - \gamma} , 0\} \right).
	\end{align*}

	By \cref{thm:excess_loss_momentum}, because $\zeta \ge 0$, we have
	\begin{align*}
	\operatorname{ERUB} &\le \gamma^T + 2 R \eta_0 \alpha U_3 \\
	&= \gamma^T (1 + \frac{\alpha'}{\gamma} T \left(  \min\{\hat T, T\} + \max\{ {\gamma^{\hat T- T} - 1 \over 1 - \gamma} , 0\} \right)) 
	\end{align*}
	where $\alpha'=2 \eta_0 \alpha$.

	First, we consider $T\le \hat T$.
	Use $T = {1\over \ln (1/\gamma)} \ln \left( 1 + {\eta_0 \over \kappa \alpha } \right) = \left\lceil \cO \left( {\kappa\over \eta_0} \ln \left( 1 + {\eta_0 \over \kappa \alpha } \right) \right) \right\rceil$ to get
	\begin{align*}
	\operatorname{ERUB} &\le \left( {\alpha \over \alpha + \eta_0 /\kappa} \right) \left( 1 + \alpha' \gamma^{-1} ({2\over \ln(1/\gamma)} \ln \left( 1 + {\eta_0 \over \kappa \alpha } \right))^2 \right) \\
	&\le \left( {\alpha \over \alpha + \eta_0 /\kappa} \right) \left( 1 + {8 \kappa^2 \alpha \over \eta_0 \gamma} \ln^2 \left( 1 + {\eta_0 \over \kappa \alpha } \right) \right) \\
	&\le \cO\left( {\kappa \over \kappa + \eta_0/ \alpha} \left( 1 + {8\kappa^2 \alpha \over \eta_0 \gamma} \ln^2 \left( 1 + {\eta_0 \over \kappa \alpha } \right) \right) \right) \\
	&= \cO\left( {\kappa \over \kappa + \eta_0 / \alpha} \left( 1 + {4\kappa \over \gamma} \right) \right) \\
	&= \cO\left( {\kappa^2 \over \kappa + \eta_0 / \alpha} \right)
	\end{align*}
	where we used $\ln(1/\gamma) \ge \eta_0/\kappa$ and $\ln(1+x) \le \sqrt{x}$ for any $x>0$.

	Second, when $T > \hat T$, 
	\begin{align*}
		\operatorname{ERUB} &\le \gamma^T (1 + \frac{\alpha'}{\gamma} T \left(  \hat T + {\gamma^{\hat T- T} - 1 \over 1 - \gamma} \right)) \\
		&\le \cO \left(\gamma^T + \frac{2\alpha'}{\gamma} T \kappa  (\gamma^{\hat T} - \gamma^T) \right).
	\end{align*}
	Make use of $T = \left\lceil {1\over \ln (1/\gamma)} \ln \left( 1 + {\eta_0 \over \kappa \alpha } \right) \right\rceil$ to show
	\begin{align*}
		\operatorname{ERUB} &\le \cO \left({\kappa \over \kappa + \eta_0 / \alpha} + \frac{4 \kappa^2 \alpha}{\eta_0 \gamma} (\gamma^{\hat T} - {\kappa \over \kappa + \eta_0 / \alpha} ) \ln \left( 1 + {\eta_0 \over \kappa \alpha } \right) \right) \\
		&\le \cO \left({\kappa^2 \over \kappa + \eta_0 / \alpha} \gamma^{\hat T-1} \ln \left( 1 + {\eta_0 \over \kappa \alpha } \right) \right).
	\end{align*}
\end{proofEnd}

\begin{theoremEnd}{theorem}[Dynamic schedule]\stepThmRelabel
	\label{thm:mom:dyn_sch_bd}
	Suppose the assumptions in \cref{thm:excess_loss_momentum} are satisfied.
	Let $\alpha' = {2\eta_0 \alpha \over \gamma (1 - \gamma \beta^2)}$, $\beta < \gamma$ and $\hat T= \max t~ \st ~\gamma^{t-1} \ge {1-\beta \over 1- \beta^t}$.
	Use the following schedule:
		$\sigma_t^2 = {1\over R} \sum\nolimits_{i=1}^T \sqrt{q_i\over q_t},~ T^{\text{dyn}} = \left\lceil \cO \left( {2\kappa \over \eta_0} \ln \left( 1 + {\eta_0 \over \kappa \alpha } \right) \right) \right\rceil,$
	where $q_t = c_1 \gamma^{T+t} \II_{T\le \hat T} + \gamma^{\hat T-1} c_2 \gamma^{T-t} \II_{T> \hat T}$ for some positive constants $c_1$ and $c_2$.
	The following inequality holds:
	\begin{align*}
		\operatorname{ERUB} &\le \gamma^T + 2\eta_0 \alpha \sum \nolimits_{t=1}^T R q_t \sigma_t^2 ,\\ ~ \operatorname{ERUB}_{\min} &\le \cO \left({ \kappa \alpha \over \kappa \alpha + \eta_0} \left( { \kappa \alpha \over \kappa \alpha + \eta_0} \II_{T\le \hat T} + \II_{T>\hat T} \right) \right).
	\end{align*}
\end{theoremEnd}
\begin{proofEnd}
	By \cref{lm:beta_gamma_cond}, we can rewrite $U_3$ as
	\begin{align*}
		U_3 &= \sum \nolimits_{t=1}^T \gamma^{T-t} {(1 - \beta)^2 \over (1 - \beta^t)^2} \sum \nolimits_{i=1}^t \beta^{2(t-i)} \sigma_i^2 \\
		&\le \sum \nolimits_{t=1}^{\hat T} \gamma^{T-t} \gamma^{2(t-1)} \sum \nolimits_{i=1}^t \beta^{2(t-i)} \sigma_i^2 + \sum \nolimits_{t=\hat T+1}^T \gamma^{T-t} \gamma^{2(\hat T-1)} \sum \nolimits_{i=1}^t \beta^{2(t-i)} \sigma_i^2 \\
		&\le \gamma^{T - \hat T} \underbrace{ \sum \nolimits_{t=1}^{\hat T} \gamma^{\hat T-t} \gamma^{2(t-1)} \sum \nolimits_{i=1}^t \beta^{2(t-i)} \sigma_i^2 }_{V_1} + \gamma^{2(\hat T-1)} \underbrace{\sum \nolimits_{t=\hat T + 1}^{T} \gamma^{T-t} \sum \nolimits_{i=1}^t \beta^{2(t-i)} \sigma_i^2 }_{V_2}
	\end{align*}
	We derive $V_1$ and $V_2$ separately.

	For $V_1$, we can obtain the upper bound by
	\begin{align*}
		V_1 &= \sum \nolimits_{t=1}^{\hat T} \gamma^{\hat T-t} \gamma^{2(t-1)} \sum \nolimits_{i=1}^t \beta^{2(t-i)} \sigma_i^2 \\
		&= \gamma^{\hat T-2} \sum \nolimits_{t=1}^{\hat T} \gamma^{t} \sum \nolimits_{i=1}^t \beta^{2(t-i)} \sigma_i^2 \\
		&= \gamma^{\hat T-2} \sum \nolimits_{i=1}^{\hat T} \beta^{-2i} \sigma_i^2 \sum \nolimits_{t=i}^{\hat T} \left(\gamma \beta^2\right)^{t} \\
		&= \gamma^{\hat T-2} \sum \nolimits_{i=1}^{\hat T} \beta^{-2i} \sigma_i^2 {\left(\gamma \beta^2\right)^i - \left(\gamma \beta^2\right)^{\hat T+1} \over 1 - \gamma \beta^2} \\
		&= \gamma^{2\hat T-3} \sum \nolimits_{i=1}^{\hat T} {\gamma^{i-\hat T-1} - \beta^{2(\hat T+1-i)} \over 1 - \gamma \beta^2} \sigma_i^2 \\
		&= \gamma^{2\hat T-3} \sum \nolimits_{i=1}^{\hat T} {1 - (\gamma\beta^2)^{\hat T+1-i} \over 1 - \gamma \beta^2} \gamma^{i-\hat T-1} \sigma_i^2 \\
		&\le {\gamma^{\hat T} \over \gamma^2 (1 - \gamma \beta^2)} \sum \nolimits_{i=1}^{\hat T} \gamma^i \sigma_i^2 \\
		&\le {\gamma^{\hat T} \over \gamma (\gamma - \beta^2)} \sum \nolimits_{i=1}^{\hat T} \gamma^i \sigma_i^2
	\end{align*}

	For $V_2$, we can derive
	\begin{align*}
		V_2 &= \sum \nolimits_{t=\hat T + 1}^{T} \gamma^{T-t} \sum \nolimits_{i=1}^t \beta^{2(t-i)} \sigma_i^2 \\
		&= \sum \nolimits_{t=1}^{T} \gamma^{T-t} \sum \nolimits_{i=1}^t \beta^{2(t-i)} \sigma_i^2 - \sum \nolimits_{t=1}^{\hat T} \gamma^{T-t} \sum \nolimits_{i=1}^t \beta^{2(t-i)} \sigma_i^2 \\
		&= \sum \nolimits_{t=1}^{T} \gamma^{T-t} \sum \nolimits_{i=1}^t \beta^{2(t-i)} \sigma_i^2 - \gamma^{T-\hat T} \sum \nolimits_{t=1}^{\hat T} \gamma^{\hat T-t} \sum \nolimits_{i=1}^t \beta^{2(t-i)} \sigma_i^2.
	\end{align*}
	We first consider the first term
	\begin{align*}
		&\quad \sum \nolimits_{t=1}^{T} \gamma^{T-t} \sum \nolimits_{i=1}^t \beta^{2(t-i)} \sigma_i^2 \\
		&= \sum \nolimits_{i=1}^T \sigma_i^2 \sum \nolimits_{t=i}^{T} \gamma^{T-t} \beta^{2(t-i)} \\
		&= \sum \nolimits_{i=1}^T \gamma^T \beta^{-2i} \sigma_i^2 \sum \nolimits_{t=i}^{T} \gamma^{-t} \beta^{2t} \\
		&= \sum \nolimits_{i=1}^T \gamma^T \beta^{-2i} \sigma_i^2 {(\beta^2/\gamma)^i - (\beta^2/\gamma)^{T + 1} \over 1 - (\beta^2/\gamma)} \\
		&= \sum \nolimits_{i=1}^T {\gamma^{T+1-i} - \beta^{2(T +1 - i)} \over \gamma - \beta^2} \sigma_i^2.
	\end{align*}
	Similarly, we have
	\begin{align*}
		&\quad \gamma^{T-\hat T} \sum \nolimits_{t=1}^{\hat T} \gamma^{\hat T-t} \sum \nolimits_{i=1}^t \beta^{2(t-i)} \sigma_i^2 \\
		&= \gamma^{T-\hat T} \sum \nolimits_{i=1}^{\hat T} {\gamma^{\hat T+1-i} - \beta^{2(\hat T +1 - i)} \over \gamma - \beta^2} \sigma_i^2 \\
		&= \sum \nolimits_{i=1}^{\hat T} {\gamma^{T+1-i} - \gamma^{T-\hat T} \beta^{2(\hat T +1 - i)} \over \gamma - \beta^2} \sigma_i^2.
	\end{align*}
	Thus, 
	\begin{align*}
		V_2 &= \sum \nolimits_{i=1}^T {\gamma^{T+1-i} - \beta^{2(T +1 - i)} \over \gamma - \beta^2} \sigma_i^2 - \sum \nolimits_{i=1}^{\hat T} {\gamma^{T+1-i} - \gamma^{T-\hat T} \beta^{2(\hat T +1 - i)} \over \gamma - \beta^2} \sigma_i^2 \\
		&= \sum \nolimits_{i=\hat T + 1}^T {\gamma^{T+1-i} - \beta^{2(T +1 - i)} \over \gamma - \beta^2} \sigma_i^2 + \sum \nolimits_{i=1}^{\hat T} {\gamma^{T-\hat T} - \beta^{2(T-\hat T)} \over \gamma - \beta^2} \beta^{2(\hat T +1 - i)} \sigma_i^2 \\
		&\le \sum \nolimits_{i=\hat T + 1}^T {\gamma^{T+1-i} - \beta^{2(T +1 - i)} \over \gamma - \beta^2} \sigma_i^2 + \sum \nolimits_{i=1}^{\hat T} {\gamma^{T-\hat T} \over \gamma - \beta^2} \beta^{2(\hat T +1 - i)} \sigma_i^2.
	\end{align*}
	Substitute $V_1$ and $V_2$ into $U_3$ to get
	\begin{align*}
		U_3 &\le \gamma^{T} {1 \over \gamma (\gamma - \beta^2)} \sum \nolimits_{i=1}^{\hat T} \gamma^i \sigma_i^2 + \gamma^{2\hat T - 2} \sum \nolimits_{i=\hat T + 1}^T {\gamma^{T+1-i} - \beta^{2(T +1 - i)} \over \gamma - \beta^2} \sigma_i^2 \\
		&\quad + \sum \nolimits_{i=1}^{\hat T} {\gamma^{T+\hat T-2} \over \gamma - \beta^2} \beta^{2(\hat T +1 - i)} \sigma_i^2 \\
		&\le \left( {\gamma^T \over \gamma (\gamma - \beta^2)} \sum \nolimits_{i=1}^{\hat T} (\gamma^i + \gamma^{\hat T - 1} \beta^{2(\hat T +1 - i)}) \sigma_i^2 + \gamma^{2\hat T - 2} \sum \nolimits_{i=\hat T + 1}^T {\gamma^{T+1-i} - \beta^{2(T +1 - i)} \over \gamma - \beta^2} \sigma_i^2 \right) \\
		&\le \left( {2\gamma^T \over \gamma (\gamma - \beta^2)} \sum \nolimits_{i=1}^{\hat T} \gamma^i \sigma_i^2 + \gamma^{2\hat T - 2} \sum \nolimits_{i=\hat T + 1}^T {\gamma^{T+1-i} - \beta^{2(T +1 - i)} \over \gamma - \beta^2} \sigma_i^2 \right) \\
		&= \sum \nolimits_{t=1}^{T} q_t \sigma_t^2
	\end{align*}
	where
	\begin{align*}
		q_t &= {2 \over \gamma (\gamma - \beta^2)} \gamma^{T+t} \II_{T\le \hat T} + \gamma^{2(\hat T-1)} {\gamma^{T+1-i} - \beta^{2(T +1 - i)} \over \gamma - \beta^2} \gamma^{T-t} \II_{T> \hat T} \\
		&\le c_1 \gamma^{T+t} \II_{T\le \hat T} + \gamma^{\hat T-1} c_2 \gamma^{T-t} \II_{T> \hat T}
	\end{align*}
	where $c_1 = {2 \over \gamma (\gamma - \beta^2)}$ and $c_2 = {\gamma^{2\hat T} \over \gamma - \beta^2}$.

	When $T> \hat T$, by \cref{thm:cachy-scharz_weighted_sum}, the lower bound of $R \sum \nolimits_{t=1}^T q_t \sigma_t^2$ is
	\begin{align*}
	\left(\sum\nolimits_{t=1}^T \sqrt{q_t} \right)^2 &= \gamma^T \left( \sum\nolimits_{t=1}^{\hat T} \sqrt{c_1 \gamma^t} + \sum\nolimits_{t=\hat T + 1}^T \sqrt{\gamma^{\hat T-1} c_2 \gamma^{-t}} \right)^2 \\
	&= \gamma^T \left( \sqrt{c_1 \gamma} \frac{1 - \gamma^{\hat T/2}}{1 - \sqrt{\gamma}} + \sqrt{c_2} \frac{1 - \gamma^{(\hat T- T - 1)/2}}{\sqrt{\gamma} - 1} \right)^2 \\
	&= \gamma^T \left( \sqrt{c_1 \gamma} \frac{1 - \gamma^{\hat T/2}}{1 - \sqrt{\gamma}} + \sqrt{c_2} \frac{\gamma^{(\hat T- T - 1)/2} - 1}{1 - \sqrt{\gamma}} \right)^2 \\
	&\le \cO \left( c_2 \left\{ \frac{\gamma^{(\hat T - 1)/2} - \gamma^{T/2}}{1 - \sqrt{\gamma}} \right\}^2 \right)
	\end{align*}
	which is achieved when 
	\begin{align*}
	\sigma_t^2 = {1\over R} \sum\nolimits_{i=1}^T \sqrt{q_i\over q_t}.
	\end{align*}

	By \cref{thm:excess_loss_momentum}, because $\zeta \ge 0$, we have
	\begin{align*}
	\operatorname{ERUB} &\le \gamma^T + 2 R \eta_0 \alpha U_3 \\
	&= \gamma^T + 2\eta_0 \alpha \sum \nolimits_{t=1}^T R q_t \sigma_t^2.
	\end{align*}
	And the minimum of the upper bound is
	\begin{align*}
	\operatorname{ERUB}_{\min} &= \gamma^T + \alpha' \cO \left( \left\{ \frac{\gamma^{(\hat T - 1)/2} - \gamma^{T/2}}{1 - \sqrt{\gamma}} \right\}^2 \right)
	\end{align*}
	where $\alpha'={2\eta_0 c_2 \alpha}$.
	Let $T = {2\over \ln(1/\gamma)} \ln \left( 1 + { \eta_0 \over \kappa \alpha } \right)$.
	Then,
	\begin{align*}
		\operatorname{ERUB}_{\min} &= \cO\left( \left({ \kappa \alpha \over \kappa \alpha + \eta_0} \right)^2 + {\alpha' \over (1 - \sqrt{\gamma})^2} \left\{ { {\gamma^{(\hat T - 1)/2} - (1-\gamma^{(\hat T - 1)/2})\kappa \alpha \over \kappa \alpha + \eta_0} } \right\}^2 \right) \\
		&\le \cO\left( \left({ \kappa \alpha \over \kappa \alpha + \eta_0} \right)^2 + {2\eta_0 c_2 \alpha \over (1 - \sqrt{\gamma})^2} \left\{ { {\gamma^{(\hat T - 1)/2} \over \kappa \alpha + \eta_0} } \right\}^2 \right) \\
		&= \cO\left( { \kappa \alpha \over (\kappa \alpha + \eta_0)^2 } \left(\kappa \alpha + {2\eta_0 c_2 /\kappa \over (1 - \sqrt{\gamma})^2} \gamma^{(\hat T - 1)} \right) \right)  \\
		&= \cO\left( { \kappa \alpha \over (\kappa \alpha + \eta_0)^2 } \left(\kappa \alpha + c_3 \eta_0 \right) \right) \\
		&\le \cO\left( { \kappa \alpha \over \kappa \alpha + \eta_0 } \right)
	\end{align*}
	where $c_3$ is some constant.

	When $T\le \hat T$,
	\begin{align*}
		U_3 &\le \gamma^{T - T} \underbrace{ \sum \nolimits_{t=1}^{T} \gamma^{T-t} \gamma^{2(t-1)} \sum \nolimits_{i=1}^t \beta^{2(t-i)} \sigma_i^2 }_{V_1} \\
		&\le {\gamma^{T-2} \over 1 - \gamma \beta^2} \sum \nolimits_{i=1}^{T} \gamma^i \sigma_i^2
	\end{align*}
	with which we obtain
	\begin{align*}
		\operatorname{ERUB} &\le \gamma^T + 2 R \eta_0 \alpha U_3 \\
		&\le \gamma^T + 2 \eta_0 \alpha {\gamma^{-2} \over 1 - \gamma \beta^2} \sum \nolimits_{t=1}^{T} R q_t \sigma_t^2.
	\end{align*}
	where we let $q_t=\gamma^{T+t}$.
	By \cref{thm:cachy-scharz_weighted_sum}, 
	\begin{align*}
		\sum \nolimits_{i=1}^{T} R q_t \sigma_i^2 &\ge \left(\sum\nolimits_{t=1}^T \sqrt{q_t} \right)^2 \\
		&= \gamma^T \left( \sum\nolimits_{t=1}^T \gamma^{t/2} \right)^2 \\
		&= \gamma^{T+1} \left( {1 - \gamma^{T/2} \over 1 - \sqrt{\gamma}} \right)^2.
	\end{align*}
	Thus,
	\begin{align*}
		\operatorname{ERUB}_{\min} &\le \gamma^T + 2 \eta_0 \alpha {\gamma^{T-1} \over 1 - \gamma \beta^2} \left( {1 - \gamma^{T/2} \over 1 - \sqrt{\gamma}} \right)^2 \\
		&= \gamma^T \left( 1 + 2 \eta_0 \gamma c_ 1\alpha \left( {1 - \gamma^{T/2} \over 1 - \sqrt{\gamma}} \right)^2 \right)
	\end{align*}
	Let $T = \left\lceil {2\over \ln(1/\gamma)} \ln \left( 1 + { \eta_0 \over \kappa \alpha } \right) \right\rceil$.
	Then,
	\begin{align*}
		\operatorname{ERUB}_{\min} &\le \left({ \kappa \alpha \over \kappa \alpha + \eta_0} \right)^2 \left( 1 + {2 \eta_0 \gamma c_ 1\alpha \over (1 - \sqrt{\gamma})^2} ({ 1 \over \kappa \alpha + 1})^2 \right) \\
		&\le \left({ \kappa \alpha \over \kappa \alpha + \eta_0} \right)^2 \left( 1 + \cO({ 1 \over \kappa \alpha + 1}) \right) \\
		&\le \cO \left({ \kappa \alpha \over \kappa \alpha + \eta_0} \right)^2.
	\end{align*}
	In summary,
	\begin{align*}
		\operatorname{ERUB}_{\min} &\le \cO \left({ \kappa \alpha \over \kappa \alpha + \eta_0} \left(  \II_{T\le \hat T} { \kappa \alpha \over \kappa \alpha + \eta_0} + \II_{T>\hat T} \right) \right)
	\end{align*}
\end{proofEnd}

\noindent\textbf{Discussion.} Theoretically, the dynamic schedule is more influential in vanilla gradient descent methods than the momentum variant.
The result is mainly attributed to the averaging operation.
The moving averaging, $(1-\beta) \sum_{i=1}^t \beta^{t-i} g_i / (1-\beta^t)$, increase the influence of the under-presented initial steps and decrease the one of the over-sensitive last steps.
Counterintuitively, the preferred dynamic schedule should be increasing since $q_t$ decreases when $t\le \hat T$.

\subsection{\rev{Extension to Private Stochastic Gradient Descent}}

Though PGD provides a guarantee both for utility and privacy, computing gradients of the whole dataset is impractical for large-scale problems.
For this sake, studying the convergence of Private Stochastic Gradient Descent (PSGD) is meaningful.
The \cref{alg:private_grad} can be easily extended to PSGD by subsampling $n$ gradients where the batch size $n \ll N$.
According to \citep{yu2019differentially}, when privacy is measured by zCDP, there are two ways to account for the privacy cost of PSGD depending on the batch-sampling method: sub-sampling with or without replacement.
In this paper, we focus on the random subsampling with replacement since it is widely used in deep learning in literature, e.g., \citep{abadi2016deep,feldman2020private}.
Accordingly, we replace $N$ in the definition of $\alpha$ by $n$ because the term is from the sensitivity of batch data (see \cref{eq:private_iter}).
For clarity, we assume that $T$ is the number of iterations rather than epochs and that $\tilde \nabla_t$ is mean stochastic gradient.

When a batch of data are randomly sampled, the privacy cost of one iteration is $c p^2 {/ \sigma_t}$ where $c$ is some constant, $p=n/N$ is the sample rate, and ${1/ \sigma_t^2}$ is the full-batch privacy cost.
Details of the sub-sampling theorems are referred to the Theorem 3 of \citep{yu2019differentially} and their empirical setting.
Threfore, we can replace the privacy constraint $\sum_t p^2 / \sigma_t^2 = R$ by $\sum_t {1/ \sigma_t^2} = R'$ where $R'=R/p^2 = {N^2\over n^2}R$.
Remarkably, we omit the constant $c$ because it will not affect the results regarding uniform or dynamic schedules.
Notice $N^2 R$ in the $\alpha$ is replaced by $n^2 R' = N^2 R$.
Thus, the form of $\alpha$ is not changed which provides convenience for the following derivations.

Now we study the utility bound of PSGD.
To quantify the randomness of batch sampling, we define a random vector $\xi_t$ with $\Ebb [\xi_t] = 0$ and $\Ebb \norm{\xi_t}^2 \le D$ such that $\tilde \nabla_t \le \nabla_t + \sigma_g \xi_t/n$ for some positive constant $\sigma_g$. %
Because $\xi_t$ has similar property to the privacy noise $\nu_t$, we can easily extend the PGD bounds to PSGD bounds by following theories.
\begin{theoremEnd}{theorem}[Utility bounds of PSGD]\stepThmRelabel
	Let $\alpha$, $\kappa$ and $\gamma$ be defined in \cref{eq:def:alpha_kappa_gamma}, and $\eta_t = {1\over M}$.
    Suppose $f(\theta; x_i)$ is $G$-Lipschitz and $f(\theta)$ is $M$-smooth satisfying the Polyak-Lojasiewicz condition.
    For PSGD, when batch size satisfies $n = \max\{N\sqrt{R}, 1 \}$, the following holds:  %
    $\operatorname{ERUB} = \gamma^T + \alpha_g \sigma_g^2 + R' \sum\nolimits_{t=1}^T q_t \sigma_t^2, \text{ where } q_t \triangleq \gamma^{T-t} \alpha,~\sum_t {1/ \sigma_t^2} = R'.$
    where $\alpha_g = {D \over 2 \mu N^2 R (f(\theta_1) - f(\theta^*) )}$.
\end{theoremEnd}
\begin{proofEnd}
	Let $\tilde \nabla_t$ be the stochastic gradient of the step $t$.
	By the smoothness, we have
	\begin{align*}
		f(\theta_{t+1}) - f(\theta_t) &\le - \eta_t \nabla_t^\top (\tilde \nabla_t + G \sigma_t \nu_t / n) + {1\over 2} M \eta_t^2 \norm{\tilde \nabla_t + G \sigma_t \nu_t / n}^2 \\
		&= - \eta_t \nabla_t^\top (\nabla_t + \sigma_g \xi_t/n + G \sigma_t \nu_t / n) + {1\over 2} M \eta_t^2 \norm{\nabla_t + \sigma_g \xi_t/n + G \sigma_t \nu_t / n}^2.
	\end{align*}
	Note that $\Ebb (\sigma_g \xi_t/n + G \sigma_t \nu_t / n) = 0$ and $\Ebb(\sigma_g \xi_t/n + G \sigma_t \nu_t / n)^2 = \sigma_g^2 + (G \sigma_t / n)^2$.
	Without loss of generality, we can write $\sigma_g \xi_t + G \sigma_t \nu_t $ as $ \tilde \sigma_t \zeta_t $ where $\tilde \sigma_t \triangleq \sqrt{\sigma_g^2 + (G \sigma_t)^2}$ and $\zeta_t$ is a random vector with $\Ebb \zeta_t = 0$ and $\Ebb \norm{\zeta_t}^2 \le D$.
	Therefore,
	\begin{align*}
		f(\theta_{t+1}) - f(\theta_t) &\le - \eta_t \nabla_t^\top ( \nabla_t + \tilde \sigma_t \zeta_t/n ) + {1\over 2} M \eta_t^2 \norm{\nabla_t + \tilde \sigma_t \zeta_t/n }^2 \\
		&= - \eta_t (1 - {1\over 2} M \eta_t) \norm{\nabla_t}^2 - (1 - M\eta_t) \eta_t \nabla_t^\top \tilde \sigma_t \zeta_t/n + {1\over 2} M \eta_t^2 \norm{ \tilde \sigma_t \zeta_t/n }^2 \\
		&\le - 2 \mu \eta_t (1 - {1\over 2} M \eta_t) (f(\theta_t) - f(\theta^*) ) - (1 - M\eta_t) \eta_t \nabla_t^\top \tilde \sigma_t \zeta_t/n \\
		&\quad + {1\over 2} M \eta_t^2 \norm{ \tilde \sigma_t \zeta_t/n }^2.
	\end{align*}
	Then following the same proof of \cref{thm:excess_loss_motivation}, we can get
	\begin{align*}
		\Ebb[f(\theta_{T+1})] - f(\theta^*) &\le \gamma^T (\Ebb[f(\theta_1)] - f(\theta^*) ) + R'  \sum_{t=1}^T \gamma^{T-t} \alpha {1 \over G^2} \tilde \sigma_t^2 (\Ebb[f(\theta_1)] - f(\theta^*) ) \\
		&= \left[ \gamma^T + R'  \sum_{t=1}^T \gamma^{T-t} \alpha ({1 \over G^2} \sigma_g^2 + \sigma_t^2) \right] (\Ebb[f(\theta_1)] - f(\theta^*) ) \\
		&= \left[ \gamma^T + R' \alpha {1 \over G^2} \sigma_g^2 {1 - \gamma^T \over 1 - \gamma} + R' \sum_{t=1}^T \gamma^{T-t} \alpha \sigma_t^2 \right] (\Ebb[f(\theta_1)] - f(\theta^*) ) \\
		&\le \left[ \gamma^T + {R' \kappa \alpha \over G^2} \sigma_g^2 + R' \sum_{t=1}^T \gamma^{T-t} \alpha \sigma_t^2 \right] (\Ebb[f(\theta_1)] - f(\theta^*) ).
	\end{align*}
	where ${R' \kappa \alpha \over G^2} = {D \over 2 \mu (f(\theta_1) - f(\theta^*) )} {1\over n^2} = {D \over 2 \mu (f(\theta_1) - f(\theta^*) )} \min\{ {1\over N^2 R}, 1 \} \le {D \over 2 \mu (f(\theta_1) - f(\theta^*) )} {1\over N^2 R}$.
\end{proofEnd}

\begin{theoremEnd}{theorem}[PSGD with momentum]\stepThmRelabel
	Let $\alpha_g = {D \over 2 \mu N^2 R (f(\theta_1) - f(\theta^*) )}$.
	Suppose assumptions in \cref{thm:excess_loss_momentum} holds.
	When batch size satisfies $n = \max\{N\sqrt{R}, 1 \}$, the $U_3(\sigma, T)$ has to be replaced by
	$\tilde U_3 = U_3^g + U_3, \text{ with } \alpha R' U_3^g \le \alpha_g \sigma_g^2$
	when PSGD is used.
\end{theoremEnd}
\begin{proofEnd}
	Without loss of generality, we can write $\sigma_g \xi_t + G \sigma_t \nu_t $ as $ \tilde \sigma_t \zeta_t $ where $\tilde \sigma_t \triangleq \sqrt{\sigma_g^2 + (G \sigma_t)^2}$ and $\zeta_t$ is a random vector with $\Ebb \zeta_t = 0$ and $\Ebb \norm{\zeta_t}^2 \le D$.
	Therefore, we replace $\nu_t$ by $\zeta_t$ and $\sigma_t^2$ by $\tilde \sigma_t^2 / G^2 = \sigma_g^2 / G^2 + \sigma_t^2$.
	Now, we only need to update $U_3(\sigma, T)$ as
	\begin{align*}
		\tilde U_3 &= {1 \over G^2} \sum\nolimits_{t=1}^T \gamma^{T-t} {(1 - \beta)^2 \over (1-\beta^t)^2} \sum \nolimits_{i=1}^t \beta^{2(t-i)} \tilde \sigma_i^2 \\
		&= \sum\nolimits_{t=1}^T \gamma^{T-t} {(1 - \beta)^2 \over (1-\beta^t)^2} \sum \nolimits_{i=1}^t \beta^{2(t-i)} ( {1 \over G^2} \sigma_g^2 + \sigma_t^2 ) \\
		&= U_3^g + U_3
	\end{align*}
	where we define
	\begin{align*}
		U^g_3 &\triangleq  {1 \over G^2} \sigma_g^2 \sum\nolimits_{t=1}^T \gamma^{T-t} {(1 - \beta)^2 \over (1-\beta^t)^2} \sum \nolimits_{i=1}^t \beta^{2(t-i)}.
	\end{align*}
	We can upper bound $U^g_3$ by
	\begin{align*}
		U^g_3 &= {1 \over G^2} \sigma_g^2 \sum\nolimits_{t=1}^T \gamma^{T-t} {(1 - \beta)^2 \over (1-\beta^t)^2} {1 - \beta^{2t} \over 1-\beta^2} \\
		&= {1 \over G^2} \sigma_g^2 \sum\nolimits_{t=1}^T \gamma^{T-t} {1 - \beta \over 1-\beta^t } {1 + \beta^{t} \over 1 + \beta} \\
		&\le {1 \over G^2} \sigma_g^2 \sum\nolimits_{t=1}^T \gamma^{T-t} \\
		&\le {1 \over G^2} \sigma_g^2 {1\over 1- \gamma} \\
		&= {1 \over G^2} \kappa \sigma_g^2.
	\end{align*}
	Combine with the factors of $U_3$ in the PGD bounds:
	\begin{align*}
		\alpha R' U_3^g &\le {\alpha R' \over G^2} \kappa \sigma_g^2 = {\alpha R' \over G^2} \kappa \sigma_g^2 = { D \sigma_g^2 \over 2 \mu n^2 ( f(\theta_1) - f(\theta^*))} \le { D \sigma_g^2 \over 2 \mu N^2 R ( f(\theta_1) - f(\theta^*))}.
	\end{align*}
\end{proofEnd}
As shown above, the utility bound of PSGD differs from the PGD merely by $\alpha_g \sigma_g^2$.
Note $\alpha_g = \cO({D\over N^2 R})$ which fits the order of dynamic-schedule bounds.
In addition, $\alpha$ and other variables are not changed.
Hence, the conclusions w.r.t. the dynamic/uniform schedules maintain the same.

\subsection{\rev{Comaprison of generalization bounds}}
\label{sec:gen_bd}

\begin{table*}[ht]
  \caption{
  Comparison of true excess risk bounds.
  The algorithms are $T$-iteration $\frac{1}{2}R$-zCDP or equivalently $(\epsilon, \delta)$-DP under the $\mu$-strongly-convex condition. %
  The $\cO$ notation in this table drops other $\ln$ terms. 
  Assume loss functions are $1$-smooth and $1$-Lipschitz continuous, and
  all parameters satisfy their numeric assumptions.
  * marks the method with convex assumption.
  }
  \label{tbl:compare_gen_bd_appd}
  \small
  \centering
  \begin{tabular}{lcc}
    \toprule
    Algorithm & Utility Upper Bd. & T \\
    \midrule
    GD+Adv~\citep{bassily2014private} & $\cO_{1-p} \left( \sqrt{D} \ln^2 N \ln(1/p) \over p \mu N R_{\epsilon, \delta} \right)$ & $\cO(N^2)$ \\
    SVRG+MA~\citep{wang2017differentially} & $\cO \left( {D \ln N \over \mu N^2 R_{\epsilon, \delta}} \right)$ & $\cO(\ln{N^2R_{\epsilon, \delta}\over D})$  \\
    SSGD+zCDP~\citep{feldman2020private} & $\cO \left( \left( {1\over \sqrt{N}} + {2 \sqrt{D}\over \sqrt{R} N} \right) \ln N \right)$ & $\cO({N^2 \over 16 D/R^2 + 4N})$ \\
    * SGD+MA~\citep{bassily2019private} & $\cO \left( \max \left\{ \frac{\sqrt{D}}{N\sqrt{R_{\epsilon, \delta}}}, {1\over \sqrt{N}} \right\} \right)$ & $\cO(\min \{ {N\over 8}, {N^2 R_{\epsilon, \delta} \over 32 D} \}) $  \\
    \midrule
	\multicolumn{3}{c}{\cellcolor{lightgray!25}True risk in high probability ($1-p$)}  \\
    GD+zCDP, Static Schedule \vspace{+0.05in} & $\cO_{1-p} \left( {G^2 \over \mu N} \left( \sqrt{D \ln(N) \ln(1/p) \over N R} + {4\over p} \right) \right)$ & $\cO(\ln{N^2R\over D})$ \\
    GD+zCDP, Dynamic Schedule \vspace{+0.05in} & $\cO_{1-p} \left( {G^2 \over \mu N} \left( \sqrt{D \ln(1/p) \over N R} + {4\over p} \right) \right)$ & $\cO(\ln{N^2R\over D})$ \\
    Momentum+zCDP, Static Sch. \vspace{+0.05in} &  $ \cO_{1-p} \left( {G^2 \over \mu N} \left( \sqrt{ \frac{D \ln(1/p)}{ N R} (c + \ln N \II_{T>\hat T}) } + {4\over p} \right) \right)$ & $\cO(\ln{N^2R\over D})$ \\
    Momentum+zCDP, Dynamic Sch. & $ \cO_{1-p} \left( {G^2 \over \mu N} \left( \sqrt{ \frac{D \ln(1/p)}{ N R} (1 + \frac{c D}{ N^2 R}\II_{T>\hat T} )} + {4\over p} \right) \right)$ & $\cO(\ln{N^2R\over D})$ \\
    \midrule
	\multicolumn{3}{c}{\cellcolor{lightgray!25}True risk by uniform stability}  \\
    GD, Non-Private \vspace{+0.05in} & $\cO \left( \frac{D }{ N^2 R} \right)$ & $\cO(\ln{N^2R\over D})$ \\
    GD+zCDP, Static Schedule \vspace{+0.05in} & $\cO \left( {D \ln N \over N^2 R} \right)$ & $\cO(\ln{N^2R\over D})$   \\
    GD+zCDP, Dynamic Schedule \vspace{+0.05in} & $ \cO \left( \frac{D }{ N^2 R} \right)$ & $\cO(\ln{N^2R\over D})$ \\
    Momentum+zCDP, Static Sch. \vspace{+0.05in} & $ \cO \left( \frac{D }{ N^2 R} (c + \ln N \II_{T>\hat T}) \right)$ & $\cO(\ln{N^2R\over D})$   \\
    Momentum+zCDP, Dynamic Sch. & $ \cO \left( \frac{D }{ N^2 R} (1 + \frac{c D}{ N^2 R}\II_{T>\hat T} ) \right)$ & $\cO(\ln{N^2R\over D})$ \\
	\bottomrule
  \end{tabular}
\end{table*}

In addition to the empirical risk bounds in \cref{tbl:compare_bd}, in this section we study the \emph{true risk bounds}, or generalization error bounds.
True risk bounds characterize how well the learnt model can generalize to unseen samples subject to the inherent data distribution.
By leveraging the generic learning-theory tools, we extend our results to the \emph{True Excess Risk} (TER) for strongly convex functions as follows. For a model $\theta$, its TER is defined as follows:
\begin{align*}
  \operatorname{TER} \triangleq \Ebb_{x \sim \cX} [\Ebb[f(\theta; x)]] - \min \nolimits_{\hat \theta} \Ebb_{x \sim \cX} [f(\hat \theta; x)],
\end{align*}
where the second expectation is over the randomness of generating $\theta$ (e.g., the noise and stochastic batches).
Assume a dataset $d$ consist of $N$ samples drawn i.i.d. from the distribution $\cX$.
Two approaches could be used to extend the empirical bounds to the true excess risk:
One is proposed by \cite{shalev-shwartz2009stochastic} where the true excess risk of PGD can be bounded in high probability. For example, \cite{bassily2014private} achieved a $\ln^2 N \over N$ bound with $N^2$ iterations.
Alternatively, instead of relying on the probabilistic bound, \citet{bassily2019private} used the uniform stability to give a tighter bound.
Later, \citet{feldman2020private} improve the efficiency of gradient computation to achieve a similar bound.
Both approaches introduce an additive term to the empirical bounds.
In this section, we adopt both approaches to investigate the two types of resulting true risk bounds. 

\textbf{(1) True Risk in High Probability}.
First, we consider the high-probability true risk bound.
Based on Section 5.4 from \citep{shalev-shwartz2009stochastic} (restated in \cref{thm:ex_risk_gen_bound}), we can relate the EER to the TER.
\begin{theorem} \label{thm:ex_risk_gen_bound}
  Let $f(\theta; x)$ be $G$-Lipschitz, and $f(\theta)$ be $\mu$-strong convex loss function given any $x\in \cX$.
  With probability at least $1-p$ over the randomness of sampling the data set $d$, the following inequality holds:
  \begin{align}
    \operatorname{TER}(\theta) \le \sqrt{2 G^2 \over \mu N } \sqrt{f(\theta) - f(\theta^*)} + {4 G^2 \over p \mu N}, \label{eq:ex_risk_gen_bound}
  \end{align}
  where $\theta^* = \argmin_\theta f(\theta)$.
\end{theorem}

To apply the \cref{eq:ex_risk_gen_bound}, we need to extend EER, the expectation bound, to a high-probability bound.
Following \citep{bassily2014private} (Section D), we repeat the PGD with privacy budget $R/k$ for $k$ times.
Note, the output of all repetitions is still of $R$ budget.
When $k=1$, let the EER of the algorithm be denoted as $F(R)$. 
Then the EER of one execution of the $k$ repetitions is $F(R/k)$ where privacy is accounted by zCDP.
When $k=\log_2(1/p)$ for $p\in [0, 1]$, by Markov's inequality, there exists one repetition whose EER is $F(R/\log_2(1/p))$ with probability at least $1 - 1/2^k=1-p$.
Combined with \cref{eq:ex_risk_gen_bound}, we use the bounds of uniform schedule and dynamic schedules in \cref{sec:gd_discussion} to obtain:
\begin{align}
  \operatorname{TER}^{\text{uniform}} %
  &\le \tilde \cO \left( {G^2 \over \mu N} \left( \sqrt{D \ln(N) \ln(1/p) \over N R} + {4\over p} \right) \right), \\
  \operatorname{TER}^{\text{dynamic}} %
  &\le \tilde \cO \left( {G^2 \over \mu N} \left( \sqrt{D \ln(1/p) \over N R}  + {4\over p} \right) \right),
\end{align}
where we again ignore the $\kappa$ and other constants.
Similarly, we can extend the momentum methods.

\textbf{(2) True Risk by Uniform Stability}.
Following \citet{bassily2019private}, we use the uniform stability (defined in \cref{def:uni_stab}) to extend the empirical bounds.
We restate the related definition and theorems as follows.
\begin{definition}[Uniform stability]\label{def:uni_stab}
  Let $s>0$. A randomized algorithm $\cM: \cD^N \rightarrow \Theta$ is $s$-uniformly stable w.r.t. the loss function $f$ if for any neighbor datasets $d$ and $d'$, we have:
  \begin{align*}
    \sup\nolimits_{x\in \cX} \Ebb[f(\cM(d); x) - f(\cM(d'); x)] \le s,
  \end{align*}
  where the expectation is over the internal randomness of $\cM$.
\end{definition}

\begin{theorem}[See, e.g., \citep{shalev-shwartz2014understanding}] \label{thm:uni_stable_gen_bd}
Suppose $\cM: \cD^N \rightarrow \Theta$ is a $s$-uniformly stable algorithm w.r.t. the loss function $f$. Let $\cD$ be any distribution over data space and let $d\sim \cD^N$.
The following holds true.
\begin{align*}
  \Ebb_{d\sim \cD^N} [\Ebb [f(\cM(d); \cD) - f(\cM(d); d)]] \le s,
\end{align*}
where the second expectation is over the internal randomness of $\cM$.
$f(\cM(d); \cD)$ and $ f(\cM(d); d)$ represent the true loss and the empirical loss, respectively.
\end{theorem}

\begin{theorem}[Uniform stability of PGD from \citep{bassily2019private}] \label{thm:uni_sable_pgd}
Suppose $\eta < 2/M$ for $M$ smooth, $G$-Lipschitz $f(\theta; x)$.
Then PGD is $s$-uniformly stable with $s=G^2T\eta/N$.
\end{theorem}

Combining \cref{thm:uni_stable_gen_bd,thm:uni_sable_pgd}, we obtain the following:
\begin{align*}
  \operatorname{TER} &\le \operatorname{EER} + G^2 {\eta T \over N}.
\end{align*}
Because $\operatorname{EER}$ in this paper compresses a $\gamma^T$ or similar exponential terms, unlike \citep{bassily2019private}, we cannot directly minimize the TER upper bound w.r.t. $T$ and $\eta$ in the presence of a polynomial form of $\gamma^T$ and $T$.
Therefore, we still use $T=\cO(\ln{N^2R\over D})$ and $\eta$ for minimizing EER.
Note that 
\begin{align*}
  G^2 {\eta T \over N} \le \cO({G^2 \over M N}  \ln{N^2R\over D}) \le \cO\left({G^2 \over M}\right)
\end{align*}
where we assume $N\gg D$ and use $\ln N \le N$.
Because the term $\cO\left({G^2 / M}\right)$ is constant and independent from dimension, we follow \citep{bassily2019private} to drop the term when comparing the bounds.
After dropping the additive term, it is obvious to see that the advantage of dynamic schedules still maintains since $\operatorname{TER} \le \operatorname{EER}$.
A similar extension can be derived for \citep{wang2017differentially}.

We summarize the results and compare them to prior works in \cref{tbl:compare_gen_bd_appd} where we include an additional method: Snowball Stochastic Gradient Descent (SSGD).
SSGD dynamically schedule the batch size to achieve an optimal convergence rate in linear time. %

\textbf{Discussion}.
By using uniform stability, we successfully transfer the advantage of our dynamic schedules from empirical bounds to true risk bounds.
The inherent reason is that our bounds only need $\ln N$ iterations to reach the preferred final loss.
With uniform stability, the logarithmic $T$ reduce the gap caused by transferring.
Compared to the \citep{feldman2020private,bassily2019private}, our method has remarkably improved efficiency in $T$ from $N$ or $N^2$ to $\ln(N)$.
That implies fewer iterations are required for converging to the same generalization error.

\section{Experiments}
\label{sec:case_and_experiment}

We empirically validate the properties of privacy schedules and their connections to learning algorithms. 
In this section, we briefly review the schedule behavior on quadratic losses under varying data sensitivity. 

\textbf{Dataset}.
We create a subset of the MNIST dataset \citep{lecun1998gradientbased} including $1000$ handwritten images of 10 digits (MNIST). We also construct a subset of the MNIST dataset with digit 3 and 5 only, denoted as MNIST35. Compared to the original dataset ($70,000$ samples), the small set will be more vulnerable to attack and the private learning will require larger noise (see the $1/N$ factor in \cref{eq:private_iter}).
Following the preprocessing in \citep{abadi2016deep}, we project the vectorized images into a $60$-dimensional subspace extracted by PCA.

\textbf{Setup}.
The samples are first normalized so that $\sum_{n=1}^N x_{n} = 0$ and the standard deviation is $1$.
Then the sample norms are scaled such that $\max_n \norm{x_n} = 10$ (i.e., \emph{data scales}). 
Upon the scaled data, we train a 2-layer Deep Neural Network (DNN) with 1000 hidden units by logistic regression.
We fix the learning rate to $0.1$ based on the corresponding experiments of non-private training (same setting without noise).
The total privacy budget is $(4, 10^{-8})$-DP, equal to $0.1963$-zCDP, which implies $R=0.3927$.
To control the sensitivity of the gradients, we clip gradients by a clipping norm fixed at $4$.
Formally, we scale down the sample gradients to length $4$ if its norm is larger than $4$.
Because the schedule highly depends on the iteration number $T$, we grid search the best $T$ in range $[50, 150]$ for compared methods.
Therefore, we ignore the privacy cost of such tuning in our experiments which protocol is also used in previous work \citep{abadi2016deep,wu2017bolton}.
All the experiments are repeated $100$ times and metrics are averaged afterwards. %

\begin{figure*}[!ht]
	\centering
	\includegraphics[width=0.32\textwidth]{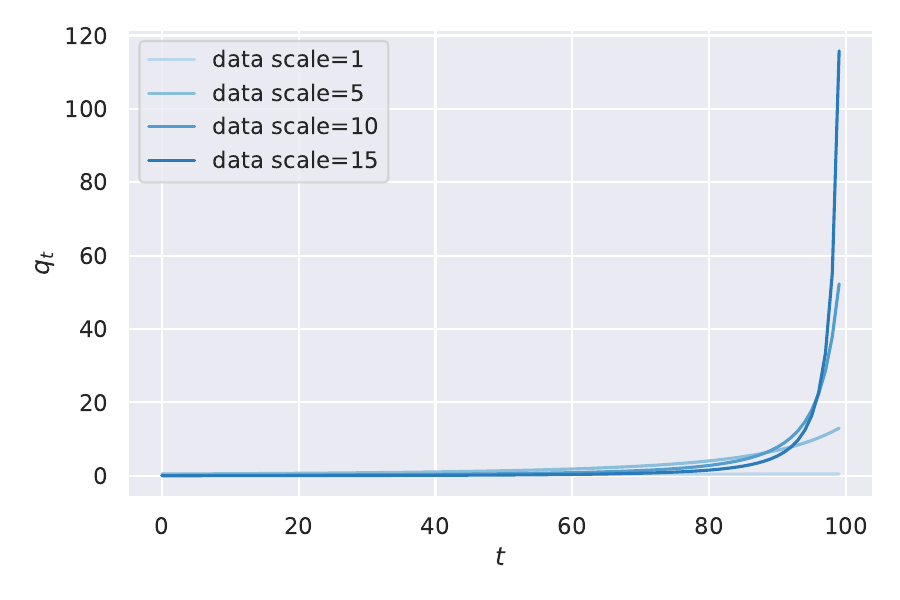}
	\hfil
	\hfil
	\includegraphics[width=0.32\textwidth]{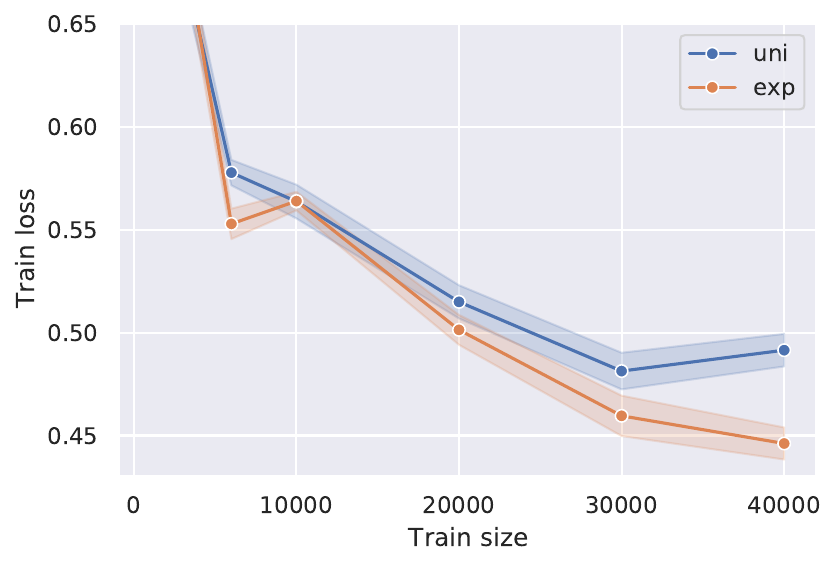}
	\hfil
	\includegraphics[width=0.32\textwidth]{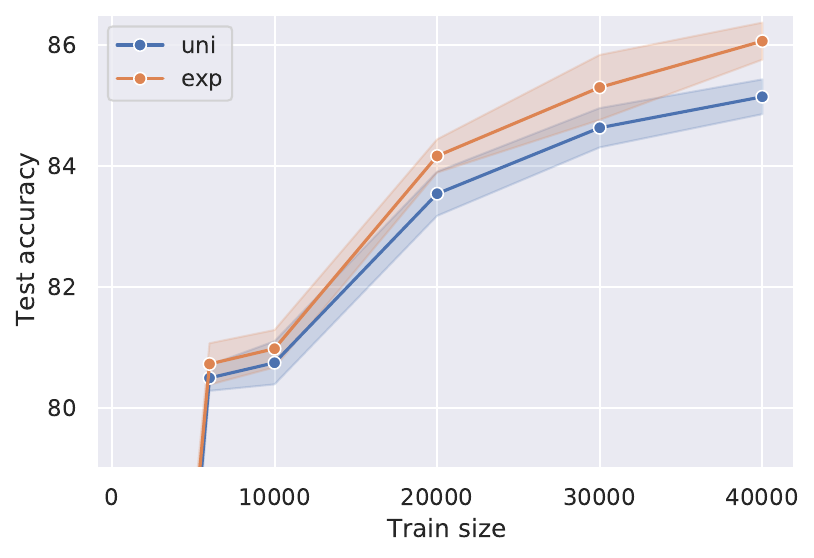}
	\caption{Comparison of dynamic schedule and uniform schedule on different data scale. 
	Left pane is the influence by iteration estimated by retraining. 
	The rest two panes are performance of DNN trained on the MNIST35 dataset with a varying total number of training samples, when the exponential influence is estimated on a randomly-generated auxiliary dataset.
	}
	\Description{Comparison of dynamic schedule and uniform schedule on different data scale.}
    \label{fig:exp_logistic_demo}
\end{figure*}

We first show the estimated influence of step noise $q_t$ (by retraining the private learning algorithms) in \cref{fig:exp_logistic_demo} Left. 
We see the trends of influence are approximately in an exponential form of $t$.
By \cref{eq:pl_dyn_sigmat}, the resultant schedule on noise scale $\sigma_t$ will be a normalized exponential decay.
This observation motivates the use of exponential decay schedule in practice.

To estimate the influence without extra privacy costs, we use an auxiliary set, which is randomly sampled from Gaussian distribution, to pick the proper influence curvature parameterized by an exponential function.
We use auxiliary synthesized datasets of the same size as the corresponding private datasets to tune the parameters.  %
We vary the size of training data to examine the data efficiency of the dynamic schedule denoted as \texttt{exp}.
For a fair comparison, we also choose the hyper-parameters of uniform schedule (\texttt{uni}) on the same auxiliary dataset.
We show that as the training size increases, \texttt{exp} outperforms \texttt{uni} both on the training loss and the test accuracy.
The result verifies our theoretic conclusion: dynamic schedule is more data efficient than the static schedule.

\section{Conclusion}
When a privacy budget is provided for a certain learning task, one has to carefully schedule the privacy usage through the learning process. Uniformly scheduling the budget has been widely used in literature whereas increasing evidence suggests that dynamically schedules could empirically outperform the uniform one. This paper provided a principled analysis on the problem of optimal budget allocation and connected the advantages of dynamic schedules to both the loss structure and the learning behavior. We further validated our results through empirical studies.

\begin{acks}
  This material is based in part upon work supported by National Institute of Aging (1RF1AG072449), Office of Naval Research (N00014-20-1-2382), National Science Foundation (IIS-1749940).
  Z. Wang is in part supported by Good Systems, a UT Austin Grand Challenge to develop responsible AI technologies
\end{acks}

\bibliographystyle{ACM-Reference-Format}
	\bibliography{auto_gen}  %


\begin{thebibliography}{42}


\ifx \showCODEN    \undefined \def \showCODEN     #1{\unskip}     \fi
\ifx \showDOI      \undefined \def \showDOI       #1{#1}\fi
\ifx \showISBNx    \undefined \def \showISBNx     #1{\unskip}     \fi
\ifx \showISBNxiii \undefined \def \showISBNxiii  #1{\unskip}     \fi
\ifx \showISSN     \undefined \def \showISSN      #1{\unskip}     \fi
\ifx \showLCCN     \undefined \def \showLCCN      #1{\unskip}     \fi
\ifx \shownote     \undefined \def \shownote      #1{#1}          \fi
\ifx \showarticletitle \undefined \def \showarticletitle #1{#1}   \fi
\ifx \showURL      \undefined \def \showURL       {\relax}        \fi
\providecommand\bibfield[2]{#2}
\providecommand\bibinfo[2]{#2}
\providecommand\natexlab[1]{#1}
\providecommand\showeprint[2][]{arXiv:#2}

\bibitem[\protect\citeauthoryear{Abadi, Chu, Goodfellow, McMahan, Mironov,
  Talwar, and Zhang}{Abadi et~al\mbox{.}}{2016}]%
        {abadi2016deep}
\bibfield{author}{\bibinfo{person}{Martin Abadi}, \bibinfo{person}{Andy Chu},
  \bibinfo{person}{Ian Goodfellow}, \bibinfo{person}{H.~Brendan McMahan},
  \bibinfo{person}{Ilya Mironov}, \bibinfo{person}{Kunal Talwar}, {and}
  \bibinfo{person}{Li Zhang}.} \bibinfo{year}{2016}\natexlab{}.
\newblock \showarticletitle{Deep Learning with Differential Privacy}. In
  \bibinfo{booktitle}{\emph{CCS: Proceedings of the 2016 ACM SIGSAC Conference
  on Computer and Communications Security}} \emph{(\bibinfo{series}{CCS '16})}.
  \bibinfo{publisher}{ACM}, \bibinfo{address}{New York, NY, USA},
  \bibinfo{pages}{308--318}.
\newblock


\bibitem[\protect\citeauthoryear{Bassily, Feldman, Talwar, and
  Guha~Thakurta}{Bassily et~al\mbox{.}}{2019}]%
        {bassily2019private}
\bibfield{author}{\bibinfo{person}{Raef Bassily}, \bibinfo{person}{Vitaly
  Feldman}, \bibinfo{person}{Kunal Talwar}, {and} \bibinfo{person}{Abhradeep
  Guha~Thakurta}.} \bibinfo{year}{2019}\natexlab{}.
\newblock \showarticletitle{Private Stochastic Convex Optimization with Optimal
  Rates}.
\newblock In \bibinfo{booktitle}{\emph{Advances in Neural Information
  Processing Systems 32}}. \bibinfo{publisher}{Curran Associates, Inc.},
  \bibinfo{pages}{11282--11291}.
\newblock


\bibitem[\protect\citeauthoryear{Bassily, Smith, and Thakurta}{Bassily
  et~al\mbox{.}}{2014}]%
        {bassily2014private}
\bibfield{author}{\bibinfo{person}{R. Bassily}, \bibinfo{person}{A. Smith},
  {and} \bibinfo{person}{A. Thakurta}.} \bibinfo{year}{2014}\natexlab{}.
\newblock \showarticletitle{Private Empirical Risk Minimization: Efficient
  Algorithms and Tight Error Bounds}. In \bibinfo{booktitle}{\emph{2014 IEEE
  55th Annual Symposium on Foundations of Computer Science}}.
  \bibinfo{pages}{464--473}.
\newblock


\bibitem[\protect\citeauthoryear{Bun and Steinke}{Bun and Steinke}{2016}]%
        {bun2016concentrated}
\bibfield{author}{\bibinfo{person}{Mark Bun} {and} \bibinfo{person}{Thomas
  Steinke}.} \bibinfo{year}{2016}\natexlab{}.
\newblock \showarticletitle{Concentrated Differential Privacy: Simplifications,
  Extensions, and Lower Bounds}.
\newblock In \bibinfo{booktitle}{\emph{Theory of Cryptography}}.
  Vol.~\bibinfo{volume}{9985}. \bibinfo{publisher}{Springer Berlin Heidelberg},
  \bibinfo{address}{Berlin, Heidelberg}, \bibinfo{pages}{635--658}.
\newblock


\bibitem[\protect\citeauthoryear{Chaudhuri, Monteleoni, and Sarwate}{Chaudhuri
  et~al\mbox{.}}{2011}]%
        {chaudhuri2011differentially}
\bibfield{author}{\bibinfo{person}{Kamalika Chaudhuri}, \bibinfo{person}{Claire
  Monteleoni}, {and} \bibinfo{person}{Anand~D. Sarwate}.}
  \bibinfo{year}{2011}\natexlab{}.
\newblock \showarticletitle{Differentially Private Empirical Risk
  Minimization}.
\newblock \bibinfo{journal}{\emph{Journal of Machine Learning Research}}
  \bibinfo{volume}{12}, \bibinfo{number}{Mar} (\bibinfo{year}{2011}),
  \bibinfo{pages}{1069--1109}.
\newblock


\bibitem[\protect\citeauthoryear{Cheng, Liu, Wang, Lu, Feng, Li, and
  Duan}{Cheng et~al\mbox{.}}{2020}]%
        {cheng2020adaptive}
\bibfield{author}{\bibinfo{person}{Junhong Cheng}, \bibinfo{person}{Wenyan
  Liu}, \bibinfo{person}{Xiaoling Wang}, \bibinfo{person}{Xingjian Lu},
  \bibinfo{person}{Jing Feng}, \bibinfo{person}{Yi Li}, {and}
  \bibinfo{person}{Chaofan Duan}.} \bibinfo{year}{2020}\natexlab{}.
\newblock \showarticletitle{Adaptive Distributed Differential Privacy with
  SGD}.
\newblock \bibinfo{journal}{\emph{Workshop on Privacy-Preserving Artificial
  Intelligence}} (\bibinfo{year}{2020}), \bibinfo{pages}{6}.
\newblock


\bibitem[\protect\citeauthoryear{Cummings, Krehbiel, Lai, and
  Tantipongpipat}{Cummings et~al\mbox{.}}{2018}]%
        {cummings2018differential}
\bibfield{author}{\bibinfo{person}{Rachel Cummings}, \bibinfo{person}{Sara
  Krehbiel}, \bibinfo{person}{Kevin~A Lai}, {and} \bibinfo{person}{Uthaipon
  Tantipongpipat}.} \bibinfo{year}{2018}\natexlab{}.
\newblock \showarticletitle{Differential Privacy for Growing Databases}.
\newblock In \bibinfo{booktitle}{\emph{Advances in Neural Information
  Processing Systems 31}}. \bibinfo{publisher}{Curran Associates, Inc.},
  \bibinfo{pages}{8864--8873}.
\newblock


\bibitem[\protect\citeauthoryear{Desfontaines and Pej{\'o}}{Desfontaines and
  Pej{\'o}}{2019}]%
        {desfontaines2019sok}
\bibfield{author}{\bibinfo{person}{Damien Desfontaines} {and}
  \bibinfo{person}{Bal{\'a}zs Pej{\'o}}.} \bibinfo{year}{2019}\natexlab{}.
\newblock \showarticletitle{SoK: Differential Privacies}.
\newblock \bibinfo{journal}{\emph{arXiv:1906.01337 [cs]}} (\bibinfo{date}{June}
  \bibinfo{year}{2019}).
\newblock


\bibitem[\protect\citeauthoryear{Dwork, Karr, Nissim, and Vilhuber}{Dwork
  et~al\mbox{.}}{2020}]%
        {dwork2020privacy}
\bibfield{author}{\bibinfo{person}{Cynthia Dwork}, \bibinfo{person}{Alan Karr},
  \bibinfo{person}{Kobbi Nissim}, {and} \bibinfo{person}{Lars Vilhuber}.}
  \bibinfo{year}{2020}\natexlab{}.
\newblock \showarticletitle{On Privacy in the Age of COVID-19}.
\newblock \bibinfo{journal}{\emph{Journal of Privacy and Confidentiality}}
  \bibinfo{volume}{10}, \bibinfo{number}{2} (\bibinfo{date}{June}
  \bibinfo{year}{2020}).
\newblock


\bibitem[\protect\citeauthoryear{Dwork, McSherry, Nissim, and Smith}{Dwork
  et~al\mbox{.}}{2006}]%
        {dwork2006calibrating}
\bibfield{author}{\bibinfo{person}{Cynthia Dwork}, \bibinfo{person}{Frank
  McSherry}, \bibinfo{person}{Kobbi Nissim}, {and} \bibinfo{person}{Adam
  Smith}.} \bibinfo{year}{2006}\natexlab{}.
\newblock \showarticletitle{Calibrating Noise to Sensitivity in Private Data
  Analysis}. In \bibinfo{booktitle}{\emph{Theory of Cryptography}}
  \emph{(\bibinfo{series}{Lecture Notes in Computer Science})}.
  \bibinfo{publisher}{Springer Berlin Heidelberg}, \bibinfo{pages}{265--284}.
\newblock


\bibitem[\protect\citeauthoryear{Feldman, Koren, and Talwar}{Feldman
  et~al\mbox{.}}{2020}]%
        {feldman2020private}
\bibfield{author}{\bibinfo{person}{Vitaly Feldman}, \bibinfo{person}{Tomer
  Koren}, {and} \bibinfo{person}{Kunal Talwar}.}
  \bibinfo{year}{2020}\natexlab{}.
\newblock \showarticletitle{Private stochastic convex optimization: optimal
  rates in linear time}. In \bibinfo{booktitle}{\emph{Proceedings of the 52nd
  Annual ACM SIGACT Symposium on Theory of Computing}}
  \emph{(\bibinfo{series}{STOC 2020})}. \bibinfo{publisher}{Association for
  Computing Machinery}, \bibinfo{address}{New York, NY, USA},
  \bibinfo{pages}{439--449}.
\newblock


\bibitem[\protect\citeauthoryear{Fredrikson, Jha, and Ristenpart}{Fredrikson
  et~al\mbox{.}}{2015}]%
        {fredrikson2015model}
\bibfield{author}{\bibinfo{person}{Matt Fredrikson}, \bibinfo{person}{Somesh
  Jha}, {and} \bibinfo{person}{Thomas Ristenpart}.}
  \bibinfo{year}{2015}\natexlab{}.
\newblock \showarticletitle{Model Inversion Attacks That Exploit Confidence
  Information and Basic Countermeasures}. In \bibinfo{booktitle}{\emph{CCS:
  Proceedings of the 22Nd ACM SIGSAC Conference on Computer and Communications
  Security}} \emph{(\bibinfo{series}{CCS '15})}. \bibinfo{publisher}{ACM},
  \bibinfo{address}{New York, NY, USA}, \bibinfo{pages}{1322--1333}.
\newblock


\bibitem[\protect\citeauthoryear{Hong, Wang, Wang, and Zhou}{Hong
  et~al\mbox{.}}{2021}]%
        {hong2021learning}
\bibfield{author}{\bibinfo{person}{Junyuan Hong}, \bibinfo{person}{Haotao
  Wang}, \bibinfo{person}{Zhangyang Wang}, {and} \bibinfo{person}{Jiayu Zhou}.}
  \bibinfo{year}{2021}\natexlab{}.
\newblock \showarticletitle{Learning Model-Based Privacy Protection under
  Budget Constraints}. In \bibinfo{booktitle}{\emph{AAAI}}. \bibinfo{pages}{9}.
\newblock


\bibitem[\protect\citeauthoryear{Huang, Guan, Zhang, Qi, Wang, and Liao}{Huang
  et~al\mbox{.}}{2019}]%
        {huang2019differentially}
\bibfield{author}{\bibinfo{person}{Xixi Huang}, \bibinfo{person}{Jian Guan},
  \bibinfo{person}{Bin Zhang}, \bibinfo{person}{Shuhan Qi},
  \bibinfo{person}{Xuan Wang}, {and} \bibinfo{person}{Qing Liao}.}
  \bibinfo{year}{2019}\natexlab{}.
\newblock \showarticletitle{Differentially Private Convolutional Neural
  Networks with Adaptive Gradient Descent}. In \bibinfo{booktitle}{\emph{2019
  IEEE Fourth International Conference on Data Science in Cyberspace (DSC)}}.
  \bibinfo{pages}{642--648}.
\newblock


\bibitem[\protect\citeauthoryear{Jain, Nagaraj, and Netrapalli}{Jain
  et~al\mbox{.}}{2019}]%
        {jain2019making}
\bibfield{author}{\bibinfo{person}{Prateek Jain}, \bibinfo{person}{Dheeraj
  Nagaraj}, {and} \bibinfo{person}{Praneeth Netrapalli}.}
  \bibinfo{year}{2019}\natexlab{}.
\newblock \showarticletitle{Making the Last Iterate of SGD Information
  Theoretically Optimal}. In \bibinfo{booktitle}{\emph{Conference on Learning
  Theory}}. \bibinfo{pages}{1752--1755}.
\newblock


\bibitem[\protect\citeauthoryear{Karimi, Nutini, and Schmidt}{Karimi
  et~al\mbox{.}}{2016}]%
        {karimi2016linear}
\bibfield{author}{\bibinfo{person}{Hamed Karimi}, \bibinfo{person}{Julie
  Nutini}, {and} \bibinfo{person}{Mark Schmidt}.}
  \bibinfo{year}{2016}\natexlab{}.
\newblock \showarticletitle{Linear Convergence of Gradient and
  Proximal-Gradient Methods Under the Polyak-\L ojasiewicz Condition}. In
  \bibinfo{booktitle}{\emph{Machine Learning and Knowledge Discovery in
  Databases}} \emph{(\bibinfo{series}{Lecture Notes in Computer Science})}.
  \bibinfo{publisher}{Springer International Publishing},
  \bibinfo{address}{Cham}, \bibinfo{pages}{795--811}.
\newblock


\bibitem[\protect\citeauthoryear{Kingma and Ba}{Kingma and Ba}{2015}]%
        {kingma2015adam}
\bibfield{author}{\bibinfo{person}{Diederik~P. Kingma} {and}
  \bibinfo{person}{Jimmy Ba}.} \bibinfo{year}{2015}\natexlab{}.
\newblock \showarticletitle{Adam: A Method for Stochastic Optimization}. In
  \bibinfo{booktitle}{\emph{the 3rd International Conference for Learning
  Representations}}. \bibinfo{address}{San Diego, CA}.
\newblock


\bibitem[\protect\citeauthoryear{Lecun, Bottou, Bengio, and Haffner}{Lecun
  et~al\mbox{.}}{1998}]%
        {lecun1998gradientbased}
\bibfield{author}{\bibinfo{person}{Y. Lecun}, \bibinfo{person}{L. Bottou},
  \bibinfo{person}{Y. Bengio}, {and} \bibinfo{person}{P. Haffner}.}
  \bibinfo{year}{1998}\natexlab{}.
\newblock \showarticletitle{Gradient-based learning applied to document
  recognition}.
\newblock \bibinfo{journal}{\emph{Proc. IEEE}} \bibinfo{volume}{86},
  \bibinfo{number}{11} (\bibinfo{date}{Nov.} \bibinfo{year}{1998}),
  \bibinfo{pages}{2278--2324}.
\newblock


\bibitem[\protect\citeauthoryear{Lee and Kifer}{Lee and Kifer}{2018}]%
        {lee2018concentrated}
\bibfield{author}{\bibinfo{person}{Jaewoo Lee} {and} \bibinfo{person}{Daniel
  Kifer}.} \bibinfo{year}{2018}\natexlab{}.
\newblock \showarticletitle{Concentrated Differentially Private Gradient
  Descent with Adaptive per-Iteration Privacy Budget}. In
  \bibinfo{booktitle}{\emph{Proceedings of the 24th ACM SIGKDD International
  Conference on Knowledge Discovery \& Data Mining}}
  \emph{(\bibinfo{series}{KDD '18})}. \bibinfo{publisher}{ACM},
  \bibinfo{address}{New York, NY, USA}, \bibinfo{pages}{1656--1665}.
\newblock


\bibitem[\protect\citeauthoryear{McMahan, Moore, Ramage, Hampson, and
  y~Arcas}{McMahan et~al\mbox{.}}{2017}]%
        {mcmahan2017communicationefficient}
\bibfield{author}{\bibinfo{person}{Brendan McMahan}, \bibinfo{person}{Eider
  Moore}, \bibinfo{person}{Daniel Ramage}, \bibinfo{person}{Seth Hampson},
  {and} \bibinfo{person}{Blaise~Aguera y Arcas}.}
  \bibinfo{year}{2017}\natexlab{}.
\newblock \showarticletitle{Communication-Efficient Learning of Deep Networks
  from Decentralized Data}. In \bibinfo{booktitle}{\emph{Artificial
  Intelligence and Statistics}}. \bibinfo{pages}{1273--1282}.
\newblock


\bibitem[\protect\citeauthoryear{McMahan, Ramage, Talwar, and Zhang}{McMahan
  et~al\mbox{.}}{2018}]%
        {mcmahan2018learning}
\bibfield{author}{\bibinfo{person}{H.~Brendan McMahan}, \bibinfo{person}{Daniel
  Ramage}, \bibinfo{person}{Kunal Talwar}, {and} \bibinfo{person}{Li Zhang}.}
  \bibinfo{year}{2018}\natexlab{}.
\newblock \showarticletitle{Learning Differentially Private Recurrent Language
  Models}. In \bibinfo{booktitle}{\emph{International Conference on Learning
  Representations}}.
\newblock


\bibitem[\protect\citeauthoryear{Nesterov and Polyak}{Nesterov and
  Polyak}{2006}]%
        {nesterov2006cubic}
\bibfield{author}{\bibinfo{person}{Yurii Nesterov} {and} \bibinfo{person}{B.T.
  Polyak}.} \bibinfo{year}{2006}\natexlab{}.
\newblock \showarticletitle{Cubic regularization of Newton method and its
  global performance}.
\newblock \bibinfo{journal}{\emph{Mathematical Programming}}
  \bibinfo{volume}{108}, \bibinfo{number}{1} (\bibinfo{date}{Aug.}
  \bibinfo{year}{2006}), \bibinfo{pages}{177--205}.
\newblock


\bibitem[\protect\citeauthoryear{Pichapati, Suresh, Yu, Reddi, and
  Kumar}{Pichapati et~al\mbox{.}}{2019}]%
        {pichapati2019adaclip}
\bibfield{author}{\bibinfo{person}{Venkatadheeraj Pichapati},
  \bibinfo{person}{Ananda~Theertha Suresh}, \bibinfo{person}{Felix~X. Yu},
  \bibinfo{person}{Sashank~J. Reddi}, {and} \bibinfo{person}{Sanjiv Kumar}.}
  \bibinfo{year}{2019}\natexlab{}.
\newblock \showarticletitle{AdaCliP: Adaptive Clipping for Private SGD}.
\newblock \bibinfo{journal}{\emph{arXiv:1908.07643 [cs, stat]}}
  (\bibinfo{date}{Oct.} \bibinfo{year}{2019}).
\newblock


\bibitem[\protect\citeauthoryear{Polyak}{Polyak}{1963}]%
        {polyak1963gradient}
\bibfield{author}{\bibinfo{person}{B.~T. Polyak}.}
  \bibinfo{year}{1963}\natexlab{}.
\newblock \showarticletitle{Gradient methods for the minimisation of
  functionals}.
\newblock \bibinfo{journal}{\emph{U. S. S. R. Comput. Math. and Math. Phys.}}
  \bibinfo{volume}{3}, \bibinfo{number}{4} (\bibinfo{date}{Jan.}
  \bibinfo{year}{1963}), \bibinfo{pages}{864--878}.
\newblock


\bibitem[\protect\citeauthoryear{Polyak}{Polyak}{1964}]%
        {polyak1964methods}
\bibfield{author}{\bibinfo{person}{B.~T. Polyak}.}
  \bibinfo{year}{1964}\natexlab{}.
\newblock \showarticletitle{Some methods of speeding up the convergence of
  iteration methods}.
\newblock \bibinfo{journal}{\emph{U. S. S. R. Comput. Math. and Math. Phys.}}
  \bibinfo{volume}{4}, \bibinfo{number}{5} (\bibinfo{date}{Jan.}
  \bibinfo{year}{1964}), \bibinfo{pages}{1--17}.
\newblock


\bibitem[\protect\citeauthoryear{Qian}{Qian}{1999}]%
        {qian1999momentum}
\bibfield{author}{\bibinfo{person}{Ning Qian}.}
  \bibinfo{year}{1999}\natexlab{}.
\newblock \showarticletitle{On the momentum term in gradient descent learning
  algorithms}.
\newblock \bibinfo{journal}{\emph{Neural Networks}} \bibinfo{volume}{12},
  \bibinfo{number}{1} (\bibinfo{date}{Jan.} \bibinfo{year}{1999}),
  \bibinfo{pages}{145--151}.
\newblock


\bibitem[\protect\citeauthoryear{Reddi, Hefny, Sra, Poczos, and Smola}{Reddi
  et~al\mbox{.}}{2016}]%
        {reddi2016stochastic}
\bibfield{author}{\bibinfo{person}{Sashank~J. Reddi}, \bibinfo{person}{Ahmed
  Hefny}, \bibinfo{person}{Suvrit Sra}, \bibinfo{person}{Barnabas Poczos},
  {and} \bibinfo{person}{Alex Smola}.} \bibinfo{year}{2016}\natexlab{}.
\newblock \showarticletitle{Stochastic Variance Reduction for Nonconvex
  Optimization}. In \bibinfo{booktitle}{\emph{International Conference on
  Machine Learning}}. \bibinfo{pages}{314--323}.
\newblock


\bibitem[\protect\citeauthoryear{R{\'e}nyi}{R{\'e}nyi}{1961}]%
        {renyi1961measures}
\bibfield{author}{\bibinfo{person}{Alfr{\'e}d R{\'e}nyi}.}
  \bibinfo{year}{1961}\natexlab{}.
\newblock \showarticletitle{On measures of entropy and information}. In
  \bibinfo{booktitle}{\emph{Proceedings of the Fourth Berkeley Symposium on
  Mathematical Statistics and Probability, Volume 1: Contributions to the
  Theory of Statistics}}. \bibinfo{publisher}{The Regents of the University of
  California}.
\newblock


\bibitem[\protect\citeauthoryear{{Shalev-Shwartz} and
  {Ben-David}}{{Shalev-Shwartz} and {Ben-David}}{2014}]%
        {shalev-shwartz2014understanding}
\bibfield{author}{\bibinfo{person}{Shai {Shalev-Shwartz}} {and}
  \bibinfo{person}{Shai {Ben-David}}.} \bibinfo{year}{2014}\natexlab{}.
\newblock \bibinfo{booktitle}{\emph{Understanding Machine Learning: From Theory
  to Algorithms}}.
\newblock \bibinfo{publisher}{Cambridge University Press}.
\newblock


\bibitem[\protect\citeauthoryear{{Shalev-Shwartz}, Srebro, and
  Sridharan}{{Shalev-Shwartz} et~al\mbox{.}}{2009}]%
        {shalev-shwartz2009stochastic}
\bibfield{author}{\bibinfo{person}{Shai {Shalev-Shwartz}},
  \bibinfo{person}{Nathan Srebro}, {and} \bibinfo{person}{Karthik Sridharan}.}
  \bibinfo{year}{2009}\natexlab{}.
\newblock \showarticletitle{Stochastic Convex Optimization}. In
  \bibinfo{booktitle}{\emph{Proceedings of the 22nd Annual Conference on
  Learning Theory, COLT '09}}. \bibinfo{pages}{11}.
\newblock


\bibitem[\protect\citeauthoryear{Shokri, Stronati, Song, and Shmatikov}{Shokri
  et~al\mbox{.}}{2017}]%
        {shokri2017membership}
\bibfield{author}{\bibinfo{person}{R. Shokri}, \bibinfo{person}{M. Stronati},
  \bibinfo{person}{C. Song}, {and} \bibinfo{person}{V. Shmatikov}.}
  \bibinfo{year}{2017}\natexlab{}.
\newblock \showarticletitle{Membership Inference Attacks Against Machine
  Learning Models}. In \bibinfo{booktitle}{\emph{2017 IEEE Symposium on
  Security and Privacy (SP)}}. \bibinfo{pages}{3--18}.
\newblock


\bibitem[\protect\citeauthoryear{Thakkar, Andrew, and McMahan}{Thakkar
  et~al\mbox{.}}{2019}]%
        {thakkar2019differentially}
\bibfield{author}{\bibinfo{person}{Om Thakkar}, \bibinfo{person}{Galen Andrew},
  {and} \bibinfo{person}{H.~Brendan McMahan}.} \bibinfo{year}{2019}\natexlab{}.
\newblock \showarticletitle{Differentially Private Learning with Adaptive
  Clipping}.
\newblock \bibinfo{journal}{\emph{arXiv:1905.03871 [cs, stat]}}
  (\bibinfo{date}{May} \bibinfo{year}{2019}).
\newblock


\bibitem[\protect\citeauthoryear{Wang, Chen, and Xu}{Wang
  et~al\mbox{.}}{2019}]%
        {wang2019differentially}
\bibfield{author}{\bibinfo{person}{Di Wang}, \bibinfo{person}{Changyou Chen},
  {and} \bibinfo{person}{Jinhui Xu}.} \bibinfo{year}{2019}\natexlab{}.
\newblock \showarticletitle{Differentially Private Empirical Risk Minimization
  with Non-convex Loss Functions}. In \bibinfo{booktitle}{\emph{International
  Conference on Machine Learning}}. \bibinfo{pages}{6526--6535}.
\newblock


\bibitem[\protect\citeauthoryear{Wang, Ye, and Xu}{Wang et~al\mbox{.}}{2017}]%
        {wang2017differentially}
\bibfield{author}{\bibinfo{person}{Di Wang}, \bibinfo{person}{Minwei Ye}, {and}
  \bibinfo{person}{Jinhui Xu}.} \bibinfo{year}{2017}\natexlab{}.
\newblock \showarticletitle{Differentially Private Empirical Risk Minimization
  Revisited: Faster and More General}.
\newblock In \bibinfo{booktitle}{\emph{Advances in Neural Information
  Processing Systems 30}}. \bibinfo{publisher}{Curran Associates, Inc.},
  \bibinfo{pages}{2722--2731}.
\newblock


\bibitem[\protect\citeauthoryear{Weiner, Veitch, Aisen, Beckett, Cairns, Green,
  Harvey, Jack, Jagust, Liu, Morris, Petersen, Saykin, Schmidt, Shaw, Shen,
  Siuciak, Soares, Toga, and Trojanowski}{Weiner et~al\mbox{.}}{2013}]%
        {weiner2013alzheimer}
\bibfield{author}{\bibinfo{person}{Michael~W. Weiner},
  \bibinfo{person}{Dallas~P. Veitch}, \bibinfo{person}{Paul~S. Aisen},
  \bibinfo{person}{Laurel~A. Beckett}, \bibinfo{person}{Nigel~J. Cairns},
  \bibinfo{person}{Robert~C. Green}, \bibinfo{person}{Danielle Harvey},
  \bibinfo{person}{Clifford~R. Jack}, \bibinfo{person}{William Jagust},
  \bibinfo{person}{Enchi Liu}, \bibinfo{person}{John~C. Morris},
  \bibinfo{person}{Ronald~C. Petersen}, \bibinfo{person}{Andrew~J. Saykin},
  \bibinfo{person}{Mark~E. Schmidt}, \bibinfo{person}{Leslie Shaw},
  \bibinfo{person}{Li Shen}, \bibinfo{person}{Judith~A. Siuciak},
  \bibinfo{person}{Holly Soares}, \bibinfo{person}{Arthur~W. Toga}, {and}
  \bibinfo{person}{John~Q. Trojanowski}.} \bibinfo{year}{2013}\natexlab{}.
\newblock \showarticletitle{The Alzheimer's Disease Neuroimaging Initiative: A
  review of papers published since its inception}.
\newblock \bibinfo{journal}{\emph{Alzheimer's \& Dementia}}
  \bibinfo{volume}{9}, \bibinfo{number}{5} (\bibinfo{date}{Sept.}
  \bibinfo{year}{2013}), \bibinfo{pages}{e111--e194}.
\newblock


\bibitem[\protect\citeauthoryear{Wu, Li, Kumar, Chaudhuri, Jha, and
  Naughton}{Wu et~al\mbox{.}}{2017}]%
        {wu2017bolton}
\bibfield{author}{\bibinfo{person}{Xi Wu}, \bibinfo{person}{Fengan Li},
  \bibinfo{person}{Arun Kumar}, \bibinfo{person}{Kamalika Chaudhuri},
  \bibinfo{person}{Somesh Jha}, {and} \bibinfo{person}{Jeffrey Naughton}.}
  \bibinfo{year}{2017}\natexlab{}.
\newblock \showarticletitle{Bolt-on Differential Privacy for Scalable
  Stochastic Gradient Descent-based Analytics}. In
  \bibinfo{booktitle}{\emph{Proceedings of the 2017 ACM International
  Conference on Management of Data}} \emph{(\bibinfo{series}{SIGMOD '17})}.
  \bibinfo{publisher}{ACM}, \bibinfo{address}{New York, NY, USA},
  \bibinfo{pages}{1307--1322}.
\newblock


\bibitem[\protect\citeauthoryear{Xie, Li, Wu, and Wu}{Xie
  et~al\mbox{.}}{2021}]%
        {xie2021differential}
\bibfield{author}{\bibinfo{person}{Yun Xie}, \bibinfo{person}{Peng Li},
  \bibinfo{person}{Chao Wu}, {and} \bibinfo{person}{Qiuling Wu}.}
  \bibinfo{year}{2021}\natexlab{}.
\newblock \showarticletitle{Differential Privacy Stochastic Gradient Descent
  with Adaptive Privacy Budget Allocation}. In \bibinfo{booktitle}{\emph{2021
  IEEE International Conference on Consumer Electronics and Computer
  Engineering (ICCECE)}}. \bibinfo{pages}{227--231}.
\newblock


\bibitem[\protect\citeauthoryear{Xu, Shi, Liu, Zhao, and Chen}{Xu
  et~al\mbox{.}}{2020}]%
        {xu2020adaptive}
\bibfield{author}{\bibinfo{person}{Zhiying Xu}, \bibinfo{person}{Shuyu Shi},
  \bibinfo{person}{Alex~X. Liu}, \bibinfo{person}{Jun Zhao}, {and}
  \bibinfo{person}{Lin Chen}.} \bibinfo{year}{2020}\natexlab{}.
\newblock \showarticletitle{An Adaptive and Fast Convergent Approach to
  Differentially Private Deep Learning}.
\newblock \bibinfo{journal}{\emph{the Proceedings of IEEE International
  Conference on Computer Communications}} (\bibinfo{year}{2020}).
\newblock


\bibitem[\protect\citeauthoryear{Yu, Zhang, Chen, Yin, and Liu}{Yu
  et~al\mbox{.}}{2020}]%
        {yu2020gradientb}
\bibfield{author}{\bibinfo{person}{Da Yu}, \bibinfo{person}{Huishuai Zhang},
  \bibinfo{person}{Wei Chen}, \bibinfo{person}{Jian Yin}, {and}
  \bibinfo{person}{Tie-Yan Liu}.} \bibinfo{year}{2020}\natexlab{}.
\newblock \showarticletitle{Gradient Perturbation is Underrated for
  Differentially Private Convex Optimization}. In
  \bibinfo{booktitle}{\emph{Proceedings of the Twenty-Ninth International Joint
  Conference on Artificial Intelligence}}. \bibinfo{publisher}{International
  Joint Conferences on Artificial Intelligence Organization},
  \bibinfo{address}{Yokohama, Japan}, \bibinfo{pages}{3117--3123}.
\newblock


\bibitem[\protect\citeauthoryear{Yu, Liu, Pu, Gursoy, and Truex}{Yu
  et~al\mbox{.}}{2019}]%
        {yu2019differentially}
\bibfield{author}{\bibinfo{person}{Lei Yu}, \bibinfo{person}{Ling Liu},
  \bibinfo{person}{Calton Pu}, \bibinfo{person}{Mehmet~Emre Gursoy}, {and}
  \bibinfo{person}{Stacey Truex}.} \bibinfo{year}{2019}\natexlab{}.
\newblock \showarticletitle{Differentially Private Model Publishing for Deep
  Learning}.
\newblock \bibinfo{journal}{\emph{proceedings of 40th IEEE Symposium on
  Security and Privacy}} (\bibinfo{date}{April} \bibinfo{year}{2019}).
\newblock


\bibitem[\protect\citeauthoryear{Zhang, Ding, Wu, Wong, Van~Nguyen, and
  Pan}{Zhang et~al\mbox{.}}{2021}]%
        {zhang2021adaptive}
\bibfield{author}{\bibinfo{person}{Xinyue Zhang}, \bibinfo{person}{Jiahao
  Ding}, \bibinfo{person}{Maoqiang Wu}, \bibinfo{person}{Stephen T.~C. Wong},
  \bibinfo{person}{Hien Van~Nguyen}, {and} \bibinfo{person}{Miao Pan}.}
  \bibinfo{year}{2021}\natexlab{}.
\newblock \showarticletitle{Adaptive Privacy Preserving Deep Learning
  Algorithms for Medical Data}. In \bibinfo{booktitle}{\emph{Proceedings of the
  IEEE/CVF Winter Conference on Applications of Computer Vision}}.
  \bibinfo{pages}{1169--1178}.
\newblock


\bibitem[\protect\citeauthoryear{Zhou, Chen, Hong, Wu, and Banerjee}{Zhou
  et~al\mbox{.}}{2020}]%
        {zhou2020private}
\bibfield{author}{\bibinfo{person}{Yingxue Zhou}, \bibinfo{person}{Xiangyi
  Chen}, \bibinfo{person}{Mingyi Hong}, \bibinfo{person}{Zhiwei~Steven Wu},
  {and} \bibinfo{person}{Arindam Banerjee}.} \bibinfo{year}{2020}\natexlab{}.
\newblock \showarticletitle{Private Stochastic Non-Convex Optimization:
  Adaptive Algorithms and Tighter Generalization Bounds}.
\newblock \bibinfo{journal}{\emph{arXiv:2006.13501 [cs, stat]}}
  (\bibinfo{date}{Aug.} \bibinfo{year}{2020}).
\newblock


\end{thebibliography}

\clearpage

\appendix

\onecolumn

\section{Social Impact}

The wide usage of personal data in training machine learning has led to huge successes in many application domains but is also accompanied by rising concerns on privacy protection due to the sensitive information in the data. The development of privacy-preserving algorithms has become one of critical research areas of machine learning, in which the key challenge is to train high performance models under the constraint of a given privacy budget, or how much sensitive information can be accessed during the training phase. Differential privacy provided a principled framework to quantify the privacy budget, under which researchers proposed various schemes to schedule the budget usage during a learning process, yet there is a lack of systematical studies on when and why some schedules are better than other. Our efforts in this paper are among the first to study and compare the effectiveness of these schedules from a rigorous optimization perspective. Our theoretical results can benefit any privacy-preserving machine learning practitioners to efficiently and effectively choose proper privacy schedules tailored to their learning tasks.

\section{Comparison of algorithms}
\label{sec:comp_alg}

Here we elaborate the meaning of algorithm names in \cref{tbl:compare_bd}.
Asymptotic upper bounds are achieved when sample size $N$ approaches infinity. Both $R$ and $R_{\epsilon, \delta}$ with $R_{\epsilon, \delta}<R$ are the privacy budgets of corresponding algorithms. 
Specifically, $R_{\epsilon, \delta} = \epsilon^2/\ln(1/\delta) < R$ when the private algorithm is $(\epsilon, \delta)$-DP with $\epsilon \le 2 \ln(1/\delta)$.

\textbf{PGD+Adv}. Adv denotes the Advanced Composition method \citep{bassily2014private}. The method assumes that loss function is $1$-strongly convex which implies the PL condition and optimized variable is in a convex set of diameter $1$ w.r.t. $l_2$ norm.

\textbf{PGD+MA} and the adjusted-utility version. MA denotes the Moment Accountant \citep{abadi2016deep} which improve the composed privacy bound versus the Advanced Composition. The improvement on privacy bound lead to a enhanced utility bound, as a result.

\textbf{PGD+Adv+BBImp}. The dynamic method assumes that the loss is $1$-strongly convex and data comes in stream with $n\le N$ samples at each round. 
Their utility upper bound is achieved at some probability $p$ with any positive $c$.

\textbf{Adam+MA}. The authors prove a convergence bound for the gradient norms which is extended to loss bound by using PL condition. They also presents the results for AdaGrad and GD which are basically of the same upper bound.
Out theorems improve their bound by using the recursive derivation based on the PL condition, while their bound is a simple application of the condition on the gradient norm bound.

\textbf{GD, Non-Private}. This method does not inject noise into gradients but limit the number of iterations. With the bound, we can see that our utility bound are optimal with dynamic schedule.

\textbf{GD+zCDP}. We discussed the static and dynamic schedule for the gradient descent method where the dynamic noise influence is the key to tighten the bound.

\textbf{Momentum+zCDP}. Different from the GD+zCDP, momentum methods will have two phase of utility upper bound. When $T$ is small than some positive constant $\hat T$, the bound is as tight as the non-private one. Afterwards, the momentum has a bound degraded as the GD bound.

\section{Preliminaries}

\subsection{Privacy}

\begin{lemma}[Composition \& Post-processing]
\label{thm:comp_postprocess_zCDP}
Let two mechanisms be $M: \mathcal{D}^n \rightarrow \cY$ and $M': \cD^n \times \cY \rightarrow \cZ$. Suppose $M$ satisfies $(\rho_1, a)$-zCDP and $M'(\cdot, y)$ satisfies $(\rho_2, a)$-zCDP for $\forall y \in \cY$. Then, mechanism $M'': \cD^n \rightarrow \cZ$ (defined by $M''(x) = M'(x, M(x))$) satisfies $(\rho_1+\rho_2)$-zCDP.
\end{lemma}

\begin{definition} [Sensitivity] \label{def:sens}
The sensitivity of a gradient query $\nabla_t$ to the dataset $\{x_i\}_{i=1}^N$ is
\begin{align}
\Delta_2(\nabla_t) 
&= \max_n \norm{ {1\over N} \sum\nolimits_{j=1, j\neq n}^N \nabla_t^{(j)}  - {1\over N} \sum\nolimits_{j=1}^N \nabla_t^{(j)} }_2 \notag
\\ &= {1\over N} \max_n \norm{\nabla^{(n)}_t}_2 \label{eq:sens}
\end{align}
where $\nabla^{(n)}_t$ denotes the gradient of the $n$-th sample.
\end{definition}

\begin{lemma}
[Gaussian mechanism \citep{bun2016concentrated}]
\label{thm:gauss_zCDP}
Let $f: \cD^n \rightarrow \cZ$ have sensitivity $\Delta$.
Define a randomized algorithm $M: \cD^n \rightarrow \cZ$ by $M(x) \leftarrow f(x) + \cN(0, \Delta^2 \sigma^2 I )$.
Then $M$ satisfies ${1 \over 2 \sigma^2}$-zCDP.
\end{lemma}

\begin{lemma}[\citep{bun2016concentrated}]\label{lm:zCDP2DP}
If $M$ is a mechanism satisfying $\rho$-zCDP, then $M$ is $(\rho + 2\sqrt{\rho \ln(1/\delta)}, \delta)$-DP for any $\delta>0$.
\end{lemma}
By solving $\rho + 2\sqrt{\rho \ln(1/\delta)}=\epsilon$, we can get $\rho= \epsilon + 2\ln(1/\delta) + 2 \sqrt{\ln(1/\delta)(\epsilon + \ln(1/\delta)}$.

\subsection{Auxiliary lemmas}

\begin{lemma} \label{lm:quad_sens}
If $\max_n \norm{x_n}_2 = 1$ and ${1\over N} \sum_n x_n = 0$, then the gradient sensitivity of the squared loss will be 
\begin{align*}
\Delta_2(\nabla) = \max_i {1\over N} \sqrt{2 f(\theta; x_i)} \norm{x_i}_2 \le {1\over 2} (DM\norm{\theta}^2 + 1),
\end{align*}
where $\Theta_{\cM}$ is the set of all possible parameters $\theta_t$ generated by the learning algorithm $\cM$.
\end{lemma}
\begin{proof}
According to the definition of sensitivity in \cref{eq:sens}, we have
\begin{align*}
\Delta_2(\nabla) = \max_i \norm{\nabla^{(i)}}_2 = \max_n {1\over n} \norm{A^{(i)} \theta - x_i}_2
\end{align*}
where we use $i$ denotes the index of sample in the dataset.
Here, we assume it is constant 1.
We may get
\begin{align}
\norm{A^{(i)} \theta - x_i}_2^2 &= \norm{x_i (x_i^\top \theta - 1)}_2^2 \notag \\ &= (x_i^\top \theta - 1)^2 \norm{x_i}_2^2 = 2 f(\theta; x_i) \norm{x_i}_2^2
\end{align}
where $f(\theta; x_i) = {1\over 2} (x_i^\top \theta - 1)^2$.
Thus,
\begin{align*}
\Delta_2(\nabla) = \max_i {1\over N} \sqrt{2 f(\theta; x_i)} \norm{x_i}_2
\end{align*}
Since $\norm{x_n}_2 \le 1$ and ${1\over N} \sum_{n=1}^N x_n = 0$,
\begin{align*}
f(\theta) &= {1\over 2N} \sum_{n=1}^N [ (x_n^\top \theta)^2 - 2 x_n^\top \theta + 1 ] \\&\le {1\over 2N} \sum_{n=1}^N [ (\norm{x_n} \norm{\theta})^2 + 1 ] \\&\le {1\over 2} (DM \norm{\theta}^2 + 1)
\end{align*}
\end{proof}

\begin{lemma} \label{lm:ineq:sum_of_prop_mom_x1}
Assume assumptions in \cref{thm:excess_loss_momentum} are satisfied.
Given variables defined in \cref{thm:excess_loss_momentum}, the following inequality holds true:
\begin{align*}
\sum_{t=1}^T \gamma^{T-t} {2 (1 - \beta) \eta_t \over b_t}  \sum \nolimits_{i=1}^t \beta^{t-i} \norm{\nabla_t - \nabla_i}^2 \\ \le {\eta^3_0 \beta \gamma \over 2 M (1 - \beta)^3(\gamma - \beta)^2} \sum_{i=1}^{T-1} \gamma^{T-i} \norm{v_{i+1}}^2.
\end{align*}
\end{lemma}
\begin{proof}
    We first handle the inner summation.
    By smoothness, the inequality $\norm{\nabla f(x) - \nabla f(y)} \le M \norm{x - y}$ holds true.
    Thus,
    \begin{align*}
    &\quad \sum \nolimits_{i=1}^t \beta^{t-i} \norm{\nabla_t - \nabla_i}^2  \\
    &\le M^2 \sum \nolimits_{i=1}^t \beta^{t-i} \norm{\theta_t - \theta_i}^2 \\
    &= M^2 \sum \nolimits_{k=0}^{t-1} \beta^{k} \norm{\theta_t - \theta_{t-k}}^2 \\
    &= M^2 \sum \nolimits_{k=0}^{t-1} \beta^{k} \norm{\sum \nolimits_{i=t-k}^{t-1} \eta_i v_{i+1}/b_i}^2  \\
    &\le M^2 \sum \nolimits_{k=0}^{t-1} \beta^{k} \left(\sum \nolimits_{j=t-k}^{t-1} \eta_j^2/b_j^2 \right) \left(\sum \nolimits_{i=t-k}^{t-1}  \norm{v_{i+1}}^2 \right)
    \end{align*}
    where the last inequality is by Cauchy-Schwartz inequality.
    Because ${1\over b_t}={1 \over 1 - \beta^t} \le {1 \over 1 - \beta}$ and $\eta_t = {\eta_0 \over 2M}$,
    \begin{align}
    &\quad \sum \nolimits_{i=1}^t \beta^{t-i} \norm{\nabla_t - \nabla_i}^2 \\&\le {\eta^2_0 \over 4 (1 - \beta)^2} \sum \nolimits_{k=0}^{t-1} \beta^{k} k \sum \nolimits_{i=t-k}^{t-1}  \norm{v_{i+1}}^2 \notag \\
    &= {\eta^2_0 \over 4 (1 - \beta)^2} \sum \nolimits_{k=0}^{t-1} \beta^{k} k \sum \nolimits_{i=1}^{t-1}  \norm{v_{i+1}}^2 \mathbb{I}(i \ge t-k) \notag  \\
    &= {\eta^2_0 \over 4 (1 - \beta)^2} \sum \nolimits_{i=1}^{t-1}  \norm{v_{i+1}}^2\sum \nolimits_{k=0}^{t-1} \beta^{k} k  \mathbb{I}(k \ge t-i) \notag \\
    &= {\eta^2_0 \over 4 (1 - \beta)^2} \sum \nolimits_{i=1}^{t-1} \norm{v_{i+1}}^2 \sum \nolimits_{k=t-i}^{t-1} \beta^{k} k  \label{eq:mom:U_2(t)_ub}
    \end{align}
    where $\II(\cdot)$ is the indicating function which output $1$ if the condition holds true, otherwise $0$.
    
    Denote the left-hand-side of the conclusion as LHS.
    We plug \cref{eq:mom:U_2(t)_ub} into LHS to get
    \begin{align*}
    \operatorname{LHS} &\le \sum_{t=1}^T \gamma^{T-t} {1 \over b_t} {\eta^3_0 \over 4 M (1 - \beta)} \sum \nolimits_{i=1}^{t-1} \norm{v_{i+1}}^2 \sum \nolimits_{k=t-i}^{t-1} \beta^{k} k \\
    &\le {\eta^3_0 \over 4 M (1 - \beta)^2} \sum_{t=1}^T \gamma^{T-t} \sum \nolimits_{i=1}^{t-1} \norm{v_{i+1}}^2 \sum \nolimits_{k=t-i}^{t-1} \beta^{k} k
    \end{align*}
    where we relax the upper bound by ${1\over b_t} = {1 \over 1 - \beta^t} \le {1 \over 1 - \beta}$.
    Using \cref{lm:ineq:sum_of_prop_mom} can directly lead to the conclusion:
    \begin{align*}
    \operatorname{LHS} &\le {\eta^3_0 \beta \gamma \over 2 M (1 - \beta)^3(\gamma - \beta)^2} \sum_{i=1}^{T-1} \gamma^{T-i} \norm{v_{i+1}}^2.
    \end{align*}
\end{proof}

\begin{lemma} \label{lm:ineq:sum_of_prop_mom}
Given variables defined in \cref{thm:excess_loss_momentum}, the following inequality holds true:
\begin{align*}
\sum_{t=1}^T \gamma^{T-t} \sum \nolimits_{i=1}^{t-1} \norm{v_{i+1}}^2 \sum \nolimits_{k=t-i}^{t-1} k \beta^{k} 
\\ \le {2 \beta \gamma \over (\gamma - \beta)^2 (1 - \beta)} \sum_{i=1}^{T-1} \gamma^{T-i} \norm{v_{i+1}}^2.
\end{align*}
\end{lemma}
\begin{proof}
    We first derive the summation:
    \begin{align*}
    &\quad U_1(t, i) \triangleq \sum \nolimits_{k=t-i}^{t-1} \beta^{k} k \\&= \sum \nolimits_{k=t-i}^{t-1} \sum \nolimits_{j=1}^{k} \beta^{k}  \\
    &= \sum \nolimits_{k=t-i}^{t-1} \sum \nolimits_{j=1}^{t-1} \beta^{k} \II(j\le k) \\
    &= \sum \nolimits_{j=1}^{t-1} \sum \nolimits_{k=\max(t-i, j)}^{t-1} \beta^{k} \\
    &= \sum \nolimits_{j=1}^{t-1} {\beta^{\max(t-i, j)} - \beta^t \over 1 - \beta} \\
    &= {1 \over 1 - \beta} \left( (t-i) \beta^{t-i} + {\beta^{t-i+1} - \beta^t \over 1 - \beta} - {\beta - \beta^t \over 1 - \beta} \right)  \\
    &= {1 \over 1 - \beta} \left( (t-i) \beta^{t-i} + {\beta^{t-i+1} - \beta \over 1 - \beta} \right)
    \end{align*}
    Now, we substitute $U_1(t, i)$ into LHS and replace $t-i$ by $j$, i.e., $t=j+i$, to get
    \begin{align*}
    \operatorname{LHS} &= \sum_{t=1}^T \gamma^{T-t} \sum_{i=1}^{t-1} \norm{v_{i+1}}^2 {1 \over 1 - \beta} \left( (t-i) \beta^{t-i} + {\beta^{t-i+1} - \beta \over 1 - \beta} \right) \\
    &= \sum_{i=1}^{T-1} \norm{v_{i+1}}^2 \sum_{t=i+1}^T \gamma^{T-t} {1 \over 1 - \beta} \left( (t-i) \beta^{t-i} + {\beta^{t-i+1} - \beta \over 1 - \beta} \right) \\
    &= \sum_{i=1}^{T-1} \norm{v_{i+1}}^2 \sum_{j=1}^{T-i} \gamma^{T-(j+i)} {1 \over 1 - \beta} \left( j \beta^{j} + {\beta^{j+1} - \beta \over 1 - \beta} \right) \\
    &= \sum_{i=1}^{T-1} \gamma^{T-i} \norm{v_{i+1}}^2 \sum_{j=1}^{T-i} \gamma^{-j} {1 \over 1 - \beta} \left( j \beta^{j} + {\beta^{j+1} - \beta \over 1 - \beta} \right) \\
    &\le {1 \over 1 - \beta} \sum_{i=1}^{T-1} \gamma^{T-i} \norm{v_{i+1}}^2 \sum_{j=1}^{T-i} \left( j \left({\beta \over \gamma}\right)^{j} + {\beta \over 1 - \beta} \left({\beta \over \gamma}\right)^{j} \right)
    \end{align*}
    Let $a = \beta/\gamma$, we show
    \begin{align*}
    \sum_{j=1}^{T-i} j a^{j} &= \sum_{j=1}^{T-i} \sum_{o=1}^j a^{j} \\
    &= \sum_{o=1}^{T-i} \sum_{j=o}^{T-i} a^j \\
    &=  \sum_{o=1}^{T-i} ({a^{o} - a^{T-i+1} \over 1- a}) \\
    &= {a - a^{T-i+1} \over (1- a)^2} - (T-i) { a^{T-i+1} \over 1- a}  \\
    &\le {a \over (1- a)^2}.
    \end{align*}
    Thus,
    \begin{align*}
    \operatorname{LHS} &\le {1 \over 1 - \beta} \sum_{i=1}^{T-1} \gamma^{T-i} \norm{v_{i+1}}^2 \left( {a \over (1- a)^2} + {\beta \over 1 - \beta} \sum_{j=1}^{T-i} a^{j} \right) \\
    &\le {1 \over 1 - \beta} \sum_{i=1}^{T-1} \gamma^{T-i} \norm{v_{i+1}}^2 \left( {a \over (1- a)^2} + {\beta \over 1 - \beta} {a \over 1 - a} \right) \\
    &\le {a \over (1- a)^2 (1 - \beta)} \sum_{i=1}^{T-1} \gamma^{T-i} \norm{v_{i+1}}^2 
    \end{align*}
    Because $\gamma < 1$, $\beta < a = \beta/\gamma$ and
    \begin{align*}
    {a \over (1- a)^2} + {\beta \over 1 - \beta} {a \over 1 - a} \le {2 a \over (1- a)^2}.
    \end{align*}
    Therefore,
    \begin{align*}
    \operatorname{LHS} &\le {2 a \over (1- a)^2 (1 - \beta)} \sum_{i=1}^{T-1} \gamma^{T-i} \norm{v_{i+1}}^2 \\&= {2 \beta \gamma \over (\gamma - \beta)^2 (1 - \beta)} \sum_{i=1}^{T-1} \gamma^{T-i} \norm{v_{i+1}}^2
    \end{align*}
\end{proof}

\begin{lemma} \label{lm:beta_gamma_cond}
Suppose $\gamma \in (0, 1)$ and $\beta \in (0,1)$.
Define 
\begin{align*}
    \hat T = \max t~ \st ~\gamma^{t-1} \ge {1-\beta \over 1- \beta^t}.
\end{align*}
If $t\le \hat T$, ${1-\beta \over 1- \beta^t} \le \gamma^{t-1}$ for $t= 1,\dots, T$. If $t>\hat T$, ${1-\beta \over 1- \beta^t} < \gamma^{\hat T-1}$.
\end{lemma}
\begin{proof}
    Define $h(t) = \gamma^{t-1} (1-\beta^t)$ whose derivatives are
    \begin{align*}
    h'(t) &= \gamma^{t-1} (1-\beta^t) \ln \gamma + \gamma^{t-1} (-\beta^t)\ln \beta \\
    &= \gamma^{t-1} \left[ \ln \gamma - \beta^t (\ln \gamma + \ln \beta) \right] \\
    &= \gamma^{t-1} \left[ 1 - \beta^t (1 + \log_\gamma \beta) \right] \ln \gamma.
    \end{align*}
    Simple calculation shows $\left. 1 - \beta^t (1 + \log_\gamma \beta) \right|_{t=0} =-\log_\gamma \beta < 0$ and $\lim_{t\rightarrow +\infty} 1 - \beta^t (1 + \log_\gamma \beta) = 1$.
    When $t = -\log_\beta(1+\log_\gamma \beta)$ denoted as $t_0$, $1 - \beta^t (1 + \log_\gamma \beta)=0$.
    Because $1 - \beta^t (1 + \log_\gamma \beta)$ is monotonically increasing by $t$ and $\gamma^{t-1}\ln \gamma$ is negative, $h'(t) \ge 0$ if $t\le t_0$.
    Otherwise, $h'(t) < 0$.
    Therefore, $h(t)$ is a concave function.
    Because $h(1)=1-\beta$ and $h(\hat T)=\gamma^{\hat T-1}(1-\beta^{\hat T})\ge 1-\beta>0$, $h(t) \ge 1-\beta$ for $t=1,\dots, \hat T$.
    Thus, for all $t\in[1, \hat T]$, we have ${1-\beta \over 1- \beta^t} \le \gamma^{t-1}$.

    For $t > \hat T$, because ${1-\beta \over 1- \beta^t}$ monotonically increases by $t$, we have ${1-\beta \over 1- \beta^t} < {1-\beta \over 1- \beta^{\hat T}} \le \gamma^{\hat T-1}$. 
\end{proof}
\section{Proofs}

\label{proofsection:prAtEndii}\begin{proof}[Proof of \autoref{thm:prAtEndii}]\phantomsection\label{proof:prAtEndii}Because all sample gradient are $G$-Lipschitz continuous, the sensitivity of the averaged gradient is upper bounded by $G/N$. Based on \cref {thm:gauss_zCDP}, the privacy cost of $g_t$ is $1\over 2\sigma _t^2$\footnote {For brevity, when we say the privacy cost of some value, e.g., gradient, we actually refer to the cost of mechanism that output the value.}. \par Here, we make the output of each iteration a tuple of $(\theta _{t+1}, v_{t=1})$. For the $1$st iteration, because $\theta _1$ does not embrace private information by random initialization, the mapping, \begin {gather*} \left [\begin {array}{cc} v_{2} \\ \theta _{2} \end {array} \right ] = \left [\begin {array}{cc} g_1 \\ \theta _1 - \eta _1 g_1 \end {array} \right ], \end {gather*} is $\hat \rho _1$-zCDP where $\hat \rho _1 = {1\over 2\sigma _t^2}$. \par Suppose the output of the $t$-th iteration, $(\theta _t, v_t)$, is $\hat \rho _t$-zCDP. At each iteration, we have the following mapping $(\theta _t, v_t) \rightarrow (\theta _{t+1}, v_{t+1})$ defined as \begin {gather*} \left [\begin {array}{cc} v_{t+1} \\ \theta _{t+1} \end {array} \right ] = \left [\begin {array}{cc} \phi (v_t, g_t) \\ \theta _t - \eta _t \phi (v_t, g_t) \end {array} \right ]. \end {gather*} Thus, the output tuple $(\theta _{t+1}, v_{t+1})$ is $(\hat \rho _t + {1\over 2\sigma _t^2})$-zCDP by \cref {thm:comp_postprocess_zCDP}. \par Thus, the recursion implies that $(\theta _{T+1}, v_{T+1})$ has privacy cost as \begin {align*} \hat \rho _{T+1} = \hat \rho _{T} + {1\over 2\sigma _T^2} = \cdots = \sum _{t=1}^T {1\over 2\sigma _t^2} = {1\over 2} \sum _{t=1}^T \rho _t \le {1\over 2} (R-R_T) \le {1\over 2} R. \end {align*} Let $\rho = \hat \rho _{T+1}$. Then we can get the conclusion.\end{proof}

\subsection{Gradient Descents}

\label{proofsection:prAtEndiii}\begin{proof}[Proof of \autoref{thm:prAtEndiii}]\phantomsection\label{proof:prAtEndiii}With the definition of smoothness in \cref {def:smooth} and \cref {eq:private_iter}, we have \begin {align*} f(\theta _{t+1}) - f(\theta _t) &\le - \eta _t \nabla _t^\top (\nabla _t + G \sigma _t \nu _t / N) + {1\over 2} M \eta _t^2 \norm {\nabla _t + G \sigma _t \nu _t / N}^2 \\ &= - \eta _t (1 - {1\over 2} M \eta _t) \norm {\nabla _t}^2 - (1 - M\eta _t) \eta _t \nabla _t^\top G \sigma _t \nu _t / N + {1\over 2} M \eta _t^2 \norm { G \sigma _t \nu _t / N}^2 \\ &\le - 2 \mu \eta _t (1 - {1\over 2} M \eta _t) (f(\theta _t) - f(\theta ^*) ) - (1 - M\eta _t) \eta _t \nabla _t^\top G \sigma _t \nu _t / N \\ &\quad + {1\over 2} M \eta _t^2 \norm { G \sigma _t \nu _t / N}^2. \end {align*} where the last inequality is due to the Polyak-Lojasiewicz condition. Taking expectation on both sides, we can obtain \begin {align*} \Ebb [f(\theta _{t+1})] - \Ebb [f(\theta _t)] &\le - 2 \mu \eta _t (1 - {M\over 2} \eta _t) (\Ebb [f(\theta _t)] - f(\theta ^*) ) + {M\over 2} (\eta _t G \sigma _t / N)^2 \Ebb \norm { \nu _t }^2 \end {align*} which can be reformulated by substacting $f(\theta ^*)$ on both sides and re-arranged as \begin {align*} \Ebb [f(\theta _{t+1})] - f(\theta ^*) &\le \left ( 1 - 2 \mu \eta _t (1 - {M\over 2} \eta _t) \right ) (\Ebb [f(\theta _t)] - f(\theta ^*) ) + {M\over 2} (\eta _t G \sigma _t / N)^2 D \end {align*} Recursively using the inequality, we can get \begin {align*} \Ebb [f(\theta _{T+1})] - f(\theta ^*) &\le \prod _{t=1}^T \left (1 - 2 \mu \eta _t (1 - {M\over 2} \eta _t) \right ) (\Ebb [f(\theta _1)] - f(\theta ^*) ) \\ &\quad + {M D\over 2} \sum _{t=1}^T \prod _{i=t+1}^T \left (1 - 2 \mu \eta _i (1 - {M\over 2} \eta _i) \right ) (\eta _t G \sigma _t / N)^2. \end {align*} Let $\eta _t \equiv 1/M$. Then the above inequality can be simplified as \begin {align*} \Ebb [f(\theta _{T+1})] - f(\theta ^*) &\le \gamma ^T (\Ebb [f(\theta _1)] - f(\theta ^*) ) + R \sum _{t=1}^T \gamma ^{T-t} {MD \over 2R} \left ({\eta _t G \over N}\right )^2 \sigma _t^2 \\ &= \gamma ^T (\Ebb [f(\theta _1)] - f(\theta ^*) ) + R \sum _{t=1}^T \gamma ^{T-t} \alpha \sigma _t^2 (\Ebb [f(\theta _1)] - f(\theta ^*) ) \\ &= \left ( \gamma ^T + R \sum \nolimits _{t=1}^T q_t \sigma _t^2 \right ) (f(\theta _1) - f(\theta ^*)) \end {align*}\end{proof}
\label{proofsection:prAtEndiv}\begin{proof}[Proof of \autoref{thm:prAtEndiv}]\phantomsection\label{proof:prAtEndiv}The minimizer of the upper bound of \cref {eq:excess_loss_motivation} can be written as \begin {align} T^* &= \argmin _T \gamma ^T + \alpha \kappa (1 - \gamma ^T) T \label {eq:optimal_T_eer:erub} \end {align} where we substitute $\sigma ^2 = T/R$ in the second line. To find the convex minimization problem, we need to vanishing its gradient which involves an equation like $T\gamma ^T = c$ for some real constant $c$. However, the solution is $W_k(c)$ for some integer $k$ where $W$ is Lambert W function which does not have a simple analytical form. Instead, because $\gamma ^T > 0$, we can minimize a surrogate upper bound as following \begin {align} T^* &= \argmin _T \gamma ^T + \alpha \kappa T = {1\over \ln (1/\gamma )} \ln \left ( {\ln (1/\gamma ) \over \kappa \alpha } \right ), \text { if } \kappa \alpha + \ln \gamma < 0 \end {align} where we use the surrogate upper bound in the second line and utilize $\gamma = 1 - {1 \over \kappa }$. However, the minimizer of the surrogate objective is not optimal for the original objective. When $\kappa $ is large, the term, $- \alpha \kappa \gamma ^T T$, cannot be neglected as we expect. On the other hands, $T$ suffers from explosion if $\kappa \rightarrow \infty $ and meanwhile $1/\gamma \rightarrow _+ 1$. The tendency is counterintuitive since a small $T$ should be taken for sharp losses. To fix the issue, we change the form of $T^*$ as \begin {align} T^* = {1\over \ln (1/\gamma )} \ln \left ( 1+ { \ln (1/\gamma ) \over \alpha } \right ), \label {eq:optimal_T_eer:optimal_T:proof} \end {align} which gradually converges to $0$ as $\kappa \rightarrow \infty $. \par Now we substitute \cref {eq:optimal_T_eer:optimal_T:proof} into the original objective function, \cref {eq:optimal_T_eer:erub}, to get \begin {align} \operatorname {ERUB}^{\text {uniform}} &= {1 \over 1 + \ln (1/\gamma )/ \alpha } \left [ 1 + \kappa \ln \left ( 1 + { \ln (1/\gamma ) \over \alpha } \right ) \right ] . \label {eq:optimal_T_eer:proof:erub_opt} \end {align} Notice that \begin {align*} \ln (1/\gamma ) &= \ln (\kappa /(\kappa -1)) = \ln (1 + 1/(\kappa -1)) \le {1\over \kappa - 1} \le {1\over c\kappa } \end {align*} because $\kappa \ge {1\over 1 - c} > 1$ for some constant $c \in (0, 1)$. In addition, \begin {align*} \ln (1/\gamma ) &= - \ln (1-1/\kappa ) \ge 1/\kappa . \end {align*} Now, we can get the upper bound of \cref {eq:optimal_T_eer:proof:erub_opt} as \begin {align*} \operatorname {ERUB}^{\text {uniform}} &\le {\kappa \over \kappa + 1/ \alpha } \left [ 1 + \kappa \ln \left ( 1 + { 1 \over c\kappa \alpha } \right ) \right ] \\ &\le c_1 {\kappa \over \kappa + 1/ \alpha } \kappa \left [ \ln \left ( 1 + { 1 \over \kappa \alpha } \right ) + \ln ({1\over c}) )\right ] \\ &\le c_1 c_2 {\kappa ^2 \over \kappa + 1/ \alpha } \ln \left ( 1 + { 1 \over \kappa \alpha } \right ) \end {align*} for some constants $c_1, c_2$ and large enough $1\over \alpha $. Also, we can get the lower bound \begin {align*} \operatorname {ERUB}^{\text {uniform}} &\ge {c\kappa \over c\kappa + 1/ \alpha } \left [ 1 + \kappa \ln \left ( 1 + {1 \over \kappa \alpha } \right ) \right ] \ge c {\kappa ^2 \over \kappa + 1/ \alpha } \ln \left ( 1+ {1 \over \kappa \alpha } \right ) . \end {align*} where we use the condition $c\in (0,1)$. Thus, $\operatorname {ERUB}^{\text {uniform}} = \Theta \left ({\kappa ^2 \over \kappa + 1/ \alpha } \ln \left ( 1 + {1 \over \kappa \alpha } \right )\right )$.\end{proof}
\label{proofsection:prAtEndv}\begin{proof}[Proof of \cref{thm:prAtEndv}]\phantomsection\label{proof:prAtEndv}By $\sum _{t=1}^T \sigma ^{-2}=R$ and Cauchy-Schwarz inequality, we can derive the achievable lower bound as \begin {align*} R \sum _t q_t \sigma _t^2 = \sum _t {1\over \sigma _t^2} \sum _t q_t \sigma _t^2 \ge \left (\sum _{t=1}^T \sqrt {q_t} \right )^2 \end {align*} where the inequality becomes equality if and only if ${s/ \sigma _t^2} = q_t \sigma _t^2$, i.e., $\sigma _t = (s/q_t)^{1/4}$, for some positive constant $s$. The equality $\sum _{t=1}^T \sigma _t^{-2}=R$ immediately suggests $\sqrt {s} = {1\over R} \sum _{t=1}^T \sqrt {q_t}$. Thus, we get the $\sigma _t$. \par \rev { Notice \begin {align} T\sum _{t=1}^T q_t - \left (\sum \nolimits _{t=1}^T \sqrt {q_t} \right )^2 = T^2 {1\over T} \sum _{t=1}^T \left (\sqrt {q_t} - {1\over T} \sum _{i=1}^T \sqrt {q_i} \right )^2 = T^2 \operatorname {Var}[q_t] \end {align} where the variance is w.r.t. $t$. }\end{proof}
\label{proofsection:prAtEndvi}\begin{proof}[Proof of \autoref{thm:prAtEndvi}]\phantomsection\label{proof:prAtEndvi}The upper bound of \cref {eq:excess_loss_motivation} can be written as \begin {align*} \text {ERUB}^{\text {dyn}} &= \gamma ^T + \sum \nolimits _{t=1}^T \gamma ^{T-t} \alpha R \sigma ^2_t \\ &= \gamma ^T + \alpha \left ( \sum \nolimits _{t=1}^T \sqrt {\gamma ^{T-t}} \right )^2 \\ &= \gamma ^T + \alpha \left ({1 - \gamma ^{T/2} \over 1 - \sqrt {\gamma }} \right )^2 \end {align*} where we make use of \cref {thm:cachy-scharz_weighted_sum}. Then, the minimizer of the ERUB is \begin {align} T^* &= \argmin _T \gamma ^T + \alpha \left ({1 - \gamma ^{T/2} \over 1 - \sqrt {\gamma }} \right )^2 \notag \\ &= 2 \log _{\gamma } \left ( {\alpha \over \alpha + (1 - \sqrt {\gamma })^2} \right ). \label {eq:thm:optimal_T_eer_dyn:proof:T} \end {align} We can substitute \cref {eq:thm:optimal_T_eer_dyn:proof:T} into $\text {ERUB}^{\text {dyn}}$ to get \begin {align*} \operatorname {ERUB}^{\text {dyn}}_{\min } &= \left ( {\alpha \over \alpha + (1 - \sqrt {\gamma })^2} \right )^2 + \alpha \left ({1 \over 1 - \sqrt {\gamma }} \right )^2 \left ( 1 - {\alpha \over \alpha + (1 - \sqrt {\gamma })^2} \right )^2 \\ &= \left ( {\alpha (1 - \sqrt {\gamma })^{-2} \over \alpha (1 - \sqrt {\gamma })^{-2} + 1} \right )^2 + {\alpha (1 - \sqrt {\gamma })^{-2} \over \left ( \alpha (1 - \sqrt {\gamma })^{-2} + 1 \right )^2} \\ &= {\alpha (1 - \sqrt {\gamma })^{-2} \over \alpha (1 - \sqrt {\gamma })^{-2} + 1} \end {align*} Notice that $\left ( 1 - \sqrt {\gamma } \right )^{-2} = \kappa ^2 + \kappa ^2 - \kappa + 2 \kappa \sqrt {\kappa (\kappa - 1)} = \kappa (2\kappa -1 + 2 \sqrt {\kappa (\kappa - 1)} )$ and it is bounded by \begin {gather*} \kappa (2\kappa -1 + 2 \sqrt {\kappa (\kappa - 1)} ) \le 4 \kappa ^2, \\ \kappa (2\kappa -1 + 2 \sqrt {\kappa (\kappa - 1)} ) \ge \kappa (2\kappa - (3\kappa -2) + 2 \sqrt {(\kappa - 1) (\kappa - 1)} ) = \kappa (-\kappa +2 + 2\kappa - 2 ) = \kappa ^2. \end {gather*} Therefore, $\kappa \le \left ( 1 - \sqrt {\gamma } \right )^{-1} \le 2\kappa $, with which we can derive \begin {align*} \operatorname {ERUB}^{\text {dyn}}_{\min } &\le 4 {\kappa ^2 \alpha \over \kappa ^2 \alpha + 1}, \\ \operatorname {ERUB}^{\text {dyn}}_{\min } &\ge {\kappa ^2 \alpha \over 4 \kappa ^2 \alpha + 1} \ge {1\over 4} {\kappa ^2 \alpha \over \kappa ^2 \alpha + 1}. \end {align*} Thus, $\operatorname {ERUB}^{\text {dyn}}_{\min } = \Theta \left ({\kappa ^2 \alpha \over \kappa ^2 \alpha + 1}\right )$.\end{proof}
 
\subsection{Gradient Descents with Momentum}

\label{proofsection:prAtEndvii}\begin{proof}[Proof of \autoref{thm:prAtEndvii}]\phantomsection\label{proof:prAtEndvii}Without loss of generality, we absorb the $C\sigma _t/N$ into the variance of $\nu _t$ such that $\nu _t \sim \mathcal {N}(0, {C \sigma _t^2\over N} I)$ and $g_t \leftarrow \nabla _t + \nu _t$. Define $b_t = 1 - \beta ^t$. \par By smoothness and \cref {eq:private_iter}, we have \begin {align} f(\theta _{t+1}) - f(\theta _t) &\le \nabla _t^\top (\theta _{t+1} - \theta _t) + {1\over 2} M \norm {\theta _{t+1} - \theta _t}^2 \notag \\ &= - {\eta _t\over b_t^2} b_t \nabla _t^\top v_{t+1} + {1\over 2} M {\eta _t^2 \over b_t^2} \norm {v_{t+1}}^2 \notag \\ &= {\eta _t \over b_t^2} \left ( \norm {b_t \nabla _t - v_{t+1} }^2 - \norm {b_t \nabla _t}^2 - \norm {v_{t+1}}^2 \right ) + {1\over 2} M {\eta _t^2 \over b_t^2} \norm {v_{t+1}}^2 \notag \\ &= {\eta _t \over b_t^2} \underbrace { \norm {b_t \nabla _t - v_{t+1} }^2 }_{U_1(t)} - \eta _t \norm { \nabla _t}^2 - {\eta _t\over b_t^2} (1 - {1\over 2} M \eta _t ) \norm {v_{t+1}}^2, \label {eq:mom:fthetat+1-ftheta_t_annoted} \end {align} where only the $U_1(t)$ is non-negative. Specifically, $U_1(t)$ describes the difference between current gradient and the average. We can expand $v_{t+1}$ to get an upper bound: \begin {align*} U_1(t) &= \norm {b_t \nabla _t - v_{t+1} }^2 \\ &= \norm {(1 - \beta ) \sum \nolimits _{i=1}^t \beta ^{t-i} \nabla _t - (1 - \beta ) \sum \nolimits _{i=1}^t \beta ^{t-i} g_i }^2 \\ &= (1 - \beta )^2 \norm {\sum \nolimits _{i=1}^t \beta ^{t-i} (\nabla _t - g_i) }^2 \\ &= (1 - \beta )^2 \norm {\sum \nolimits _{i=1}^t \beta ^{t-i} (\nabla _t - \nabla _i) + \sum \nolimits _{i=1}^t \beta ^{t-i} (\nabla _i - g_i) }^2 \\ &\le 2 (1 - \beta )^2 \left [ \norm {\sum \nolimits _{i=1}^t \beta ^{t-i} (\nabla _t - \nabla _i) }^2 + \norm {\sum \nolimits _{i=1}^t \beta ^{t-i} (\nabla _i - g_i) }^2 \right ] \\ &\le 2 (1 - \beta ) \left [ b_t \underbrace { \sum \nolimits _{i=1}^t \beta ^{t-i} \norm {\nabla _t - \nabla _i}^2}_{U_2(t)\text { (gradient variance)}} + (1 - \beta ) \underbrace { \norm {\sum \nolimits _{i=1}^t \beta ^{t-i} \nu _i }^2 }_{\text {noise variance}} \right ] \end {align*} where we use $\norm {x + y}^2 \le (\norm {x} + \norm {y})^2 \le 2(\norm {x}^2 + \norm {y}^2)$. The last inequality can be proved by Cauchy-Schwartz inequality for each coordinate. \par We plug the $U_1(t)$ into \cref {eq:mom:fthetat+1-ftheta_t_annoted} and use the PL condition to get \begin {align*} f(\theta _{t+1}) - f(\theta _t) &\le {\eta _t \over b_t^2} U_1(t) - \eta _t \norm { \nabla _t}^2 - {\eta _t\over b_t^2} (1 - {1\over 2} M \eta _t ) \norm {v_{t+1}}^2 \\ &\le - \eta _t \norm { \nabla _t}^2 + {\eta _t \over b_t^2} 2 (1 - \beta ) \left [ b_t U_2(t) + (1 - \beta ) \norm {\sum \nolimits _{i=1}^t \beta ^{t-i} \nu _i }^2 \right ] \\ &\quad - {\eta _t\over b_t^2} (1 - {1\over 2} M \eta _t ) \norm {v_{t+1}}^2 \\ &\le - 2 \mu \eta _t (f(\theta _{t}) - f(\theta ^*)) + {2 (1 - \beta ) \eta _t \over b_t} U_2(t) + {2 (1 - \beta )^2\eta _t \over b_t^2} \norm {\sum \nolimits _{i=1}^t \beta ^{t-i} \nu _i }^2 \\ &\quad - {\eta _t\over b_t^2} (1 - {1\over 2} M \eta _t ) \norm {v_{t+1}}^2. \end {align*} Rearranging terms and taking expectation to show \begin {align*} \Ebb [f(\theta _{t+1})] - f(\theta ^*) &\le \gamma (\Ebb [f(\theta _{t})] - f(\theta ^*)) + {2 (1 - \beta )^2\eta _t \over b_t^2} \sum \nolimits _{i=1}^t \beta ^{t-i} \Ebb \norm {\nu _i }^2 \\ &\quad + {2 (1 - \beta ) \eta _t \over b_t} \Ebb [ U_2(t) ] - {\eta _t\over b_t^2} (1 - {1\over 2} M \eta _t ) \Ebb \norm {v_{t+1}}^2 \\ &= \gamma (\Ebb [f(\theta _{t})] - f(\theta ^*)) + {2 (1 - \beta )^2\eta _t \over b_t^2} \sum \nolimits _{i=1}^t \beta ^{2(t-i)} {C^2 D \sigma _t^2 \over N^2} \\ &\quad + {2 (1 - \beta ) \eta _t \over b_t} \Ebb [ U_2(t) ] - {\eta _t\over b_t^2} (1 - {1\over 2} M \eta _t ) \Ebb \norm {v_{t+1}}^2 \end {align*} where $\gamma = 1 - \eta _0 / \kappa = 1 - 2 \mu \eta _t$. The recursive inequality implies \begin {align*} \Ebb [f(\theta _{T+1})] - f(\theta ^*) &\le \gamma ^T (f(\theta _{1}) - f(\theta ^*)) + \sum _{t=1}^T \gamma ^{T-t} {2 (1 - \beta )^2\eta _t \over b_t^2} \sum \nolimits _{i=1}^t \beta ^{2(t-i)} {C^2 D \sigma _t^2 \over N^2} \\ &\quad + \sum _{t=1}^T \gamma ^{T-t} {2 (1 - \beta ) \eta _t \over b_t} \Ebb [ U_2(t) ] - \sum _{t=1}^T \gamma ^{T-t} {\eta _t\over b_t^2} (1 - {1\over 2} M \eta _t ) \Ebb \norm {v_{t+1}}^2 \\ &= \bigg ( \gamma ^T + 2 \eta _0 \alpha R \underbrace { \sum \nolimits _{t=1}^T \gamma ^{T-t} {(1 - \beta )^2 \over b_t^2} \sum \nolimits _{i=1}^t \beta ^{2(t-i)} \sigma _t^2 }_{U_3} \bigg ) (f(\theta _{1}) - f(\theta ^*)) \\ &\quad + \underbrace { \sum _{t=1}^T \gamma ^{T-t} {2 (1 - \beta ) \eta _t \over b_t} \Ebb [ U_2(t) ] - \sum _{t=1}^T \gamma ^{T-t} {\eta _t\over b_t^2} (1 - {1\over 2} M \eta _t ) \Ebb \norm {v_{t+1}}^2 }_{U_4(t)}. \end {align*} where we utilize $\alpha = {D C^2 \over 2 M N^2 R} {1\over f(\theta _1) - f(\theta ^*)}$ and $\eta _t = {\eta _0\over 2M}$. \par By \cref {lm:ineq:sum_of_prop_mom_x1}, we have \begin {align*} \sum _{t=1}^T \gamma ^{T-t} {2 (1 - \beta ) \eta _t \over b_t} U_2(t) &\le {\eta ^3_0 \beta \gamma \over 2 M (1 - \beta )^3(\gamma - \beta )^2} \sum _{i=1}^{T-1} \gamma ^{T-i} \norm {v_{i+1}}^2. \end {align*} Thus, by ${1\over b_t} \ge 1$, \begin {align*} U_4(t) &\le {\eta ^3_0 \beta \gamma \over 2 M (1 - \beta )^3(\gamma - \beta )^2} \sum _{i=1}^{T-1} \gamma ^{T-i} \Ebb \norm {v_{i+1}}^2 - {\eta _0\over 2M} (1 - {\eta _0\over 4} ) \sum _{t=1}^T \gamma ^{T-t} \Ebb \norm {v_{t+1}}^2 \\ &= - {\eta _0\over 2M} \zeta \sum _{t=1}^T \gamma ^{T-t} \Ebb \norm {v_{t+1}}^2 \end {align*} where \begin {align*} \zeta = 1 - {1 \over 4} \eta _0 - {\beta \gamma \over (\gamma - \beta )^2 (1 - \beta )^3} \eta ^2_0 = 1 - {1 \over 4} \eta _0 - {\beta / \gamma \over (1 - \beta /\gamma )^2 (1 - \beta )^3} \eta ^2_0 \end {align*} When a small enough $\eta _0$, e.g., Specifically, \begin {align*} \eta _0 &\le {(\gamma - \beta )^2 (1 - \beta )^3 \over 8 \beta \gamma } \left [ \sqrt {1 + {64 \beta \gamma \over (\gamma - \beta )^2 (1 - \beta )^3}} - 1 \right ] \\ &= {8 \over \sqrt {1 + {64 \beta \gamma (\gamma - \beta )^{-2} (1 - \beta )^{-3}}} + 1} \end {align*} We can have $\zeta \ge 0$. \par By the definition of $U_3(T, \sigma )$, we can get \begin {align*} \Ebb [f(\theta _{T+1})] - f(\theta ^*) &\le \left ( \gamma ^T + 2 \eta _0 \alpha R U_3(T, \sigma ) \right ) (f(\theta _{1}) - f(\theta ^*)) - {\eta _0\over 2M} \zeta \sum _{t=1}^T \gamma ^{T-t} \Ebb \norm {v_{t+1}}^2. \end {align*} \par \par \par \end{proof}
\label{proofsection:prAtEndviii}\begin{proof}[Proof of \autoref{thm:prAtEndviii}]\phantomsection\label{proof:prAtEndviii}Since $\sigma _t$ is static, by definition of $U_3$ in \cref {thm:excess_loss_momentum}, \begin {align*} U_3 &= \sum \nolimits _{t=1}^T \gamma ^{T-t} {(1 - \beta )^2 \over (1 - \beta ^t)^2} \sum \nolimits _{i=1}^t \beta ^{2(t-i)} \sigma ^2 \\ &= \sigma ^2 \sum \nolimits _{t=1}^T \gamma ^{T-t} {(1 - \beta )^2 \over (1 - \beta ^t)^2} \sum \nolimits _{i=1}^t \beta ^{2(t-i)} \\ &= \sigma ^2 \sum \nolimits _{t=1}^T \gamma ^{T-t} {(1 - \beta )^2 \over (1 - \beta ^t)^2} {1 - \beta ^{2t} \over 1 - \beta ^2} \\ &= \sigma ^2 \sum \nolimits _{t=1}^T \gamma ^{T-t} {1 - \beta \over 1 - \beta ^t} {1 + \beta ^{t} \over 1 + \beta }. \end {align*} Because ${1 - \beta \over 1 - \beta ^t} {1 + \beta ^{t} \over 1 + \beta } \le 1$, the $U_3$ will be smaller than the corresponding summation in GD with uniform schedule. \par By \cref {lm:beta_gamma_cond}, when $T>\hat T$, we can rewrite $U_3$ as \begin {align*} U_3 &\le \sigma ^2 \sum \nolimits _{t=1}^T \gamma ^{T-t} {1 - \beta \over 1 - \beta ^t} \\ &= \sigma ^2 \sum \nolimits _{t=1}^{\hat T} \gamma ^{T-t} {1 - \beta \over 1 - \beta ^t} + \sigma ^2 \sum \nolimits _{t=\hat T+1}^T \gamma ^{T-t} {1 - \beta \over 1 - \beta ^t} \\ &\le \sigma ^2 \sum \nolimits _{t=1}^{\hat T} \gamma ^{T-t} \gamma ^{t-1} + \sigma ^2 \sum \nolimits _{t=\hat T+1}^T \gamma ^{T-t} \gamma ^{\hat T-1} \\ &= \sigma ^2 \gamma ^{T-1} \hat T + \sigma ^2 \gamma ^{\hat T-1} \sum \nolimits _{t=1}^{T-\hat T} \gamma ^{T- \hat T -t} \\ &= \sigma ^2 \gamma ^{T-1} \hat T + \sigma ^2 {\gamma ^{\hat T-1} - \gamma ^{T-1} \over 1 - \gamma } \\ &= {T\over \gamma R} \gamma ^{T} \left ( \hat T + {\gamma ^{\hat T- T} - 1 \over 1 - \gamma } \right ) \end {align*} where we use $\sigma ^2 = T/R$ in the last line. Without assuming $T> \hat T$, we can generally write the upper bound as \begin {align*} U_3 &\le {T\over \gamma R} \gamma ^{T} \left ( \min \{\hat T, T\} + \max \{ {\gamma ^{\hat T- T} - 1 \over 1 - \gamma } , 0\} \right ). \end {align*} \par \par By \cref {thm:excess_loss_momentum}, because $\zeta \ge 0$, we have \begin {align*} \operatorname {ERUB} &\le \gamma ^T + 2 R \eta _0 \alpha U_3 \\ &= \gamma ^T (1 + \frac {\alpha '}{\gamma } T \left ( \min \{\hat T, T\} + \max \{ {\gamma ^{\hat T- T} - 1 \over 1 - \gamma } , 0\} \right )) \end {align*} where $\alpha '=2 \eta _0 \alpha $. \par First, we consider $T\le \hat T$. Use $T = {1\over \ln (1/\gamma )} \ln \left ( 1 + {\eta _0 \over \kappa \alpha } \right ) = \left \lceil \cO \left ( {\kappa \over \eta _0} \ln \left ( 1 + {\eta _0 \over \kappa \alpha } \right ) \right ) \right \rceil $ to get \begin {align*} \operatorname {ERUB} &\le \left ( {\alpha \over \alpha + \eta _0 /\kappa } \right ) \left ( 1 + \alpha ' \gamma ^{-1} ({2\over \ln (1/\gamma )} \ln \left ( 1 + {\eta _0 \over \kappa \alpha } \right ))^2 \right ) \\ &\le \left ( {\alpha \over \alpha + \eta _0 /\kappa } \right ) \left ( 1 + {8 \kappa ^2 \alpha \over \eta _0 \gamma } \ln ^2 \left ( 1 + {\eta _0 \over \kappa \alpha } \right ) \right ) \\ &\le \cO \left ( {\kappa \over \kappa + \eta _0/ \alpha } \left ( 1 + {8\kappa ^2 \alpha \over \eta _0 \gamma } \ln ^2 \left ( 1 + {\eta _0 \over \kappa \alpha } \right ) \right ) \right ) \\ &= \cO \left ( {\kappa \over \kappa + \eta _0 / \alpha } \left ( 1 + {4\kappa \over \gamma } \right ) \right ) \\ &= \cO \left ( {\kappa ^2 \over \kappa + \eta _0 / \alpha } \right ) \end {align*} where we used $\ln (1/\gamma ) \ge \eta _0/\kappa $ and $\ln (1+x) \le \sqrt {x}$ for any $x>0$. \par Second, when $T > \hat T$, \begin {align*} \operatorname {ERUB} &\le \gamma ^T (1 + \frac {\alpha '}{\gamma } T \left ( \hat T + {\gamma ^{\hat T- T} - 1 \over 1 - \gamma } \right )) \\ &\le \cO \left (\gamma ^T + \frac {2\alpha '}{\gamma } T \kappa (\gamma ^{\hat T} - \gamma ^T) \right ). \end {align*} Make use of $T = \left \lceil {1\over \ln (1/\gamma )} \ln \left ( 1 + {\eta _0 \over \kappa \alpha } \right ) \right \rceil $ to show \begin {align*} \operatorname {ERUB} &\le \cO \left ({\kappa \over \kappa + \eta _0 / \alpha } + \frac {4 \kappa ^2 \alpha }{\eta _0 \gamma } (\gamma ^{\hat T} - {\kappa \over \kappa + \eta _0 / \alpha } ) \ln \left ( 1 + {\eta _0 \over \kappa \alpha } \right ) \right ) \\ &\le \cO \left ({\kappa ^2 \over \kappa + \eta _0 / \alpha } \gamma ^{\hat T-1} \ln \left ( 1 + {\eta _0 \over \kappa \alpha } \right ) \right ). \end {align*}\end{proof}
\label{proofsection:prAtEndix}\begin{proof}[Proof of \autoref{thm:prAtEndix}]\phantomsection\label{proof:prAtEndix}By \cref {lm:beta_gamma_cond}, we can rewrite $U_3$ as \begin {align*} U_3 &= \sum \nolimits _{t=1}^T \gamma ^{T-t} {(1 - \beta )^2 \over (1 - \beta ^t)^2} \sum \nolimits _{i=1}^t \beta ^{2(t-i)} \sigma _i^2 \\ &\le \sum \nolimits _{t=1}^{\hat T} \gamma ^{T-t} \gamma ^{2(t-1)} \sum \nolimits _{i=1}^t \beta ^{2(t-i)} \sigma _i^2 + \sum \nolimits _{t=\hat T+1}^T \gamma ^{T-t} \gamma ^{2(\hat T-1)} \sum \nolimits _{i=1}^t \beta ^{2(t-i)} \sigma _i^2 \\ &\le \gamma ^{T - \hat T} \underbrace { \sum \nolimits _{t=1}^{\hat T} \gamma ^{\hat T-t} \gamma ^{2(t-1)} \sum \nolimits _{i=1}^t \beta ^{2(t-i)} \sigma _i^2 }_{V_1} + \gamma ^{2(\hat T-1)} \underbrace {\sum \nolimits _{t=\hat T + 1}^{T} \gamma ^{T-t} \sum \nolimits _{i=1}^t \beta ^{2(t-i)} \sigma _i^2 }_{V_2} \end {align*} We derive $V_1$ and $V_2$ separately. \par For $V_1$, we can obtain the upper bound by \begin {align*} V_1 &= \sum \nolimits _{t=1}^{\hat T} \gamma ^{\hat T-t} \gamma ^{2(t-1)} \sum \nolimits _{i=1}^t \beta ^{2(t-i)} \sigma _i^2 \\ &= \gamma ^{\hat T-2} \sum \nolimits _{t=1}^{\hat T} \gamma ^{t} \sum \nolimits _{i=1}^t \beta ^{2(t-i)} \sigma _i^2 \\ &= \gamma ^{\hat T-2} \sum \nolimits _{i=1}^{\hat T} \beta ^{-2i} \sigma _i^2 \sum \nolimits _{t=i}^{\hat T} \left (\gamma \beta ^2\right )^{t} \\ &= \gamma ^{\hat T-2} \sum \nolimits _{i=1}^{\hat T} \beta ^{-2i} \sigma _i^2 {\left (\gamma \beta ^2\right )^i - \left (\gamma \beta ^2\right )^{\hat T+1} \over 1 - \gamma \beta ^2} \\ &= \gamma ^{2\hat T-3} \sum \nolimits _{i=1}^{\hat T} {\gamma ^{i-\hat T-1} - \beta ^{2(\hat T+1-i)} \over 1 - \gamma \beta ^2} \sigma _i^2 \\ &= \gamma ^{2\hat T-3} \sum \nolimits _{i=1}^{\hat T} {1 - (\gamma \beta ^2)^{\hat T+1-i} \over 1 - \gamma \beta ^2} \gamma ^{i-\hat T-1} \sigma _i^2 \\ &\le {\gamma ^{\hat T} \over \gamma ^2 (1 - \gamma \beta ^2)} \sum \nolimits _{i=1}^{\hat T} \gamma ^i \sigma _i^2 \\ &\le {\gamma ^{\hat T} \over \gamma (\gamma - \beta ^2)} \sum \nolimits _{i=1}^{\hat T} \gamma ^i \sigma _i^2 \end {align*} \par For $V_2$, we can derive \begin {align*} V_2 &= \sum \nolimits _{t=\hat T + 1}^{T} \gamma ^{T-t} \sum \nolimits _{i=1}^t \beta ^{2(t-i)} \sigma _i^2 \\ &= \sum \nolimits _{t=1}^{T} \gamma ^{T-t} \sum \nolimits _{i=1}^t \beta ^{2(t-i)} \sigma _i^2 - \sum \nolimits _{t=1}^{\hat T} \gamma ^{T-t} \sum \nolimits _{i=1}^t \beta ^{2(t-i)} \sigma _i^2 \\ &= \sum \nolimits _{t=1}^{T} \gamma ^{T-t} \sum \nolimits _{i=1}^t \beta ^{2(t-i)} \sigma _i^2 - \gamma ^{T-\hat T} \sum \nolimits _{t=1}^{\hat T} \gamma ^{\hat T-t} \sum \nolimits _{i=1}^t \beta ^{2(t-i)} \sigma _i^2. \end {align*} We first consider the first term \begin {align*} &\quad \sum \nolimits _{t=1}^{T} \gamma ^{T-t} \sum \nolimits _{i=1}^t \beta ^{2(t-i)} \sigma _i^2 \\ &= \sum \nolimits _{i=1}^T \sigma _i^2 \sum \nolimits _{t=i}^{T} \gamma ^{T-t} \beta ^{2(t-i)} \\ &= \sum \nolimits _{i=1}^T \gamma ^T \beta ^{-2i} \sigma _i^2 \sum \nolimits _{t=i}^{T} \gamma ^{-t} \beta ^{2t} \\ &= \sum \nolimits _{i=1}^T \gamma ^T \beta ^{-2i} \sigma _i^2 {(\beta ^2/\gamma )^i - (\beta ^2/\gamma )^{T + 1} \over 1 - (\beta ^2/\gamma )} \\ &= \sum \nolimits _{i=1}^T {\gamma ^{T+1-i} - \beta ^{2(T +1 - i)} \over \gamma - \beta ^2} \sigma _i^2. \end {align*} Similarly, we have \begin {align*} &\quad \gamma ^{T-\hat T} \sum \nolimits _{t=1}^{\hat T} \gamma ^{\hat T-t} \sum \nolimits _{i=1}^t \beta ^{2(t-i)} \sigma _i^2 \\ &= \gamma ^{T-\hat T} \sum \nolimits _{i=1}^{\hat T} {\gamma ^{\hat T+1-i} - \beta ^{2(\hat T +1 - i)} \over \gamma - \beta ^2} \sigma _i^2 \\ &= \sum \nolimits _{i=1}^{\hat T} {\gamma ^{T+1-i} - \gamma ^{T-\hat T} \beta ^{2(\hat T +1 - i)} \over \gamma - \beta ^2} \sigma _i^2. \end {align*} Thus, \begin {align*} V_2 &= \sum \nolimits _{i=1}^T {\gamma ^{T+1-i} - \beta ^{2(T +1 - i)} \over \gamma - \beta ^2} \sigma _i^2 - \sum \nolimits _{i=1}^{\hat T} {\gamma ^{T+1-i} - \gamma ^{T-\hat T} \beta ^{2(\hat T +1 - i)} \over \gamma - \beta ^2} \sigma _i^2 \\ &= \sum \nolimits _{i=\hat T + 1}^T {\gamma ^{T+1-i} - \beta ^{2(T +1 - i)} \over \gamma - \beta ^2} \sigma _i^2 + \sum \nolimits _{i=1}^{\hat T} {\gamma ^{T-\hat T} - \beta ^{2(T-\hat T)} \over \gamma - \beta ^2} \beta ^{2(\hat T +1 - i)} \sigma _i^2 \\ &\le \sum \nolimits _{i=\hat T + 1}^T {\gamma ^{T+1-i} - \beta ^{2(T +1 - i)} \over \gamma - \beta ^2} \sigma _i^2 + \sum \nolimits _{i=1}^{\hat T} {\gamma ^{T-\hat T} \over \gamma - \beta ^2} \beta ^{2(\hat T +1 - i)} \sigma _i^2. \end {align*} Substitute $V_1$ and $V_2$ into $U_3$ to get \begin {align*} U_3 &\le \gamma ^{T} {1 \over \gamma (\gamma - \beta ^2)} \sum \nolimits _{i=1}^{\hat T} \gamma ^i \sigma _i^2 + \gamma ^{2\hat T - 2} \sum \nolimits _{i=\hat T + 1}^T {\gamma ^{T+1-i} - \beta ^{2(T +1 - i)} \over \gamma - \beta ^2} \sigma _i^2 \\ &\quad + \sum \nolimits _{i=1}^{\hat T} {\gamma ^{T+\hat T-2} \over \gamma - \beta ^2} \beta ^{2(\hat T +1 - i)} \sigma _i^2 \\ &\le \left ( {\gamma ^T \over \gamma (\gamma - \beta ^2)} \sum \nolimits _{i=1}^{\hat T} (\gamma ^i + \gamma ^{\hat T - 1} \beta ^{2(\hat T +1 - i)}) \sigma _i^2 + \gamma ^{2\hat T - 2} \sum \nolimits _{i=\hat T + 1}^T {\gamma ^{T+1-i} - \beta ^{2(T +1 - i)} \over \gamma - \beta ^2} \sigma _i^2 \right ) \\ &\le \left ( {2\gamma ^T \over \gamma (\gamma - \beta ^2)} \sum \nolimits _{i=1}^{\hat T} \gamma ^i \sigma _i^2 + \gamma ^{2\hat T - 2} \sum \nolimits _{i=\hat T + 1}^T {\gamma ^{T+1-i} - \beta ^{2(T +1 - i)} \over \gamma - \beta ^2} \sigma _i^2 \right ) \\ &= \sum \nolimits _{t=1}^{T} q_t \sigma _t^2 \end {align*} where \begin {align*} q_t &= {2 \over \gamma (\gamma - \beta ^2)} \gamma ^{T+t} \II _{T\le \hat T} + \gamma ^{2(\hat T-1)} {\gamma ^{T+1-i} - \beta ^{2(T +1 - i)} \over \gamma - \beta ^2} \gamma ^{T-t} \II _{T> \hat T} \\ &\le c_1 \gamma ^{T+t} \II _{T\le \hat T} + \gamma ^{\hat T-1} c_2 \gamma ^{T-t} \II _{T> \hat T} \end {align*} where $c_1 = {2 \over \gamma (\gamma - \beta ^2)}$ and $c_2 = {\gamma ^{2\hat T} \over \gamma - \beta ^2}$. \par When $T> \hat T$, by \cref {thm:cachy-scharz_weighted_sum}, the lower bound of $R \sum \nolimits _{t=1}^T q_t \sigma _t^2$ is \begin {align*} \left (\sum \nolimits _{t=1}^T \sqrt {q_t} \right )^2 &= \gamma ^T \left ( \sum \nolimits _{t=1}^{\hat T} \sqrt {c_1 \gamma ^t} + \sum \nolimits _{t=\hat T + 1}^T \sqrt {\gamma ^{\hat T-1} c_2 \gamma ^{-t}} \right )^2 \\ &= \gamma ^T \left ( \sqrt {c_1 \gamma } \frac {1 - \gamma ^{\hat T/2}}{1 - \sqrt {\gamma }} + \sqrt {c_2} \frac {1 - \gamma ^{(\hat T- T - 1)/2}}{\sqrt {\gamma } - 1} \right )^2 \\ &= \gamma ^T \left ( \sqrt {c_1 \gamma } \frac {1 - \gamma ^{\hat T/2}}{1 - \sqrt {\gamma }} + \sqrt {c_2} \frac {\gamma ^{(\hat T- T - 1)/2} - 1}{1 - \sqrt {\gamma }} \right )^2 \\ &\le \cO \left ( c_2 \left \{ \frac {\gamma ^{(\hat T - 1)/2} - \gamma ^{T/2}}{1 - \sqrt {\gamma }} \right \}^2 \right ) \end {align*} which is achieved when \begin {align*} \sigma _t^2 = {1\over R} \sum \nolimits _{i=1}^T \sqrt {q_i\over q_t}. \end {align*} \par By \cref {thm:excess_loss_momentum}, because $\zeta \ge 0$, we have \begin {align*} \operatorname {ERUB} &\le \gamma ^T + 2 R \eta _0 \alpha U_3 \\ &= \gamma ^T + 2\eta _0 \alpha \sum \nolimits _{t=1}^T R q_t \sigma _t^2. \end {align*} And the minimum of the upper bound is \begin {align*} \operatorname {ERUB}_{\min } &= \gamma ^T + \alpha ' \cO \left ( \left \{ \frac {\gamma ^{(\hat T - 1)/2} - \gamma ^{T/2}}{1 - \sqrt {\gamma }} \right \}^2 \right ) \end {align*} where $\alpha '={2\eta _0 c_2 \alpha }$. Let $T = {2\over \ln (1/\gamma )} \ln \left ( 1 + { \eta _0 \over \kappa \alpha } \right )$. Then, \begin {align*} \operatorname {ERUB}_{\min } &= \cO \left ( \left ({ \kappa \alpha \over \kappa \alpha + \eta _0} \right )^2 + {\alpha ' \over (1 - \sqrt {\gamma })^2} \left \{ { {\gamma ^{(\hat T - 1)/2} - (1-\gamma ^{(\hat T - 1)/2})\kappa \alpha \over \kappa \alpha + \eta _0} } \right \}^2 \right ) \\ &\le \cO \left ( \left ({ \kappa \alpha \over \kappa \alpha + \eta _0} \right )^2 + {2\eta _0 c_2 \alpha \over (1 - \sqrt {\gamma })^2} \left \{ { {\gamma ^{(\hat T - 1)/2} \over \kappa \alpha + \eta _0} } \right \}^2 \right ) \\ &= \cO \left ( { \kappa \alpha \over (\kappa \alpha + \eta _0)^2 } \left (\kappa \alpha + {2\eta _0 c_2 /\kappa \over (1 - \sqrt {\gamma })^2} \gamma ^{(\hat T - 1)} \right ) \right ) \\ &= \cO \left ( { \kappa \alpha \over (\kappa \alpha + \eta _0)^2 } \left (\kappa \alpha + c_3 \eta _0 \right ) \right ) \\ &\le \cO \left ( { \kappa \alpha \over \kappa \alpha + \eta _0 } \right ) \end {align*} where $c_3$ is some constant. \par When $T\le \hat T$, \begin {align*} U_3 &\le \gamma ^{T - T} \underbrace { \sum \nolimits _{t=1}^{T} \gamma ^{T-t} \gamma ^{2(t-1)} \sum \nolimits _{i=1}^t \beta ^{2(t-i)} \sigma _i^2 }_{V_1} \\ &\le {\gamma ^{T-2} \over 1 - \gamma \beta ^2} \sum \nolimits _{i=1}^{T} \gamma ^i \sigma _i^2 \end {align*} with which we obtain \begin {align*} \operatorname {ERUB} &\le \gamma ^T + 2 R \eta _0 \alpha U_3 \\ &\le \gamma ^T + 2 \eta _0 \alpha {\gamma ^{-2} \over 1 - \gamma \beta ^2} \sum \nolimits _{t=1}^{T} R q_t \sigma _t^2. \end {align*} where we let $q_t=\gamma ^{T+t}$. By \cref {thm:cachy-scharz_weighted_sum}, \begin {align*} \sum \nolimits _{i=1}^{T} R q_t \sigma _i^2 &\ge \left (\sum \nolimits _{t=1}^T \sqrt {q_t} \right )^2 \\ &= \gamma ^T \left ( \sum \nolimits _{t=1}^T \gamma ^{t/2} \right )^2 \\ &= \gamma ^{T+1} \left ( {1 - \gamma ^{T/2} \over 1 - \sqrt {\gamma }} \right )^2. \end {align*} Thus, \begin {align*} \operatorname {ERUB}_{\min } &\le \gamma ^T + 2 \eta _0 \alpha {\gamma ^{T-1} \over 1 - \gamma \beta ^2} \left ( {1 - \gamma ^{T/2} \over 1 - \sqrt {\gamma }} \right )^2 \\ &= \gamma ^T \left ( 1 + 2 \eta _0 \gamma c_ 1\alpha \left ( {1 - \gamma ^{T/2} \over 1 - \sqrt {\gamma }} \right )^2 \right ) \end {align*} Let $T = \left \lceil {2\over \ln (1/\gamma )} \ln \left ( 1 + { \eta _0 \over \kappa \alpha } \right ) \right \rceil $. Then, \begin {align*} \operatorname {ERUB}_{\min } &\le \left ({ \kappa \alpha \over \kappa \alpha + \eta _0} \right )^2 \left ( 1 + {2 \eta _0 \gamma c_ 1\alpha \over (1 - \sqrt {\gamma })^2} ({ 1 \over \kappa \alpha + 1})^2 \right ) \\ &\le \left ({ \kappa \alpha \over \kappa \alpha + \eta _0} \right )^2 \left ( 1 + \cO ({ 1 \over \kappa \alpha + 1}) \right ) \\ &\le \cO \left ({ \kappa \alpha \over \kappa \alpha + \eta _0} \right )^2. \end {align*} In summary, \begin {align*} \operatorname {ERUB}_{\min } &\le \cO \left ({ \kappa \alpha \over \kappa \alpha + \eta _0} \left ( \II _{T\le \hat T} { \kappa \alpha \over \kappa \alpha + \eta _0} + \II _{T>\hat T} \right ) \right ) \end {align*}\end{proof}

\subsection{Stochastic Gradient Descents}

\label{proofsection:prAtEndx}\begin{proof}[Proof of \autoref{thm:prAtEndx}]\phantomsection\label{proof:prAtEndx}Let $\tilde \nabla _t$ be the stochastic gradient of the step $t$. By the smoothness, we have \begin {align*} f(\theta _{t+1}) - f(\theta _t) &\le - \eta _t \nabla _t^\top (\tilde \nabla _t + G \sigma _t \nu _t / n) + {1\over 2} M \eta _t^2 \norm {\tilde \nabla _t + G \sigma _t \nu _t / n}^2 \\ &= - \eta _t \nabla _t^\top (\nabla _t + \sigma _g \xi _t/n + G \sigma _t \nu _t / n) + {1\over 2} M \eta _t^2 \norm {\nabla _t + \sigma _g \xi _t/n + G \sigma _t \nu _t / n}^2. \end {align*} Note that $\Ebb (\sigma _g \xi _t/n + G \sigma _t \nu _t / n) = 0$ and $\Ebb (\sigma _g \xi _t/n + G \sigma _t \nu _t / n)^2 = \sigma _g^2 + (G \sigma _t / n)^2$. Without loss of generality, we can write $\sigma _g \xi _t + G \sigma _t \nu _t $ as $ \tilde \sigma _t \zeta _t $ where $\tilde \sigma _t \triangleq \sqrt {\sigma _g^2 + (G \sigma _t)^2}$ and $\zeta _t$ is a random vector with $\Ebb \zeta _t = 0$ and $\Ebb \norm {\zeta _t}^2 \le D$. Therefore, \begin {align*} f(\theta _{t+1}) - f(\theta _t) &\le - \eta _t \nabla _t^\top ( \nabla _t + \tilde \sigma _t \zeta _t/n ) + {1\over 2} M \eta _t^2 \norm {\nabla _t + \tilde \sigma _t \zeta _t/n }^2 \\ &= - \eta _t (1 - {1\over 2} M \eta _t) \norm {\nabla _t}^2 - (1 - M\eta _t) \eta _t \nabla _t^\top \tilde \sigma _t \zeta _t/n + {1\over 2} M \eta _t^2 \norm { \tilde \sigma _t \zeta _t/n }^2 \\ &\le - 2 \mu \eta _t (1 - {1\over 2} M \eta _t) (f(\theta _t) - f(\theta ^*) ) - (1 - M\eta _t) \eta _t \nabla _t^\top \tilde \sigma _t \zeta _t/n \\ &\quad + {1\over 2} M \eta _t^2 \norm { \tilde \sigma _t \zeta _t/n }^2. \end {align*} Then following the same proof of \cref {thm:excess_loss_motivation}, we can get \begin {align*} \Ebb [f(\theta _{T+1})] - f(\theta ^*) &\le \gamma ^T (\Ebb [f(\theta _1)] - f(\theta ^*) ) + R' \sum _{t=1}^T \gamma ^{T-t} \alpha {1 \over G^2} \tilde \sigma _t^2 (\Ebb [f(\theta _1)] - f(\theta ^*) ) \\ &= \left [ \gamma ^T + R' \sum _{t=1}^T \gamma ^{T-t} \alpha ({1 \over G^2} \sigma _g^2 + \sigma _t^2) \right ] (\Ebb [f(\theta _1)] - f(\theta ^*) ) \\ &= \left [ \gamma ^T + R' \alpha {1 \over G^2} \sigma _g^2 {1 - \gamma ^T \over 1 - \gamma } + R' \sum _{t=1}^T \gamma ^{T-t} \alpha \sigma _t^2 \right ] (\Ebb [f(\theta _1)] - f(\theta ^*) ) \\ &\le \left [ \gamma ^T + {R' \kappa \alpha \over G^2} \sigma _g^2 + R' \sum _{t=1}^T \gamma ^{T-t} \alpha \sigma _t^2 \right ] (\Ebb [f(\theta _1)] - f(\theta ^*) ). \end {align*} where ${R' \kappa \alpha \over G^2} = {D \over 2 \mu (f(\theta _1) - f(\theta ^*) )} {1\over n^2} = {D \over 2 \mu (f(\theta _1) - f(\theta ^*) )} \min \{ {1\over N^2 R}, 1 \} \le {D \over 2 \mu (f(\theta _1) - f(\theta ^*) )} {1\over N^2 R}$.\end{proof}
\label{proofsection:prAtEndxi}\begin{proof}[Proof of \autoref{thm:prAtEndxi}]\phantomsection\label{proof:prAtEndxi}Without loss of generality, we can write $\sigma _g \xi _t + G \sigma _t \nu _t $ as $ \tilde \sigma _t \zeta _t $ where $\tilde \sigma _t \triangleq \sqrt {\sigma _g^2 + (G \sigma _t)^2}$ and $\zeta _t$ is a random vector with $\Ebb \zeta _t = 0$ and $\Ebb \norm {\zeta _t}^2 \le D$. Therefore, we replace $\nu _t$ by $\zeta _t$ and $\sigma _t^2$ by $\tilde \sigma _t^2 / G^2 = \sigma _g^2 / G^2 + \sigma _t^2$. Now, we only need to update $U_3(\sigma , T)$ as \begin {align*} \tilde U_3 &= {1 \over G^2} \sum \nolimits _{t=1}^T \gamma ^{T-t} {(1 - \beta )^2 \over (1-\beta ^t)^2} \sum \nolimits _{i=1}^t \beta ^{2(t-i)} \tilde \sigma _i^2 \\ &= \sum \nolimits _{t=1}^T \gamma ^{T-t} {(1 - \beta )^2 \over (1-\beta ^t)^2} \sum \nolimits _{i=1}^t \beta ^{2(t-i)} ( {1 \over G^2} \sigma _g^2 + \sigma _t^2 ) \\ &= U_3^g + U_3 \end {align*} where we define \begin {align*} U^g_3 &\triangleq {1 \over G^2} \sigma _g^2 \sum \nolimits _{t=1}^T \gamma ^{T-t} {(1 - \beta )^2 \over (1-\beta ^t)^2} \sum \nolimits _{i=1}^t \beta ^{2(t-i)}. \end {align*} We can upper bound $U^g_3$ by \begin {align*} U^g_3 &= {1 \over G^2} \sigma _g^2 \sum \nolimits _{t=1}^T \gamma ^{T-t} {(1 - \beta )^2 \over (1-\beta ^t)^2} {1 - \beta ^{2t} \over 1-\beta ^2} \\ &= {1 \over G^2} \sigma _g^2 \sum \nolimits _{t=1}^T \gamma ^{T-t} {1 - \beta \over 1-\beta ^t } {1 + \beta ^{t} \over 1 + \beta } \\ &\le {1 \over G^2} \sigma _g^2 \sum \nolimits _{t=1}^T \gamma ^{T-t} \\ &\le {1 \over G^2} \sigma _g^2 {1\over 1- \gamma } \\ &= {1 \over G^2} \kappa \sigma _g^2. \end {align*} Combine with the factors of $U_3$ in the PGD bounds: \begin {align*} \alpha R' U_3^g &\le {\alpha R' \over G^2} \kappa \sigma _g^2 = {\alpha R' \over G^2} \kappa \sigma _g^2 = { D \sigma _g^2 \over 2 \mu n^2 ( f(\theta _1) - f(\theta ^*))} \le { D \sigma _g^2 \over 2 \mu N^2 R ( f(\theta _1) - f(\theta ^*))}. \end {align*}\end{proof}

\end{document}